\documentclass[journal,onecolumn]{IEEEtran}

\ifCLASSINFOpdf
\else
\fi

\usepackage{cite}
\usepackage{amsmath,amssymb,amsfonts}
\usepackage{graphicx}
\usepackage{textcomp}
\usepackage{xcolor}
\def\BibTeX{{\rm B\kern-.05em{\sc i\kern-.025em b}\kern-.08em
    T\kern-.1667em\lower.7ex\hbox{E}\kern-.125emX}}

\usepackage[english]{babel}
\usepackage[utf8]{inputenc}
\usepackage[T1]{fontenc}
\usepackage{lmodern}
\usepackage{braket}
\usepackage{wrapfig}
\usepackage{graphicx}
\usepackage{amsmath,amssymb,amsfonts,verbatim}
\usepackage{mathtools}
\usepackage{enumitem}

\usepackage{xr}

\usepackage{amsmath,amssymb,amsfonts,verbatim}
\usepackage{graphicx}
\usepackage[utf8]{inputenc}
\usepackage{textcomp}
\usepackage{xcolor}
\usepackage{braket}
\usepackage{soul}
\usepackage[colorinlistoftodos]{todonotes}

\usepackage{caption}
\usepackage{subcaption}

\usepackage{enumitem}

\usepackage{cite}
\usepackage{amsmath,amssymb,amsfonts,verbatim}
\usepackage{amsthm}
\usepackage{graphicx}
\usepackage[utf8]{inputenc}
\usepackage{textcomp}
\usepackage{xcolor}
\usepackage{braket}
\usepackage{soul}
\usepackage{enumitem}
\usepackage[colorinlistoftodos]{todonotes}

\usepackage{algpseudocode}
\usepackage{algorithm}
\usepackage{bbm}

\usepackage{wrapfig}

\usepackage{caption}

\newtheorem{theorem}{Theorem}
\newtheorem{lemma}{Lemma}
\newtheorem{proposition}{Proposition}
\newtheorem{corollary}{Corollary}

\newtheorem{assumption}{A}

\usepackage[utf8]{inputenc}
\usepackage{amsfonts}

\usepackage[english]{babel}
\usepackage[utf8]{inputenc}
\usepackage[T1]{fontenc}
\usepackage{lmodern}
\usepackage{braket}
\usepackage{wrapfig}
\usepackage{graphicx}
\usepackage{amsmath,amssymb,amsfonts,verbatim}
\usepackage{mathtools}
\usepackage{enumitem}

\usepackage{amssymb}

\usepackage{xr}

\usepackage{amsmath,amssymb,amsfonts,verbatim}

\usepackage{dsfont}

\usepackage{mathtools}

\usetikzlibrary{chains, positioning} 
\usetikzlibrary{positioning,shapes,calc}

\usepackage{tikz}
\usetikzlibrary{shapes.geometric, arrows}
\usetikzlibrary{decorations.pathreplacing,angles,quotes}
\usepackage[most]{tcolorbox}
\usepackage{lipsum}

\usepackage[symbols,nogroupskip,sort=none]{glossaries-extra}

\newcommand{\alk}{\alpha_k}
\newcommand{\alkk}{\alpha_{k+1}}
\newcommand{\alkm}{\alpha_{k-1}}
\newcommand{\thk}{\theta_k}
\newcommand{\thkk}{\theta_{k+1}}
\newcommand{\thkm}{\theta_{k-1}}
\newcommand{\step}{\epsilon}
\newcommand{\delstep}{\frac{1}{\Delta^N}}
\newcommand{\kernel}{K\left(\frac{\theta_k -\alk}{\Delta}\right)}
\newcommand{\grad}{\nabla}
\newcommand{\gradn}{\hat{\nabla}}
\newcommand{\rew}{J}
\newcommand{\gaus}{\mathcal{N}}
\newcommand{\reals}{\mathbb{R}}
\newcommand{\CE}{\mathbb{E}}
\newcommand{\PR}{\mathbb{P}}
\newcommand{\noisek}{w_k}

\newcommand{\alt}{\alpha(t)}
\newcommand{\als}{\alpha(s)}
\newcommand{\alts}{\alpha_t}
\newcommand{\altss}{\alpha_s}
\newcommand{\altm}{\alpha^{\step}(t)}
\newcommand{\Lt}{L_t}
\newcommand{\Lss}{L_s}
\newcommand{\brownt}{W(t)}
\newcommand{\browns}{W(s)}
\newcommand{\brown}{W}
\newcommand{\dimn}{N}
\newcommand{\kn}{\kappa_0}
\newcommand{\knw}{\kn^{\idw}}

\newcommand{\tjm}{\theta_j}
\newcommand{\ajm}{\bar{\alpha}(j\step)}

\newcommand{\altbar}{\bar{\alpha}(t)}
\newcommand{\alsbar}{\bar{\alpha}(s)}
\newcommand{\sbar}{\bar{s}}

\newcommand{\lawXt}{\mathbb{P}_X^t}
\newcommand{\lawAt}{\mathbb{P}_A^t}

\newcommand{\wass}{\mathcal{W}_2}
\newcommand{\KL}{D}
\newcommand{\gibbs}{\ids_{\infty}}
\newcommand{\idistr}{\ids_0}
\newcommand{\K}{K}
\newcommand{\Kd}{\K_{\Delta}}

\newcommand{\alsebar}{\bar{\alpha}(\lfloor s / \step \rfloor \step)}
\newcommand{\alsbbar}{\bar{\alpha}(\sbar)}

\newcommand{\xt}{X(t)}
\newcommand{\yt}{Y(t)}

\newcommand{\dtlaw}{\ids_{k}} 

\newcommand{\ctlaw}{\nu_{k\step}}
\newcommand{\ctlawt}{\nu_t}

\newcommand{\N}{\mathbb{N}}

\newcommand{\Dbound}{L_{\idistr}}

\newcommand{\pconst}{c_P}
\newcommand{\LSconst}{c_{LS}}

\newcommand{\sgdstep}{\eta}

\newcommand{\stn}{\tau_n} 
\newcommand{\stnn}{\tau_{n+1}}

\newcommand{\law}{\textrm{Law}}

\newcommand{\bc}{\kappa_0}
\newcommand{\bcc}{\zeta_0}
\newcommand{\bccc}{r}

\newcommand{\gen}{\mathcal{L}}

\newcommand{\ind}{\boldsymbol{1}}

\newcommand{\del}{\partial}

\newcommand{\etsup}{M_{\theta}}
\newcommand{\alksup}{M_{\alpha}}

\newcommand{\ids}{\pi}
\newcommand{\idistmax}{\bar{\ids}_0}
\newcommand{\idw}{\gamma}
\newcommand{\isp}{0,\idw}

\newcommand{\idistmaxw}{\bar{\ids}_{\isp}}

\newcommand{\gendom}{\mathcal{D}(\gen)}
\newcommand{\dirform}{\mathcal{E}}
\newcommand{\specgap}{\lambda}

\newcommand{\tailbound}{M}

\newcommand{\relentbd}{\bar{\KL}_0}
\newcommand{\relentbdw}{\relentbd^{\idw}}

\newcommand{\constA}{A}
\newcommand{\constB}{B}
\newcommand{\gibbsnorm}{\Lambda}

\newcommand{\idiffconst}{\tilde{C}}
\newcommand{\idiffc}{I}
\newcommand{\idiffcc}{I'}

\newcommand{\dissm}{m}
\newcommand{\dissb}{b}

\newcommand{\mk}{\mathcal{K}}

\newcommand{\gradnvar}{\zeta}

\newcommand{\ksup}{\hat{K}}

\newcommand{\levxw}{\Delta_{x}^{\idw}}
\newcommand{\levxwp}{\hat{\Delta}_x^{\idw}}
\newcommand{\wtk}{\zeta_{\alpha_{k-1}}^{\step,\idw}}
\newcommand{\wtz}{\zeta_{z}^{\step,\idw}}

\newcommand{\idistrw}{\ids_{\isp}}

\newcommand{\ytlaw}{\gamma_{k\step}}

\newcommand{\larg}{x}

\newcommand{\lipJ}{L_J}
\newcommand{\lipGJ}{L_{\grad\rew}}

\newcommand{\CS}{\mathcal{S}}
\newcommand{\CA}{\mathcal{A}}

\newcommand{\idistgmax}{\idistmax'}
\newcommand{\ibd}{V_1}
\newcommand{\igbd}{V_2}

\newcommand{\sgdtm}{c_{opt}}

\newcommand{\sgdvarbd}{\hat{\mu}_{sgd}}

\newcommand{\hlb}{y} 
\newcommand{\wassprox}{\rho}

\newcommand{\kfv}{x} 

\newcommand{\zconst}{\xi}

\newcommand{\snum}{T}

\newcommand{\hk}{\hat{k}}
\newcommand{\ks}{k^*}

\newcommand{\algtdum}{q}

\newcommand{\csix}{C_6}

\glsxtrnewsymbol[description={SGD \eqref{eq:sgd} iterate, $\thk \in \reals^{\dimn}$}]{thk}{\ensuremath{\thk}}
\glsxtrnewsymbol[description={PSGLD \eqref{eq:dt_sgld} iterate, $\alk \in \reals^{\dimn}$}]{alk}
{\ensuremath{\alk}}
\glsxtrnewsymbol[description={PSGLD step size}]{step}{\ensuremath{\step}}
\glsxtrnewsymbol[description={sampling distribution \eqref{eq:idistrw}, maximum value: $\idistmaxw$}]{idistrw}{\ensuremath{\idistrw}}
\glsxtrnewsymbol[description={PSGLD inverse temperature parameter}]{beta}{\ensuremath{\beta}}
\glsxtrnewsymbol[description={PSGLD kernel function}]{K}{\ensuremath{\K}}
\glsxtrnewsymbol[description={PSGLD kernel function scale parameter}]{Delta}{\ensuremath{\Delta}}
\glsxtrnewsymbol[description={cost function, $\rew: \reals^{\dimn} \to \reals_{+}$}]{rew}{\ensuremath{$\rew$}}
\glsxtrnewsymbol[description={Gibbs measure ($\gibbs\propto \exp(-\beta J)$)}]{gibbs}{\ensuremath{\gibbs}}
\glsxtrnewsymbol[description={solution of It\^o diffusion \eqref{eq:ct_diff}}]{alt}{\ensuremath{\alt}}
\glsxtrnewsymbol[description={standard Brownian motion}]{brownt}{\ensuremath{\brownt}}
\glsxtrnewsymbol[description={diffusion \eqref{eq:ct_diff} log-Sobolev constant}]{LSconst}{\ensuremath{\LSconst}}
\glsxtrnewsymbol[description={diffusion \eqref{eq:ct_diff} Poincar\^e constant}]{pconst}{\ensuremath{\pconst}}
\glsxtrnewsymbol[description={$\grad\rew$ Lipschitz constant}]{lipJ}{\ensuremath{\lipGJ}}
\glsxtrnewsymbol[description={$\rew$ dissipativity constants}]{mb}{\ensuremath{(\dissm,\dissb)}}
\glsxtrnewsymbol[description={uniform stochastic gradient $\gradn\rew(\cdot)$ variance bound}]{gradnvar}{\ensuremath{\gradnvar}}
\glsxtrnewsymbol[description={$\idistr$ Lipschitz constant}]{Dbound}{\ensuremath{\Dbound}}
\glsxtrnewsymbol[description={uniform SGD bound, $\etsup := \sup_{k\geq 0}\CE\|\thk\|^2$}]{etsup}{\ensuremath{\etsup}}
\glsxtrnewsymbol[description={$\knw := \log\CE_{\idistrw}\left[\exp(\|x\|^2)\right] < \infty$ ($\kn:=\knw|_{\idw=1}$)}]{knw}{\ensuremath{\knw}}
\glsxtrnewsymbol[description={\eqref{eq:lem11} inner product bound 1}]{idiffc}{\ensuremath{\idiffc}}
\glsxtrnewsymbol[description={\eqref{eq:lem11} inner product bound 2}]{idiffcc}{\ensuremath{\idiffcc}}
\glsxtrnewsymbol[description={$|\rew(0)|$}]{constA}{\ensuremath{\constA}}
\glsxtrnewsymbol[description={$\|\grad\rew(0)\|$}]{constB}{\ensuremath{\constB}}
\glsxtrnewsymbol[description={$l_2$ norm}]{lnorm}{$\|\cdot\|$}

\newenvironment{myassumptions}{%
   \begin{description}[style=multiline, leftmargin = 18pt, align=left,font=\normalfont]%
}{%
   \end{description}%
}
 \makeatletter
  \def\nl#1#2{\begingroup
     \scalebox{0.85}[1]{\textbf{#2}}%
    \def\@currentlabel{\textnormal{\scalebox{0.85}[1]{\textbf{#2}}}}
     \phantomsection\label{#1}\endgroup
}
\makeatother

\usepackage{tikz}
\usetikzlibrary{positioning}

\usepackage{mdframed}
\usepackage{lipsum}

\usepackage{xr}

\tikzstyle{pinstyle} = [pin edge={to-,thin,black}]
 
\allowdisplaybreaks
\title{Finite-Sample Bounds for Adaptive Inverse Reinforcement Learning using Passive Langevin Dynamics
\thanks{The results in this paper have appeared in the Proceedings of the 2023 IEEE Conference on Decision and Control in reduced and abbreviated form. This manuscript substantially expands on the exposition of the Conference version, provides all proofs, and includes more detailed examples, discussion, and mathematical details.}
\thanks{This research was supported in part by the National Science Foundation grants CCF-2112457 and CCF-2312198, Army Research Office grant W911NF-19-1-0365, and Air Force Office of Scientific Research grant FA9550-22-1-0016}}

\author{Luke Snow,  Vikram Krishnamurthy \thanks{Department of Electrical \& Computer Engineering, Cornell University, Ithaca, NY 14853, USA.  emails: las474@cornell.edu and vikramk@cornell.edu}}

\begin{document}

\maketitle

\begin{abstract} 
    This paper provides a finite-sample analysis of a passive stochastic gradient Langevin dynamics (PSGLD) algorithm. This algorithm is designed to achieve adaptive inverse reinforcement learning (IRL). Adaptive IRL aims to estimate the cost function of a forward learner performing a stochastic gradient algorithm (e.g., policy gradient reinforcement learning) by observing their estimates in real-time.
    The PSGLD algorithm is considered passive because it incorporates noisy gradients provided by an external stochastic gradient algorithm (forward learner), of which it has no control. The PSGLD algorithm acts as a randomized sampler to achieve adaptive IRL by reconstructing the forward learner's cost function nonparametrically from the stationary measure of a Langevin diffusion. 
    This paper analyzes the non-asymptotic (finite-sample) performance; we provide explicit bounds on the 2-Wasserstein distance between PSGLD algorithm sample measure and the stationary measure encoding the cost function{\color{black}, and provide guarantees for a kernel density estimation scheme which reconstructs the cost function from empirical samples.} Our analysis uses {\color{black}tools from the study of Markov diffusion operators}. 
    The derived bounds have both practical and theoretical significance. They provide finite-time guarantees for an adaptive IRL mechanism, and substantially generalize the analytical framework of a line of research in passive stochastic gradient algorithms.
\end{abstract}



\section{Introduction}


Suppose a 'forward' learner runs a stochastic gradient descent (SGD) algorithm to minimize a cost function. By observing the noisy gradients of the forward learner, how can an inverse learner estimate the cost function in real-time? This paper provides finite-sample guarantees for a passive Langevin dynamics algorithm which achieves this adaptive inverse reinforcement learning (IRL) objective, "adaptive" because the inverse learner estimates the cost function whilst the forward learner is running the gradient algorithm. Figure~\ref{fig:IRLblockdiag} displays the IRL procedure.

\begin{figure}[h!]
    \centering
    \begin{tikzpicture}
  \draw (0,0) rectangle (3,1);
  \node at (1.5,0.5) {cost function $J(\cdot)$};
  \draw (5,0) rectangle (8,1);
  \node at (6.5,0.75) {forward learner}; 
  \node at (6.5,0.3) {SGD}; %
  
  \draw (9.5,0) rectangle (12.5,1);
  \node at (11,0.75) {inverse learner};
  \node at (11,0.3) {PSGLD};

  \node at (8.75,0.75) {$\gradn\rew(\thk)$};
  \node at (4,0.75) {noisy};
  \node at (4,0.25) {evaluation};



  \draw[->] (3,0.5) -- (5,0.5); 
  \draw[->] (8,0.5) -- (9.5,0.5); 
  \draw[->] (12.5,0.5) -- (13,0.5);

  \node at (13.5,0.5) {$\rew(\cdot)$};
  

\end{tikzpicture}
    \caption{\small Schematic for adaptive inverse reinforcement learning. A forward learner evaluates sequential stochastic gradients $\{\gradn\rew(\thk)\}_{k\in\N}$ of a cost function $\rew$, through e.g., stochastic gradient descent (SGD), to obtain the minima of $\rew$. An inverse learner observes $\{\gradn\rew(\thk)\}_{k\in\N}$ and attempts to reconstruct the cost $\rew$ through the passive stochastic gradient Langevin dynamics (PSGLD) algorithm. The aim of this paper is to provide a finite-sample analysis of the PSGLD algorithm \eqref{eq:dt_psgld}. 
    }\label{fig:IRLblockdiag}
\end{figure}

\textcolor{black}{This adaptive IRL setting generalizes traditional frameworks in inverse reinforcement learning by removing the restriction to a Markov Decision Process (MDP) environment, and considers observations of the \textit{transient learning} dynamics. It applies to a broad range of learning processes including, e.g., policy gradient reinforcement learning, empirical risk minimization and Bayesian learning. Our results contribute to the field of \textit{passive} stochastic approximation algorithms, e.g. \cite{nazin1989passive}, \cite{YY96}, by extending to the structure of Langevin dynamics. Analysis in this extended framework necessitates sophisticated mathematical tools in the theory of Markov diffusion operators, and also allows for application to generalized inverse reinforcement learning contexts. As such, we also contribute to the machine learning community by providing the first sample complexity analysis for a general adaptive inverse reinforcement learning technique. }
\vspace{-0.3cm}
\subsection{Context}
This section rapidly provides the main ideas of this paper. The rigorous problem formulation is given in Section~\ref{sec:background}.
\subsubsection{Stochastic Gradient Langevin Dynamics}

The stochastic gradient Langevin dynamics (SGLD) algorithm is
\begin{equation}
\label{eq:gen_dtsgld}
    \alkk = \alk - \step \gradn\rew(\alk) + \sqrt{2 \beta^{-1} \step}\noisek, \quad k\in  \N, \quad \alk \in \reals^{\dimn}
\end{equation}
Here $\step >0$ is the constant step size, $\gradn\rew(\alk)$ is an unbiased noisy gradient evaluation, $\beta$ is an inverse temperature parameter, and $w_k$ is purposefully injected Gaussian noise. It is well known \cite{karatzas1991brownian} that the iterates $\{\alk\}$ generated by the SGLD algorithm \eqref{eq:gen_dtsgld} are asymptotically, \color{black}{ as $\step \to 0, \, k \to \infty$}, distributed according to the Gibbs measure 
\begin{equation}
\label{eq:gibbsd}
    \gibbs(\alpha) \propto \exp(-\beta \rew(\alpha))
\end{equation}
The SGLD algorithm \eqref{eq:gen_dtsgld} has several applications, e.g., in Bayesian learning \cite{welling2011bayesian} and neural network optimization \cite{li2016preconditioned}. The recent seminal work \cite{raginsky2017non} performs non-asymptotic (finite-sample) analysis of SGLD, providing guarantees on the proximity of the sample measure of \eqref{eq:gen_dtsgld} to the Gibbs measure \eqref{eq:gibbsd} after a finite run-time {\color{black}and with a fixed step-size}.
\subsubsection{Why Passive Stochastic Gradient Langevin Dynamics?}
Unfortunately, the above SGLD algorithm cannot be used in adaptive  IRL.
Adaptive IRL aims to reconstruct a cost function by passively observing a forward learner's stochastic gradient algorithm. Specifically, suppose a forward learner generates sequential stochastic gradients $\{\gradn\rew(\thk)\}_{k\in\N}$; the inverse learner (IRL) aims to reconstruct $\rew$ in real-time from $\{\gradn\rew(\thk)\}_{k\in\N}$. 

Observe that if the inverse learner could directly sample stochastic gradients from $\rew$, then employing SGLD \eqref{eq:gen_dtsgld} would suffice to reconstruct $\rew$ asymptotically, by recovering the Gibbs measure \eqref{eq:gibbsd} and taking the log-density of samples. However, the following issues render SGLD \eqref{eq:gen_dtsgld} unsuitable for IRL:
\begin{enumerate}
    \item[-] \textit{unknown cost function $\rew$}: By the IRL problem definition, the inverse learner does not know and cannot evaluate the cost function $\rew$. 
    \item[-] \textit{mis-specified noisy gradients}: The inverse learner must rely on uncontrolled and mis-specified stochastic gradients $\{\gradn\rew(\thk)\}_{k\in\N}$ evaluated by the forward learner. 
\end{enumerate}
For now, assume the uncontrolled noisy gradients $\{\gradn\rew(\thk)\}_{k\in\N}$ are evaluated at points $\{\thk\}_{k\in\N}$ drawn independently and identically distributed (i.i.d.) from a distribution $\idistr$. Then the inverse learner can incorporate these noisy gradient evaluations through the following \textit{passive} stochastic gradient Langevin dynamics (PSGLD) algorithm:
\begin{align}
\begin{split}
\label{eq:dt_psgld}
    \alkk =\ &\alk - \step\biggl[\K(\thk,\alk) \frac{\beta}{2}\gradn \rew (\thk) \color{black}{-} \grad \idistr(\alk)\biggr]\idistr(\alk) + \sqrt{\step}\idistr(\alk) w_k
\end{split}
\end{align}
where  $K(\cdot,\cdot)$ is a kernel function, $\step$ is the step size and $w_k$ is purposefully injected Gaussian noise. The kernel function plays a key role in the PSGLD algorithm; it weighs the influence of noisy gradient $\gradn\rew(\thk)$ to the current iterate $\alk$, by comparing the points $\thk$ and $\alk$. For instance, $K(\thk,\alk)$ should be large when $\thk$ and $\alk$ are close and small otherwise. A formal specification of the kernel $K(\cdot,\cdot)$ is provided in Section~\ref{sec:background}.

\vspace{0.2cm}
\textit{Remark. Kernel Function for Passive Stochastic Gradient Algorithms}: The use of a kernel function for passive stochastic gradient algorithms is widely studied \cite{revesz1977apply} \cite{hardle1987nonparametric}, \cite{nazin1989passive}, \cite{YY96}. In such algorithms, a noisy gradient is evaluated externally and is incorporated through the kernel function. Specifically suppose one is given a dataset \[\{\thk\in\reals^{\dimn},\gradn \rew(\thk)\}_{k\in\N}, \quad \thk \overset{i.i.d.}{\sim} \pi\]
where $\thk$ is sampled i.i.d. from some distribuion $\pi$ on $\reals^{\dimn}$, and one aims to estimate a local stationary point (e.g., minima) of the function $\rew$. {\color{black} Under reasonable assumptions} the following passive stochastic gradient algorithm accomplishes this:
\begin{equation*}
\label{eq:passive_sgd}
    \alkk = \alk - \step K(\thk, \alk)\gradn\rew(\thk)
\end{equation*}
One such assumption, informally, is that the kernel $K(\cdot,\cdot) : \reals^{2\dimn} \to \reals$ satisfies $K(\thk,\alk)$ is large only when $\|\thk-\alk\|$ is small.
The use of such algorithms historically focused on passive stochastic approximation for e.g., sequential non-parametric estimation of regression functions. In the same spirit of these works, \cite{krishnamurthy2021langevin} introduce a kernel function in the passive SGLD algorithm \eqref{eq:dt_psgld}, and reveal how \eqref{eq:dt_psgld} can be used to \textit{reconstruct the entire function $\rew$}, rather than only a local stationary point. $\blacksquare$
\vspace{0.2cm}

In \cite{krishnamurthy2021langevin} it is shown that the PSGLD algorithm \eqref{eq:dt_psgld} \textit{asymptotically samples from the Gibbs measure \eqref{eq:gibbsd}}. This allows us to achieve adaptive IRL; by passively observing the forward learner's stochastic gradients $\{\gradn\rew(\thk)\}_{k\in\N}$ and incorporating them into PSGLD \eqref{eq:dt_psgld} with a weighting kernel function $K(\thk,\alk)$, we asymptotically sample from the Gibbs measure \eqref{eq:gibbsd}. Then the cost function $\rew$ can be reconstructed by taking the log-sample density of asymptotic Markov chain Monte-Carlo (MCMC) samples. In this framework the iterates $\thk$ also need not be sampled i.i.d from a distribution, but can be generated through more complex dependencies such as a stochastic gradient algorithm.



As mentioned above, algorithm \eqref{eq:dt_psgld} was proposed in \cite{krishnamurthy2021langevin}. Also,  \cite{krishnamurthy2021langevin} gave an asymptotic weak convergence analysis.
\textit{In this paper we generalize the results \cite{raginsky2017non} to perform a non-asymptotic (finite-sample) analysis of the PSGLD algorithm~\eqref{eq:dt_psgld}}. Specifically, we provide non-asymptotic bounds on the distance between the PSGLD sample measure and the Gibbs measure~\eqref{eq:gibbsd}, utilizing the 2-Wasserstein metric. {\color{black}We also develop a \textcolor{black}{sequential} MCMC kernel density estimation algorithm and provide formal guarantees on its accuracy in reconstructing the cost function from empirical samples.} The derivation of these non-asymptotic bounds is highly non-trivial due to complexities in the generalized Langevin dynamics \color{black}{algorithm~\eqref{eq:dt_psgld}} and the incorporation of uncontrolled gradient estimates $\{\gradn\rew(\thk)\}_{k\in\N}$. These bounds provide:
\begin{enumerate}
    \item[-] extension of state-of-the-art finite-sample guarantees for Langevin dynamics algorithms \cite{raginsky2017non}, \cite{dalalyan2019user} to the generalized form \eqref{eq:dt_psgld}.
    \item[-] insight into the dependence of PSGLD sample measure convergence on its novel algorithmic structures, such as the kernel function $K(\cdot,\cdot)$ {\color{black} and the forward learning process.}
    \item[-] practical (finite-time) guarantees for adaptive IRL implementations {\color{black}using a kernel density estimation procedure}.
\end{enumerate}

\subsection{Non-Asymptotic Analysis. Extension of \cite{raginsky2017non}} 


The purpose of this section is to compare the non-asymptotic analysis provided in this paper to existing works. Non-asymptotic analysis of SGLD was initiated by \cite{raginsky2017non}, and has been followed up by several extensions \cite{dalalyan2019user}, 
\cite{cheng2018underdamped}. In our case the PSGLD algorithm is \textit{passive}; the stochastic gradient evaluations $\{\gradn\rew(\thk)\}_{k\in\N}$ are uncontrolled and incorporated by a weighting kernel $K(\alk,\thk)$. The PSGLD algorithm also utilizes the generalized Langevin dynamics form \eqref{eq:dt_psgld} incorporating the initial distribution $\idistr$. These structural differences introduce substantial complexity to the analysis. Non-asymptotic bounds for classical SGLD rely in part on structural assumptions about the stochastic gradient evaluation, including unbiasedness. The PSGLD incorporation of externally evaluated stochastic gradients through a weighting kernel introduces substantial bias which is dependent in a complex way on the relation between the PSGLD algorithm, the process generating evaluations $\{\gradn\rew(\thk)\}_{k\in\N}$, and the cost function $\rew$ itself. \cite{dalalyan2019user} analyzes classical SGLD with inaccurate (biased) gradients, but operates under the assumption of a strongly log-concave target distribution, which is far too restrictive for our IRL framework. Thus, one challenge of our analysis is gaining traction on controlling these biases without introducing restrictive structural assumptions. Another key difference is the utilization of a generalized Langevin dynamics form \eqref{eq:gen_dtsgld}. This necessitates analysis of a generalized continuous-time Langevin diffusion, for which we prove several structural results using tools from the study of Markov diffusion operators \cite{bakry2014analysis}. 

\subsection{Generalized Inverse Learning} 
In this section we discuss the relation between our adaptive IRL framework and traditional IRL frameworks. 
Traditional IRL \cite{ng2000algorithms}, \cite{HRA16}, \cite{BGN19} reconstructs the cost function of a Markov Decision Process (MDP) by observing decisions taken from an optimal policy, i.e., \textit{after} an observed agent has completed learning the optimal policy. Here, we consider \textit{real-time} (adaptive) IRL in a more general setting. We observe an agent (forward learner) performing stochastic gradient descent (e.g, policy gradient reinforcement learning) on a cost function $\rew$, and reconstruct $\rew$ in real-time. {\color{black} This paper provides finite-sample guarantees on this reconstruction: Theorem~\ref{thm:main1} bounds the 2-Wasserstein distance between the PSGLD sample measure and the Gibbs measure \eqref{eq:gibbsd}. \textit{Standing alone, Theorem~\ref{thm:main1} contributes to the field of passive stochastic gradient algorithms \cite{YY96}, \cite{nazin1989passive}, \cite{revesz1977apply}, \cite{hardle1987nonparametric}, by providing sample-complexity analysis for the generalized passive Langevin algorithm \eqref{eq:dt_psgld}}. However, we also provide a \textcolor{black}{sequential} MCMC kernel density estimation procedure for reconstructing the cost function from empirical samples, and give a formal guarantee on this approximation in Theorem~\ref{thm:kernest}. Thus, our methodology and formal guarantees extend to the practical adaptive IRL setting.

As far as we are aware there is not an abundant literature on adaptive IRL. However, given the prevalence of modern machine learning systems it is only natural to consider the transient learning regime, since this phase is typically time-consuming and resource intensive. PSGLD is a natural candidate for accomplishing non-parametric IRL in the transient regime, and is motivated by the aforementioned kernel-based passive stochastic gradient algorithms.} 

Notice that the forward learner need not be restricted to an MDP setting, but acts as a general stochastic gradient algorithm. The adaptive IRL framework can thus be considered an inverse stochastic gradient algorithm, and can be applied in a variety of settings, including adaptive Bayesian learning, constrained MDPs, and logistic regression classification \cite{krishnamurthy2021langevin}. In Section~\ref{sec:examples} we illustrate the application to several of these settings, including a canonical policy-gradient reinforcement learning algorithm.

\textcolor{black}{In Section~\ref{sec:fwdlrn} we detail the a canonical forward learning process which satisifies our structural requirements, and takes the form of a \textit{re-initializing} stochastic gradient optimization process. This re-initialization is a suitable model for \textit{non-convex optimization}, in the context of the adaptive learning contexts outlined above: Bayesian learning, etc. Interestingly, in Section~\ref{sec:fwdlrn} we also show how this re-initialization can be recovered from multi-agent learning dynamics, in particular federated stochastic gradient algorithms \cite{yuan2020federated}. Thus, we exemplify the capability of our algorithm to apply to both sequential single-agent learning and \textit{decentralized multi-agent learning dynamics}. }


\subsection{Main Result and Proof Technique}
\label{sec:mainresinf}
 Recall $J: \reals^{\dimn} \to \reals$ denotes the cost function being optimized by the forward learner, and which we aim to reconstruct. Denote 
 \begin{equation}
 \label{eq:dtlaw_def}
     \dtlaw := \textrm{Law}(\alk)
 \end{equation}
 i.e., $\dtlaw$ is the sampling measure of PSGLD algorithm at iterate $k\in\N$. Recall $\gibbs$ is the Gibbs measure $\gibbs \propto \exp(-\beta\rew)$, and denote $\wass(\dtlaw,\gibbs)$ the 2-Wasserstein distance between $\dtlaw$ and $\gibbs$. The 2-Wasserstein distance will defined and motivated in Section~\ref{sec:mainresults}, but for now note that it is a metric on the space of measures. The PSGLD algorithm samples from $\gibbs$ when $\wass(\dtlaw,\gibbs)=0$, and this suffices to reconstruct $\rew$ by taking the logarithm of empirical sample density. However by \cite{krishnamurthy2021langevin} $\wass(\dtlaw,\gibbs)=0$ is only obtained {\color{black}in the weak limit as $\step \to 0, k \to \infty$}, so we aim to quantify the proximity of $\dtlaw$ to $\gibbs$ after a finite-time in terms of $\wass(\dtlaw,\gibbs)$.
 
\vspace{0.1cm}
\textit{Main Result (Informal)}: The 2-Wasserstein distance between PSGLD law $\dtlaw$ and Gibbs measure \eqref{eq:gibbsd} scales as 
\[\wass(\dtlaw,\gibbs) \leq \mathcal{O}(k\step\sqrt{\step} + \exp(-k\step))\]
where $\step$ is the algorithm step size. Thus, for any $\hlb>0$, we can take the algorithm iteration number $k$ large enough and step size $\step$ small enough so that 
\[\wass(\dtlaw,\gibbs) \leq \mathcal{O}(\hlb)\]
where the order notation $\mathcal{O}$ hides dependencies on other algorithmic parameters. The {\color{black}first} main result of this paper, given as Theorem~\ref{thm:main1}, is a precise formulation of this statement.

{\color{black} Then, we propose a \textcolor{black}{sequential} MCMC kernel density estimation algorithm for reconstructing an estimate of $\dtlaw$ from empirical PSGLD samples, and hence an estimate $\hat{\rew}$ of the cost function $\rew$. Our second main result, Theorem~\ref{thm:kernest}, gives a concentration inequality bound on the $L^1$ distance between $\hat{J}$ and $J$, in terms of algorithmic parameters. We show that this $L^1$ distance can be made arbitrarily small with high probability by choosing the PSGLD and kernel density estimation parameters suitably.}
\vspace{0.1cm}


\textit{Proof Technique}: We bound $\wass(\dtlaw,\gibbs) \leq \wass(\dtlaw,\ctlaw) + \wass(\ctlaw,\gibbs)$, where $\ctlaw$ is the measure, at time $k\step$, of a particular continuous time diffusion with stationary measure $\gibbs$. We obtain $\wass(\dtlaw,\ctlaw) \leq \mathcal{O}(k\step\sqrt{\step})$ through a Girsanov change of measure technique and a weighted transportation cost inequality. We then show that the diffusion satisfies a logarithmic-Sobolev inequality, allowing us to employ exponential decay of entropy and the Otto-Villani Theorem to show exponential decay of $\wass(\ctlaw,\gibbs)$. This proof structure is mirrored in the seminal work \cite{raginsky2017non}. However, our algorithm necessitates a  non-trivial extension of the methods in \cite{raginsky2017non}; we utilize both a generalized stochastic gradient Langevin dynamics form and a weighting kernel to control the external gradient evaluations. The generalized form disrupts absolute continuity of measure between the algorithm and continuous time diffusion, necessitating the introduction and control of an intermediate process in order to apply Girsanov's Theorem. It also necessitates control of the "sampling distribution" (from which initial SGD and PSGLD points are taken) to decrease discretization error. However, this control simultaneously increases a relative entropy term appearing in the final 2-Wasserstein bound; this is handled by a careful specification of other algorithmic parameters. Finally, the continuous time diffusion is also distinct from that in \cite{raginsky2017non}, so we must prove logarithmic-Sobolev inequality satisfaction by a novel Lyapunov function. Furthermore, there are a variety of supporting lemmas, such as exponential integrability of the generalized diffusion, which are required for our analysis.

{\color{black} The proof of the subsequent $L^1$ cost reconstruction bound utilizes a uniform concentration inequality for multi-dimensional nonparametric kernel density estimators \cite{vogel2013uniform} and a relation between the 2-Wasserstein distance and a negative order Sobolev norm \cite{peyre2018comparison}, which can be translated to the $L^1$ distance.}

\subsection{Organization}
Section~\ref{sec:fwdlrn} provides a discussion of structural assumptions for the forward learning process which will allow oour inverse learning efforts to work. Section~\ref{sec:background} provides background on passive stochastic gradient Langevin dynamics and the asymptotic result of \cite{krishnamurthy2021langevin}.  Section~\ref{sec:mainresults} discusses our \textcolor{black}{first} main result (Theorem~\ref{thm:main1}): a non-asymptotic bound on the 2-Wasserstein distance \textcolor{black}{between the PSGLD sampling measure and the stationary Gibbs measure encoding the desired cost function.} {\color{black}In Section~\ref{sec:cfr} we provide our second main result (Theorem~\ref{thm:kernest}):} a cost function reconstruction concentration inequality (Theorem~\ref{thm:kernest}), and we give examples of how our procedures interface with canonical forward learning processes such as reinforcement learning. 
\textcolor{black}{Section~\ref{sec:examples} provides several practical examples of forward learning processes that fit our assumptions, and discusses the interface between these and our inverse learning scheme.} Section~\ref{sec:proofprelim} provides additional mathematical background for the proof of Theorem~\ref{thm:main1}, and Section~\ref{sec:resultspf} provides details on the structure of this proof. The complete details of all proofs appear in the Appendix. 


\section{Modeling the Forward Learner's Data Generation Process}
\label{sec:fwdlrn}


Our adaptive IRL procedure comprises a forward learner and an inverse learner. The forward learner evaluates noisy gradients of a cost function, while the inverse learner observes the forward learner’s gradients and attempts to reconstruct the cost function. In order to provide non-asymptotic guarantees for the inverse learner’s reconstruction, we require assumptions the forward learner. In this section, we present the minimal set of assumptions. We also show that practical learning algorithms including multiple-restart non-convex optimization and multi-agent federated optimization satisfy these assumptions. 

{\color{black}\subsection{Assumptions on Forward Learner's Data Generation Process} The forward learner generates stochastic gradient evaluations of a non-negative cost function
\begin{equation}
\label{eq:rew}
    \rew\colon \reals^{\dimn} \to \reals_+
\end{equation}
sequentially over a finite time horizon. Specifically, the forward learner produces 
\begin{equation}
\label{eq:dgproc}
\{\gradn\rew(\thk), \thk\in\reals^{\dimn}\},\,\, k\in [\ks] := 1,\dots,\ks
\end{equation}
where $\gradn\rew(\thk)$ denotes a noisy gradient of cost function $\rew$ evaluated at iterate $\thk$. Let $\mathcal{F}_k$ be the $\sigma$-algebra generated by $\{\theta_i, i=1,\dots,k\}$. A random variable $\tau$ is a \textit{stopping time with respect to the filtration $\{\mathcal{F}_k, k\in\N\}$} if $\{\tau = k\} \in \mathcal{F}_k\, \forall k\in\mathbb{N}$, i.e., the event that $\tau$ takes value $k \in\N$ is measurable with respect to $\mathcal{F}_k$, for any $k\in\N$. We define the terminal iteration $\ks \in \N$ to be a stopping time with respect to the filtration $\{\mathcal{F}_k,k\in\N\}$. 



\textit{Minimal Requirements}: The minimal requirements on this sequential data generation ("forward learning") process \eqref{eq:dgproc} are given by the following assumption:

\begin{myassumptions}
\item[\nl{as:fwdlrn}{A1}] The forward learning process produces sequential stochastic gradients $\gradn\rew(\thk)$ such that the following structural properties hold:
\begin{enumerate}[label=\roman*)]
    {\setlength\itemindent{25pt} \item $\exists \,\sgdtm > 0 \colon \, \min_{k\in[\ks]} \CE\|\gradn\rew(\thk)\| \geq \sgdtm$}
    {\setlength\itemindent{25pt} \item $\exists \,b_2 > 0 \colon \, \min_{k\in[\ks]}\textnormal{Var}(\grad\rew(\thk) | \theta_{k-1},\dots,\theta_0) \geq b_2$}
    {\setlength\itemindent{25pt} \item $\exists \,M_{\theta} < \infty \colon \,\, \max_{k\in[\ks]} \CE\|\thk\|^2 \leq M_{\theta}$ }
    {\setlength\itemindent{25pt} \item $\theta_0 \overset{i.i.d.}{\sim} \idistrw$, where }
\end{enumerate}
\begin{equation}
\label{eq:idistrw}
\idistrw(x) := \frac{\idistr(\frac{x}{\idw})}{\int_{\reals^{\dimn}}\idistr(\frac{y}{\idw})dy}
\end{equation}
\end{myassumptions}
Here $\pi_0$ is an arbitrary density function and $\idw$ is a scale parameter \footnote{This construction is purely for notational convenience.}. In words, we require that $(i)$ the expected noisy gradient magnitude is uniformly lower bounded, $(ii)$ the conditional variance of each gradient evaluation is uniformly lower bounded, $(iii)$ the process $\{\thk, k\in\N\}$ is uniformly bounded in its expected square magnitude, and $(iv)$ the initial evaluation $\theta_0$ is sampled $i.i.d.$ from a known sampling distribution. We list these along with our other formal assumptions in Section~\ref{subsec:mainresults}. An interesting perspective is that conditions $(ii)$ and $(iii)$ are necessary to control the \textit{exploration - exploitation} tradeoff. Condition $(ii)$ requires that the stochasticity in the process has sufficient variance to induce \textit{exploration} of the domain, while condition $(iii)$ limits this exploration such that our kernel-based algorithm (to be introduced in Section~\ref{subsec:IRLintro}) effectively contributes non-negligible weights. We discuss the appropriate perspective for the conditions $i)-iv)$ in more depth in Section~\ref{subsec:fwdprocdep}.

{\color{black}
\subsection{Stochastic Gradient Descent Process satisfying Data-Generation Assumptions}
\label{subsec:risgd}

For the purposes of adaptive IRL, it is natural to model the forward learner as a stochastic gradient descent algorithm; we must show how requirements $i)-iv)$ can be met. If the SGD runs indefinitely, than requirement $(i)$ may be violated as the expected gradient may converge to zero. Thus, it is useful to introduce a re-initialization into this stochastic gradient process, such that it restarts from an initial distribution after random stopping times. This is also nicely motivated by practical learning algorithms. This re-initialization consists of three distinct probabilistic components: a sequence of  stopping times, re-initialization samplings, and a terminal stopping time condition.}
\begin{enumerate}[label=\roman*)]
    \item \textcolor{black}{\textit{Threshold-Driven Stopping-Time Formulation}: Let $\{\tau_n,n\in\N\}$ be a sequence of stopping times with respect to $\{\mathcal{F}_k\, k\in\N\}$, such that $\tau_1 < \tau_2, < \dots $. Then, we consider learning processes which re-initialize at each stopping time $\tau_n, n\in\N$. A well-motivated choice for stopping time $\tau_n$ is the time at which the stochastic gradient magnitude of the $n$'th SGD "run" is below a certain threshold $\sgdtm$. This models e.g., non-convex optimization processes which re-initialize after converging to a neighborhood of a local stationary point, and is demonstrated in Algorithm~\ref{alg:sgd}. }
    
    \item \textcolor{black}{\textit{Re-Initialization Sampling}: Upon re-initialization, the forward learning process starts at an $i.i.d.$ sample from sampling distribution $\idistrw$ \eqref{eq:idistrw}. }

    \item \textcolor{black}{\textit{Termination}: Let $\{\mathcal{T}_k, k\in\N\}$ be a sequence of binary random variables taking values in $\{0,1\}$ and each depending on the filtration $\{\mathcal{F}_i, i\leq k\}$. The algorithm terminates at $\ks := \min_{k\in\N}\{\mathcal{T}_k = 1\}$. Thus $\mathcal{T}_k = 1$ if and only if the algorithm terminates at $\ks = k$, and equals $0$ otherwise. Since each $\mathcal{T}_k$ depends on the entire filtration $\{\mathcal{F}_i, i\leq k\}$, Algorithm~\ref{alg:sgd} can incorporate any causal stopping rule, e.g., terminate when the non-convex cost landscape has been sufficiently explored.}
\end{enumerate} 

Combining these three probabilistic components, we consider the following re-initializing stochastic gradient process,
\begin{align}
\begin{split}
\label{eq:sgd}
    &\thkk = \thk - \sgdstep\gradn \rew (\thk), \,\, k \in \{\stn,\dots, \stnn - 1\}, \quad \textcolor{black}{\theta_{\tau_n} \overset{i.i.d.}{\sim} \idistrw, \,\,n\in\mathbb{N}}
\end{split}
\end{align}
which in practice terminates at $\ks = \min_{k\in\N}\{\mathcal{T}_k=1\}$. Here $\sgdstep >0$ is a fixed step-size, and $\gradn\rew(\thk)$ is an unbiased estimate of the true gradient $\grad\rew(\thk)$, with bounded variance. $n\in \N$ represents each "run" of the SGD.  Algorithm~\ref{alg:sgd} displays this re-initializing stochastic gradient descent operation.

\setlength{\textfloatsep}{5pt}
\begin{algorithm}
\caption{Forward Learner. Re-Initializing SGD Process}\label{alg:sgd}
\begin{algorithmic}[1]
\State initialize $k=0$ 
\State \textcolor{black}{determine terminating condition $\mathcal{T}_k: \{\mathcal{F}_i, i\leq k\} \to \{0,1\}$}
\While{$\mathcal{T}_k = 0$}
    \State $\theta_{k} \overset{iid}{\sim} \idistrw$
    \While{$\|\gradn\rew(\thk)\| \geq \sgdtm$}
        \State $\thkk \gets \thk - \sgdstep\gradn \rew (\thk)$ 
        \State $k = k+1$
        \State \textcolor{black}{set $\mathcal{T}_k: \{\theta_0, \dots,\thk\} \to \{0,1\}$} \textcolor{black}{\algorithmiccomment{Set to 1 when terminating condition is satisfied}}
        \textcolor{black}{\If{$\mathcal{T}_k = 1$}
        \State $k^* = k$, \text{break}
        \EndIf}
    \EndWhile
\EndWhile
\end{algorithmic}
\end{algorithm}

Algorithm~\ref{alg:sgd} is general enough to incorporate a variety of important learning processes, including for example: Bayesian learning, logistic regression classification, constrained MDP optimization, and empirical risk minimization. We discuss several examples at length in Section~\ref{sec:mainresult}.  Algorithm~\ref{alg:sgd} also satisfies the requirements $i)-iv)$:
\begin{itemize}
    \item [-] \ref{as:fwdlrn}-$i)$ is satisfied since the algorithm only operates while $\|\gradn\rew(\thk)\| \geq \sgdtm$.
    \item[-] \ref{as:fwdlrn}-$ii)$ is satisfied by Lemma~\ref{lem:sgdvarbd}, presented in Appendix~\ref{subsec:techres}. 
    \item[-] \ref{as:fwdlrn}-$iii)$ is satisfied by Lemma~\ref{lem:unifL2}, presented in Appendix~\ref{subsec:techres}.
    \item[-] \ref{as:fwdlrn}-$iv)$ is satisfied by design.
\end{itemize}


\subsection{Examples of  Randomly Re-Initialization Stochastic Gradient Algorithms} The re-initialization of Algorithm~\ref{alg:sgd} is motivated in practice by, for instance, non-convex optimization and multi-agent learning. We discuss other examples of forward learning processes in more detail in Section~\ref{sec:examples}.

{\em Example 1. Non-Convex Optimization}: 
In non-convex optimization, it is standard practice to introduce re-initialization in order to sufficiently explore the domain and converge to different local minima. See \cite{li2023restarted} for a state-of-the-art account of re-initialization in nonconvex optimization procedures. In this case the terminal condition $\{\mathcal{T}_k = 1\}$ typically corresponds to the time $k$ at which the cost landscape has been sufficiently explored.


{\em Example 2. Federated Stochastic Gradient Descent}: Federated SGD\footnote{We distinguish this with classical "federated learning", which typically deals with multiple agents each optimizing a unique cost function.} \cite{yuan2020federated}, \cite{mcmahan2017communication}, \cite{mangasarian1995parallel}, exploits parallel, decentralized computing resources (e.g., distributed servers) to perform the stochastic optimization $\min_{w\in\reals^{\dimn}} := \CE_{\zeta \sim D}[F(w,\zeta)]$ for cost function $F$ and noise distribution $D$. It is assumed that $M$ parallel workers may each access $\grad f(w,\zeta)$ at any $w$ with independent noise realizations $\zeta \sim D$. Each worker $i$ performs stochastic gradient descent until some (in general probabilistic, often threshold-driven) stopping time $\tau_i$, producing final estimate $w_i \in \reals^{\dimn}$. These $w_i, i\in[M]$ may be averaged to output a more precise estimate of $\min_{w\in\reals^{\dimn}} := \CE_{\zeta \sim D}[F(w,\zeta)]$. Some sophisticated algorithms also use intermittent synchronization among the workers. 
    
    Observe how this can be mapped into our re-initializing framework of Algorithm~\ref{alg:sgd}. The parallel SGD streams may be processed sequentially: the first agent's entire data stream is processed, then after its termination the next agent's entire data stream is processed, etc. This sequential processing introduces a natural "re-initialization". The terminal condition $\mathcal{T}_k = 1$ holds if and only if all agents' data streams have been processed at iteration $k$.

\textit{Non-Re-Initializing Processes}. Algorithm~\ref{alg:sgd} can also model \textit{non-re-initializing} processes by setting $\sgdtm$ (in \ref{as:fwdlrn}, i) sufficiently small such that the terminating point $\ks$ occurs before the event $\{\|\gradn\rew(\thk)\| < \sgdtm\}$ with high probability. The impact of such a forward learning process (with $\sgdtm$ very small) is discussed in Section~\ref{subsec:fwdprocdep}. \cite{krishnamurthy2021langevin} motivates the practical implementation details further. Other example forward learning processes which fit our assumptions, including MDP optimization, Bayesian learning and empirical risk minimization, are outlined in more detail in \ref{sec:examples}. 

}

\section{Passive Langevin Dynamics for Adaptive Inverse Reinforcement Learning}
\label{sec:background}
\textcolor{black}{Previously we introduced structural assumptions on the forward learning process, which will enable our \textit{inverse learning} procedure to accurately reconstruct the desired cost function. In this section we introduce our passive stochastic gradient Langevin dynamics (PSGLD) algorithm which will be used for inverse learning. We first discuss Langevin dynamics, then present the specific adaptive IRL setting and PSGLD algorithm, and discuss the weak convergence asymptotic analysis of the PSGLD algorithm in \cite{krishnamurthy2021langevin}. Our main contributions will begin in the following section, where we provide \textit{finite-sample} bounds for this inverse learning procedure.}

\subsection{Stochastic Gradient Langevin Dynamics}
The classical stochastic gradient Langevin dynamics (SGLD) algorithm is given, with step size $\step_k$, objective function $\rew$, noise parameter $\beta$, and i.i.d. standard $\dimn$-variate Gaussian noise $w_k$,  as 
\begin{equation}
    \label{eq:LD}
    \thkk = \thk - \step_k \gradn \rew (\thk) + \sqrt{2 \step_k \beta^{-1}}\noisek, \quad k\in  \N
\end{equation}
Here $\theta_k \in \reals^{\dimn}$ is initialized by $\theta_0 \sim \idistr$ for some sampling distribution $\idistr$ on $\reals^{\dimn}$, and $\gradn\rew(\cdot)$ is an unbiased gradient estimate. The algorithm \eqref{eq:LD} is used primarily for either non-convex optimization \cite{xu2018global} or to sample from probability distributions via MCMC \cite{welling2011bayesian}. The former is accomplished by treating \eqref{eq:LD} as a simulated annealer, and letting the step size $\step_k$ and 'temperature' $\beta^{-1}$ decrease to zero as $k\to \infty$. To accomplish the latter, the step size $\step_k$ and temperature $\beta^{-1}$ are \textit{fixed} for all $k$. \textit{This work considers the latter case of constant step-size SGLD}, with $\step_k = \step \ \forall k \in \N$. It is well known that the Markov process \eqref{eq:LD}, with constant step-size, asymptotically samples from the Gibbs measure \eqref{eq:gibbsd}.
Indeed, \eqref{eq:LD} corresponds to a discretization of the continuous-time Langevin diffusion given by, with $\theta(t) \in \reals^{\dimn}$ and $\brownt$ standard Brownian motion in $\reals^{\dimn}$, the It\^o stochastic differential equation (SDE)
\begin{equation}
\label{eq:langdiff}
    d\theta(t) = -\grad \rew(\theta(t))dt + \sqrt{2\beta^{-1}}d\brownt, \quad t\geq 0,
\end{equation}
Under suitable conditions on $\rew$ and $\beta$, this SDE has the Gibbs measure \eqref{eq:gibbsd} as its unique stationary measure \cite{chiang1987diffusion}.  \color{black}{This asymptotic convergence is put more precisely in \cite{borkar1999strong}, where it is shown that the recursion \eqref{eq:LD} has law which converges, as $k \to \infty$, to any KL-divergence neighborhood of \eqref{eq:gibbsd} for sufficiently small $\step$. Thus, \textcolor{black}{for any $\delta > 0$, the step size $\epsilon >0$ can be chosen sufficiently small so that the SGLD algorithm \eqref{eq:LD} will asymptotically, as $k \to \infty$, sample from a distribution whose KL-divergence from the Gibbs measure \eqref{eq:gibbsd} is less than $\delta$}. So} \eqref{eq:LD} can be used as a MCMC algorithm to asymptotically, \textcolor{black}{as $k \to \infty$}, sample \color{black}{arbitrarily close to} any probability distribution which can be expressed as \eqref{eq:gibbsd} with some potential function $\rew$. \textcolor{black}{Put differently, \eqref{eq:LD} will sample from \eqref{eq:gibbsd} in the asymptotic sense as \textit{both} $k\to\infty$, $\step\to 0$.}

In \cite{stramer1999langevin} more general reversible diffusions of the form, with $\sigma: \reals^{\dimn} \to \reals$ differentiable,  
\begin{equation}
\label{eq:stram_diff}
    d\theta(t) = \left[-\frac{\beta}{2}\gradn\rew(\theta)dt \,\color{black}{+}\, \grad\sigma(\theta)dt + d\brownt \right]\sigma(\theta)
\end{equation}
are studied, and it is shown \color{black}{in \cite{krishnamurthy2021langevin}} that \eqref{eq:stram_diff} has the same stationary measure \eqref{eq:gibbsd} as the classical Langevin diffusion \eqref{eq:langdiff}. {\color{black} We list this as a technical result, given by Lemma~\ref{lem:gendiffstat} in Appendix~\ref{ap:techres}.}

The corresponding Euler-Maruyama time discretization of \eqref{eq:stram_diff} results in the following discrete-time Markov process 
\begin{align}
\label{eq:stram_disc}
    &\thkk = \thk - \step\left[\frac{\beta}{2}\grad\rew(\thk) \,\color{black}{-}\, \grad\sigma(\thk) \right]\sigma(\thk) + \sqrt{\step}\sigma(\thk)w_k
\end{align}
which can thus equivalently be used as a MCMC sampler, \color{black}{in the sense described for the simpler SGLD \eqref{eq:LD}}, from \eqref{eq:gibbsd}. We will  utilize the generalized process \eqref{eq:stram_disc}, as opposed to the classical SGLD \eqref{eq:LD}, for our PSGLD algorithm\footnote{See \cite{krishnamurthy2021langevin} for motivation}. 

\textit{Active vs. Passive Gradient Evaluation}: Notice that the above SGLD algorithms utilize, at each time step, the unbiased gradient $\gradn\rew(\thk)$ evaluated at the current iterate $\thk$. We term this \textit{active} gradient evaluation, and distinguish this from \textit{passive} gradient evaluation, where the gradient is evaluated at a different (uncontrolled) point. The following section introduces the adaptive IRL setting and motivates the need for passive gradient evaluation in our PSGLD algorithm.

\subsection{Inverse Learning through Passive Stochastic Gradient Langevin Dynamics}
\label{subsec:IRLintro}
\textcolor{black}{This section motivates the setting of adaptive IRL. Previously we introduced structural assumptions, and a pertinent model, for the forward learner. Here we introduce the inverse learning setting and PSGLD algorithm.}

\subsubsection{Inverse Learning: PSGLD}
In this paper we take the perspective of an inverse learner who observes the SGD process \eqref{eq:sgd}, and attempts to reconstruct the cost function $\rew$ being optimized.
We assume this observer knows the sampling distribution $\idistrw$ and can observe evaluations $\thk, \,k \in \N$. The agent recovers noisy gradient evaluations $\gradn \rew(\thk) = \frac{\thkk - \thk}{\sgdstep}$.

Using only these sequential noisy gradient evaluations, how can the agent learn $\rew$? This is accomplished via MCMC sampling, using the following
\textit{passive stochastic gradient Langevin dynamics (PSGLD)} updates:
\begin{align}
\begin{split}
\label{eq:dt_sgld}
    \alkk =\ &\alk - \step\biggl[\Kd(\thk-\alk)\frac{\beta}{2} \gradn \rew (\thk) \,\color{black}{-}\, \grad \idistrw(\alk)\biggr]\idistrw(\alk) + \sqrt{\step}\idistrw(\alk) w_k \\
    &\alpha_0 \sim \idistrw
\end{split}
\end{align}
Note that $\alpha_0$ is sampled randomly from the sampling distribution $\idistrw$ of the SGD process \eqref{eq:sgd}. Here $\{w_k, k\geq 0\}$ is an i.i.d. sequence of standard $N-$variate Gaussian random variables, 
\begin{equation}
    \label{eq:kDel}
    \Kd(\thk-\alk) := \delstep\kernel
\end{equation} is the $\Delta$-parametrized kernel function, and $\beta$ is the inverse temperature parameter. The algorithm is \textit{passive} since the stochastic gradients $\gradn\rew(\thk)$ and evaluation points $\thk$ are passively observed from SGD process \eqref{eq:sgd}. The kernel\footnote{An example kernel function is the multivariate normal $\gaus(0,\sigma^2I_{\dimn})$ density with $\sigma = \Delta$, i.e.,
$\delstep \K(\frac{\theta-\alpha}{\Delta}) = (2\pi)^{-\dimn/2}\Delta^{-\dimn}\exp(-\frac{\|\theta-\alpha\|^2}{2\Delta^2})$} function $\K(\cdot)$ controls for bias in these passive gradient evaluations, and can be chosen by the observer as any function $\K: \reals^{\dimn} \to \reals$ satisfying: 
\begin{align}
\begin{split}
\label{eq:Kspec}
    &\K(u) \geq 0, \quad \K(u) = \K(-u), \quad \sup_u \K(u) < \infty, \\ &\int_{\reals^{\dimn}}\K(u)du = 1, \quad \int_{\reals^{\dimn}}|u|^2\K(u) < \infty
\end{split}
\end{align}
$\K_{\Delta}$ weights the relevance of stochastic gradient $\gradn\rew(\thk)$ to the current iterate $\alt$. We obtain $\Kd$ by modulating $\K$ by the domain scaling parameter $\Delta$ as \eqref{eq:kDel}. So $\Delta$ modulates the degree to which samples $\thk$ at a fixed distance from current iterate $\alt$ impact the algorithm's evolution.
 
Algorithm~\ref{alg:psgld} displays this passive stochastic gradient Langevin dynamics algorithm, which takes as input the sequential evaluations $\thk$ made in Algorithm~\ref{alg:sgd}.

\begin{algorithm}
\caption{PSGLD for IRL}\label{alg:psgld}
\begin{algorithmic}[1]
\State parameters: step size $\step$, inverse temperature $\beta$, kernel scale $\Delta$, re-sampling distribution scale $\idw$
\State initialize $\alpha_0 \sim \idistrw$
\While{$k \geq 0$}
        \State obtain $\thk$ from Algorithm~\ref{alg:sgd}
        \If{$k\geq 1$}
        \State $\gradn\rew(\thk) = \frac{1}{\sgdstep}(\thk-\theta_{k-1}),\quad K_{k-1} = \delstep K(\frac{\theta_{k-1} - \alpha_{k-1}}{\Delta})$
        \State sample $w_k \sim \gaus(0,I_{\dimn})$
        \State  $\alk \gets \alpha_{k-1} - \step\biggl[K_{k-1} \frac{\beta}{2} \gradn \rew (\thk) \,\color{black}{-}\, \grad \idistrw(\alpha_{k-1})\biggr]\idistrw(\alpha_{k-1})+ \sqrt{\step}\idistrw(\alpha_{k-1}) w_k$
        \EndIf
\EndWhile
\end{algorithmic}
\end{algorithm}
{\color{black}\textit{Achieving adaptive IRL with Algorithm~\ref{alg:psgld}}: Recall $\dtlaw$ denotes the sample measure of Algorithm~\ref{alg:psgld} at iterate $k$, i.e., $\alk \sim \dtlaw$. Algorithm~\ref{alg:psgld} provides a principled approach for achieving adaptive inverse reinforcement learning, since we can take the algorithmic parameters appropriately in order to make $\dtlaw$ arbitrarily close to $\gibbs$. Theorem~\ref{thm:main1} is a precise formulation of this statement. It utilizes the 2-Wasserstein distance, introduced in the next subsection, to measure the proximity between these measures. Then, given $\dtlaw$ close to $\gibbs$, we can approximately reconstruct $\rew$ by taking the logarithm of the empirical sample density of Algorithm~\ref{alg:psgld}. Algorithm~\ref{alg:costrec} achieves this reconstruction by introducing a kernel density estimation procedure, and Theorem~\ref{thm:kernest} provides formal guarantees on the reconstructed cost function approximation.} Next a background result is presented which motivates Algorithm~\ref{alg:psgld} as the natural approach for achieving adaptive IRL.

\subsection{Passive SGLD: Asymptotic Convergence} 
 \cite{krishnamurthy2021langevin} provides the following weak convergence analysis  for the PSGLD algorithm.
\begin{proposition}[Weak Convergence \cite{krishnamurthy2021langevin}]
\label{prop:wca}
    Let $\altm = \alk$ for $t \in [\step k, \step (k+1)]$ be the continuous-time interpolation of PSGLD \eqref{eq:dt_sgld}. Under assumptions (A1)-(A4) of \cite{krishnamurthy2021langevin}, the process $\altm$ converges weakly to the solution of the stochastic differential equation
    \begin{gather}
    \begin{aligned}
    \begin{split}
    \label{eq:ct_diff}
    &d\alt = -\left[\frac{\beta}{2}\idistrw^2(\alt)\grad\rew(\alt) \,\color{black}{-}\, \grad\idistrw(\alt)\idistrw(\alt)\right]dt + \idistrw(\alt)d\brownt \\
    &\alpha(0) = \alpha_0 \sim \idistrw
    \end{split}
    \end{aligned}\raisetag{0\baselineskip}
    \end{gather}
where $\brownt$ is standard $\dimn$-dimensional Brownian motion. Furthermore, the stochastic differential equation \eqref{eq:ct_diff} has $\gibbs$ \eqref{eq:gibbsd} as its stationary distribution. 
\end{proposition}

Thus, \color{black}{\eqref{eq:dt_sgld} generates \textit{asymptotic} samples} 
\begin{equation}
\label{eq:as_samp}
    \color{black}\lim_{k\to\infty, \, \step \to 0}\alk \sim \gibbs(\alpha) \propto \exp(-\beta\rew(\alpha))
\end{equation} allowing for asymptotic reconstruction of $\rew$ by 
\begin{equation}
    \label{eq:Jrec}
    \rew = -\log(\gibbs)/\beta
\end{equation}
\textbf{Motivation}: 
Proposition~\ref{prop:wca} shows that Algorithm~\ref{alg:psgld} asymptotically produces samples from the Gibbs measure, \\$\color{black}\lim_{k\to\infty,\step \to 0}\alk \sim \gibbs$, and so the cost function $\rew$ can be reconstructed from the logarithm of the asymptotic sample density. However, in this paper we are interested in quantifying how well this sampling algorithm approximates the Gibbs measure after a \textit{finite run-time} \color{black}{and with \textit{non-zero} step size $\step$}. Our main result gives non-asymptotic (finite-sample) bounds on the 2-Wasserstein distance  between the distribution of the sampling density produced by Algorithm~\ref{alg:psgld} and the Gibbs measure $\gibbs$ \eqref{eq:gibbsd}.

{\color{black}\textit{Logarithmic-Sobolev Constant}: The main technical tool we employ to bound the rate of convergence of \eqref{eq:ct_diff} to its stationary measure is that of "logarithmic-Sobolev inequality" satisfaction. The logarithmic-Sobolev inequality is a certain functional inequality involving the infinitesimal generative operator corresponding to SDE \eqref{eq:ct_diff} which, if satisfied, allows for exponential convergence bounds. The \textit{rate} of exponential convergence is governed by a \textit{logarithmic-Sobolev constant} $\LSconst$. We introduce this here for clarity since our main bound will depend on this constant, but the full details of this constant will be described in Section~\ref{sec:plsineq}.}

\section{Main Result \textcolor{black}{I}. Non-Asymptotic Analysis of Passive Stochastic Gradient Langevin Dynamics}
\label{sec:main}
In this section we provide a non-asymptotic analysis of the PSGLD algorithm defined in~\eqref{eq:dt_sgld}. Recall that $\dtlaw$ \eqref{eq:dtlaw_def} is the sampling measure of Algorithm~\ref{alg:psgld} at iterate $k$, $\gibbs$ is the Gibbs measure proportional to $\exp(-\beta\rew)$, and $\wass(\dtlaw,\gibbs)$ is the 2-Wasserstein distance between these. Our main result is as follows: for any $\hlb>0$ we can choose the step size $\step$ small enough and iteration number $k$ large enough such that $\wass(\dtlaw,\gibbs) \leq \mathcal{O}(\hlb)$. In this section we formulate this result precisely. We provide a brief overview of the 2-Wasserstein metric, specify assumptions on the cost function $\rew$ and sampling distribution $\idistrw$, provide the main bound in the form of Theorem~\ref{thm:main1}
\label{sec:mainresults}, and discuss the application to adaptive inverse reinforcement learning in a Markov Decision Process framework. 
\subsection{2-Wasserstein Distance}
We provide a non-asymptotic bound on the convergence of \eqref{eq:dt_sgld} to the Gibbs measure $\gibbs$ \eqref{eq:gibbsd}, in terms of the \textit{2-Wasserstein distance}:
\begin{equation}
\label{def:wass}
\wass(\mu,\nu) := \inf_{\gamma \in \Gamma(\mu,\nu)} \left(\CE_{(x,y)\sim \gamma}\|x-y\|^2\right)^{1/2}
\end{equation}
Here $\Gamma(\mu,\nu)$ is the set of all couplings of measures $\mu$ and $\nu$, where a coupling $\gamma$ is a joint probability measure on $\reals^{\dimn} \times \reals^{\dimn}$ with marginals $\mu$ and $\nu$, i.e., 
\begin{align*}
\gamma(A,\reals^{\dimn}) = \mu(A) , \quad \gamma(\reals^{\dimn},B) = \nu(B) \quad \forall A,B \in \mathcal{B}(\reals^{\dimn})
\end{align*}
where $\mathcal{B}(\reals^{\dimn})$ is the Borel $\sigma$-algebra of $\reals^{\dimn}$. Notice that the Wasserstein distance \eqref{def:wass} indeed satisfies all axioms of a metric on the space of measures. The 2-Wasserstein distance is a more suitable metric for assessing the quality of approximate sampling schemes \cite{dalalyan2019user}, \cite{raginsky2017non}, than others such as total-variation norm, since it gives direct guarantees on the accuracy of approximating higher order moments \cite{dalalyan2019user}. However, it also precludes us from using SDE discretization analysis presented in the seminal book \cite{kloeden1992stochastic}, which utilizes total-variational norm.

\subsection{Assumptions}
\label{subsec:mainresults}
Here we list the key assumptions on the cost function $\rew$ \eqref{eq:rew} of the forward learner and the base sampling distribution $\idistr$, required for the finite-sample analysis. 
Assumptions on $\rew$ are standard and equivalent to those taken in \cite{raginsky2017non}. Assumptions on the base sampling distribution $\idistr$ hold for a wide class of probability density functions, including \textcolor{black}{sub-}Gaussian \textcolor{black}{and multi-modal Gaussian mixture} densities.

\begin{myassumptions}
\item[\nl{as:fwdlrn}{A1}]
The forward learning process produces sequential stochastic gradients $\gradn\rew(\thk)$ such that the following structural properties hold (as discussed in Section~\ref{sec:fwdlrn}):
\begin{enumerate}[label=\roman*)]
   {\setlength\itemindent{25pt}  \item $\exists \,\sgdtm > 0 : \, \CE\|\gradn\rew(\thk)\| \geq \sgdtm \,\forall k \in \N$}
   {\setlength\itemindent{25pt}  \item $\exists \,b_2 > 0 :\, \textnormal{Var}(\grad\rew(\thk) | \theta_{k-1},\dots,\theta_0) \geq b_2 \,\forall k\in\N$}
    {\setlength\itemindent{25pt} \item $\exists \,M_{\theta} < \infty: \,\, \sup_{k\geq 0} \CE\|\thk\|^2 \leq M_{\theta}$ }
     {\setlength\itemindent{25pt} \item $\theta_0 \overset{i.i.d.}{\sim} \idistrw$, given in \eqref{eq:idistrw}}
\end{enumerate}

\item[\nl{ass:Msmooth}{A2}]
 $\rew$ is $\lipJ$-Lipschitz continuous and $\lipGJ$-smooth: $\exists$ $\lipJ$, $\lipGJ >0$ such that for all $x,y \in \reals^{\dimn}$,
\begin{align*}
    &\|\rew(x) - \rew(y)\| \leq \lipJ \|x-y\|, \quad \quad  \|\grad\rew(x) - \grad\rew(y)\| \leq \lipGJ \|x - y\|
\end{align*}

\item[\nl{as:diss}{A3}]
\textcolor{black}{If $\rew$ has unbounded support}, then $\rew$ is $(\dissm,\dissb)$-dissipative:
\begin{align*}
        &\exists \ \dissm>0,\dissb\geq 0 :  \langle x, \grad\rew(x) \rangle \geq \dissm\|x\|^2 - \dissb, \ \forall x \in \reals^{\dimn} 
\end{align*}

\item[\nl{ass:gradnoise}{A4}]
The noisy SGD gradient evaluation is unbiased, i.e. $\CE[\gradn\rew(x)] = \grad\rew(x) \ \forall x \in \reals^{\dimn}$, symmetric, and admits a density $\mu_{sgd}$ with support on $\reals^{\dimn}$. Furthermore, the noise is additive such that $\gradn\rew(x) - \grad\rew(x)$ is i.i.d. with variance bounded uniformly in $x$, i.e., there exist constant $\gradnvar > 0$ such that \[\CE[\|\gradn \rew(x) - \grad\rew(x)\|^2] \leq \gradnvar, 
\ \forall x \in \reals^{\dimn}\]

\item[\nl{ass:idistexp}{A5}]
    The base sampling distribution $\idistr$ has support on the support of $\exp(-\beta\rew)$, and has exponential tail decay and differential decay $\mathcal{O}(\|x\|^{-1})$, i.e., 
    \[\exists \tailbound \in \N, \idiffconst > 0 : \idistr(x) \leq \exp(-\|x\|^2) \, , \quad \|\grad\idistr(x)\| \leq \frac{\idiffconst}{\|x\|}\ \  \forall \|x\|>\tailbound\]

\item[\nl{as:Bd_Der}{A6}]
The base sampling distribution $\idistr$ is Lipschitz-continuous:
\begin{align*}
        &\exists \Dbound > 0 : \|\idistr(x) - \idistr(y)\| \ \leq \Dbound\|x-y\| \ \forall x,y\in\reals^{\dimn}
\end{align*}

\item[\nl{ass:Kspec}{A7}]
The kernel function $K(\cdot)$ satisfies \eqref{eq:Kspec}.

\item[\nl{ass:beta}{A8}]
 Here $\wedge$ denotes the min operator and $\vee$ the max operator. Assume 
\begin{enumerate}[label=\roman*)]
    {\setlength\itemindent{25pt} \item  $\sgdstep \in (0, 1 \wedge \frac{\dissm}{4\lipGJ^2})$}
    {\setlength\itemindent{25pt} \item $\step \in \left(0, 1 \wedge \sqrt{\frac{1}{249}}\lipGJ^{-1} \right)$}
    {\setlength\itemindent{25pt} \item $\beta\geq \frac{1}{4\lipGJ^2} \vee \frac{\sqrt{2\pi+4}}{\dissm\sqrt{\lipGJ}}$}
\end{enumerate}

\end{myassumptions}
\textit{Discussion of assumptions}: \ref{as:fwdlrn} was introduced and discussed in Section~\ref{sec:fwdlrn}. \ref{ass:Msmooth} - \ref{ass:idistexp} are equivalent to those used for the objective function in \cite{raginsky2017non}. \ref{ass:Msmooth} is widely used in the literature on non-convex optimization and sampling. \ref{as:diss} {\color{black}is often satisfied in practice by the forward learner enforcing weight decay regularization \cite{krogh1991simple}, or otherwise by restricting to a bounded domain}, see Section 4 of \cite{raginsky2017non} for more details. \ref{ass:gradnoise} is a standard assumption for stochastic gradient evaluations. \ref{ass:idistexp} and \ref{as:Bd_Der} admit a wide range of probability density functions, including Gaussians.
\ref{ass:Kspec} is equivalent to assumptions on the kernel function in \cite{krishnamurthy2021langevin} and in the passive stochastic gradient algorithm literature \cite{nazin1989passive}. These assumptions admit a wide range of kernels, including Gaussians.
We note that for \ref{ass:beta} to be satisfied in practice, the inverse learner must have some knowledge of feasible ranges for Hessian bound $\lipGJ$ and dissipativity constant $\dissm$; once these ranges are known then $\step$ can be taken small enough and $\beta$ large enough so that $(ii)$ and $(iii)$ are satisfied. Notice that the feasible range for $\sgdstep$ can always be satisfied; the SGD process \eqref{eq:sgd} optimizing cost function $\rew$ with step $\hat{\eta} \geq (1 \wedge \frac{\dissm}{4\lipGJ^2})$ is equivalent to another SGD with step $\eta <\frac{\dissm}{4\lipGJ^2}$ which optimizes $\frac{\eta}{\hat{\eta}}\rew$. So assuming $\sgdstep$ which satisfies \ref{ass:beta} we can sample from $\gibbs \propto \exp(-\frac{\eta}{\hat{\eta}}\beta\rew)$, from which $\rew$ can be recovered since the scale $\frac{\eta}{\hat{\eta}}\beta$ disappears upon MCMC sample measure normalization.

\subsection{Main Result and Discussion. Finite-Sample Bound}
\label{sec:mainresult}

\begin{figure}
\centering

\begin{tikzpicture}[node distance=1.75cm]
 \tikzstyle{box} = [rectangle , minimum width=3cm, minimum height=1cm,text centered, draw=black]
 \tikzstyle{title} = [rectangle, text centered, fill=white!30]

 \tikzstyle{arrow} = [thick,->,>=stealth]


 \node (x1) [box] {\small$\thkk = \thk - \sgdstep\gradn \rew (\thk)$};
 \node (x1title) [title] at (x1.north) {\small SGD \eqref{eq:sgd}};
 
 \node (x2) [box, below of=x1] {\small$\alkk =\ \alk - \step\biggl[\delstep \kernel \frac{\beta}{2} \gradn \rew (\thk) \,\color{black}{-}\, \grad \idistrw(\alk)\biggr]\idistrw(\alk) + \sqrt{\step}\idistrw(\alk) w_k$};
  \node (x2title) [title] at (x2.north) {\small  Passive SGLD \eqref{eq:dt_sgld}};
   
 \node (x3) [box, below of=x2] {\small$d\alt = -\left[\frac{\beta}{2}\idistrw^2(\alt)\grad\rew(\alt) \,\color{black}{-}\, \grad\idistrw(\alt)\idistrw(\alt)\right]dt + \idistrw(\alt)d\brownt$}; 
 \node (x3title) [title] at (x3.north) {\small Continuous-time Diffusion \eqref{eq:ct_diff}};
  
 \node (x4) [box, below of=x3] {\small$\gibbs(\alpha) \propto \exp(-\beta \rew(\alpha))$}; 
 \node (x4title) [title] at (x4.north) {\small Gibbs Measure \eqref{eq:gibbsd}};
 
 \draw [arrow] (x1) -- (x2title);
 \draw [arrow] (x2) -- (x3title);
 \draw [arrow] (x3) -- (x4title);

 \draw[decoration={brace},thick, decorate]
  (x2.east) ++(-50pt,-15pt) -- node[right=6pt] {\small$\wass(\dtlaw,\ctlaw)$} ++(0,-20pt)(x2.east);

\draw[decoration={brace},thick, decorate]
 (x3.east) ++(-46pt,-15pt) -- node[right=6pt] {\small$\wass(\ctlaw,\gibbs)$} ++(0,-20pt)(x2.east);

\end{tikzpicture}

\caption{\small High level procedure for achieving inverse reinforcement learning. The forward learning process is represented by a stochastic gradient descent (SGD), and the inverse learner incorporates sequential SGD evaluations $\thk$ into its PSGLD algorithm to reconstruct $\rew$. The PSGLD algorithm reconstructs $\rew$ by approximately sampling from the Gibbs measure $\gibbs$ (then taking the log-sample density). We measure the proximity of the PSGLD algorithm to $\gibbs$ by $\wass(\dtlaw,\gibbs)$, the 2-Wasserstein distance between the sample law of $\alk$ and the measure $\gibbs$. We control this distance by bounding it by $\wass(\dtlaw,\ctlaw) + \wass(\ctlaw,\gibbs)$, where $\ctlaw$ is the law of $\alpha(t)$ at time $t=k\step$.}\label{fig:flowchart}
\end{figure}

Letting 
\[\dtlaw := \law(\alk), \quad \ctlaw := \law(\alpha(k\step))\] be the respective measures of the sampling density produced by iterates $\alk$ \eqref{eq:dt_sgld} and the continuous time diffusion $\alt$ \eqref{eq:ct_diff} at time $t=k\step$, we may bound 
\[ \wass(\dtlaw,\gibbs) \leq \wass(\dtlaw,\ctlaw) + \wass(\ctlaw,\gibbs)\] Figure~\ref{fig:flowchart} shows the high level procedure for achieving inverse reinforcement learning. The forward learner is represented by a stochastic gradient descent (SGD) process which optimizes $\rew$. The PSGLD algorithm takes in sequential SGD evaluations $\thk$ and produces samples $\alk$ which approximately sample from the Gibbs measure $\gibbs$, allowing for reconstruction of $\rew$ by taking the log-sample density. We measure this approximation by the distance $\wass(\dtlaw,\gibbs)$, which can be bounded by introducing the intermediate continuous-time diffusion \eqref{eq:ct_diff}, since PSGLD \eqref{eq:dt_sgld} is an approximate discretization of \eqref{eq:ct_diff} and \eqref{eq:ct_diff} has Gibbs measure \eqref{eq:gibbsd} as its stationary measure.


We present our Wasserstein bound in a way that explicitly depends on a hyperparameter $\delta$, e.g., $\wass(\dtlaw,\gibbs) \leq f(\delta)$ for some function $f$ which is monotonically increasing and has $\lim_{\delta\to0}f(\delta)=0$. Both the Wasserstein bound and certain algorithmic parameters have a functional dependence on $\delta$: for decreasing $\delta$ (decreasing $\wass(\dtlaw,\gibbs)$), we require e.g., increasing the algorithmic iterations and decreasing the step size. Specifically, our main result states that for any arbitrarily small $f(\delta)$, we can choose the step size $\step$ small enough, algorithmic iterations $k$ large enough, kernel scale parameter $\Delta$ small enough, and sampling distribution scale parameter $\idw$ small enough, such that $\wass(\dtlaw,\gibbs) \leq f(\delta)$. Next these qualitative parameter specifications are shown explicitly, as functions of control hyperparameter $\delta$. Then our main bound on $\wass(\dtlaw,\gibbs)$ is presented in Theorem~\ref{thm:main1}. 

\subsubsection{Algorithmic Parameter Specifications}
\label{sec:paramspec}
Here we show the dependence of algorithmic parameters on the hyperparameter $\delta$, which controls the main Wasserstein bound presented in Theorem~\ref{thm:main1}. $\delta$ acts as a one-dimensional "knob" that can be turned, which reveals the step size $\step$, iteration number $k$, etc., required to achieve a 2-Wasserstein bound proportional to $\delta$. The main idea is that Theorem~\ref{thm:main1} presents a (monotonically increasing) function $f(\delta)$, with $\lim_{\delta\to0}f(\delta)=0$, such that for any $\delta>0$ we can take algorithmic parameters \textit{as follows} to obtain $\wass(\dtlaw,\gibbs) \leq f(\delta)$.

\textit{Step Size}: 
\begin{equation}
\label{eq:stepspec}
    0 < \step \leq \left( \frac{\delta}{\log \left(\frac{1}{\delta} \right)}\right)^2 \wedge 1
\end{equation}

\textit{Algorithmic Iterations}:
\begin{equation}
\label{eq:itspec}
k\step \,\color{black}{=}\, \beta\, \LSconst \log\left(\frac{1}{\delta}\right)
\end{equation}
where $\LSconst$ is the logarithmic-Sobolev constant of diffusion \eqref{eq:ct_diff}, explicitly bounded in \eqref{eq:lsconst}.

\textit{Kernel Scale}: Recalling, for general $\alpha\in \reals_+$, $\K_{\alpha}(\cdot) = \frac{1}{\alpha^{\dimn}}\K(\frac{\cdot}{\alpha})$, define $\ksup_{\alpha} := \sup_{x\in\reals^{\dimn}}\K_{\alpha}(x)$. Also let $K^{-1}$ denote the inverse of $K$ and $K^{-2}$ denote the inverse of $K^2$, both mapping to the non-negative orthant, i.e., for $x\in \reals$, $\K^{-1}(x) := \{y\in \reals_{+}^{\dimn} : \K(y) = x\},\quad 
        K^{-2}(x) := \{y\in \reals_{+}^{\dimn} : \K^2(y) = x\}$
where $\reals_{+}^{\dimn}$ is the set of $\dimn$-dimensional vectors with all non-negative elements. This definition is without loss of generality, since $\K$ is chosen to be symmetric by \eqref{eq:Kspec}. 
 Then take
\begin{equation}
\label{eq:delspec}
\Delta \leq \inf_{x \in [\step,\ksup_{\step}]}\frac{K^{-1}(\frac{\ksup_1\sqrt{2\pi}}{2\step}e^{x^2/2})}{K^{-2}(x\step^{2\dimn})}
\end{equation}

\textit{Sampling Distribution Scale}: Choose the base sampling distribution $\idistr$ such that $\idistmax := \sup_x\idistr(x) = 1$, and sampling distribution scale parameter $\idw$ as 
\begin{equation}
    \label{eq:omegaspec}
    \idw \in [ \step^2,\, \step^{3/2}]
\end{equation}

\subsubsection{Main 2-Wasserstein Bound}
The following is the main result of the paper. An informal version was stated in Section~\ref{sec:mainresinf}.
\begin{theorem}[Finite-Sample 2-Wasserstein Bound]\label{thm:main1}
     Consider the PSGLD Algorithm~\ref{alg:psgld} with iterates $\alk \in \reals^{\dimn}$. Recall $\LSconst$ is the logarithmic-Sobolev constant for diffusion process~\eqref{eq:ct_diff} ($\LSconst$ will be bounded in \eqref{eq:lsconst}). For any 
     \begin{equation}
    \label{eq:delmax}
     \delta \in \left[0,\exp\left(-\frac{1}{\beta\LSconst}\right)\right]
    \end{equation}
    choose step size $\step$ according to \eqref{eq:stepspec}, number of iterations $k$ according to \eqref{eq:itspec}, kernel scale $\Delta$ according to \eqref{eq:delspec}, and sampling distribution $\idistrw$ with $\idw$ satisfying \eqref{eq:omegaspec}. Then, under assumptions (\ref{ass:Msmooth})-(\ref{ass:beta}), the 2-Wasserstein distance between the distribution $\dtlaw$, generated by the PSGLD algorithm, and the Gibbs measure $\gibbs$ \eqref{eq:gibbsd}, satisfies: 
    \begin{align}
    \begin{split}
    \label{eq:wassbound}
       &\wass(\dtlaw,\gibbs) \leq \delta\left[C_4 + \sqrt{2\LSconst C_3} \right] + \delta\sqrt{10\LSconst\dimn\log\left(1/\delta\right)}
    \end{split}
    \end{align}
    $C_3, C_4$ are constants dependent on structural specifications of $\rew$ and the process \eqref{eq:dt_sgld}, provided explicitly in Appendix~\ref{ap:bd_consts}. $\LSconst$ is the logarithmic-Sobolev constant bounded explicitly in Proposition~\ref{prop:logsob}.
\end{theorem}


\textit{Bound Discussion}: For any $\alpha>0$, $\delta\sqrt{\alpha\log\left(1/\delta\right)}$ is monotonically increasing in $\delta$ for $\delta \in (0,0.607)$ and
\[\lim_{\delta\to0}\delta\sqrt{\alpha\log\left(1/\delta\right)}=0\] So, \color{black}{our upper bound \eqref{eq:wassbound} on $\wass(\dtlaw,\gibbs)$} is monotonically increasing in $\delta$ for $\delta \in (0,0.607)$ and \[\lim_{\delta\to0}\wass(\dtlaw,\gibbs) = 0\]
Thus, Theorem~\ref{thm:main1} asserts that, through hyperparameter $\delta$, we can control the  number of iterations $k$ as \eqref{eq:itspec}, step size $\step$ as \eqref{eq:stepspec}, kernel scale $\Delta$ as \eqref{eq:delspec} and sampling distribution scale $\idw$ as \eqref{eq:omegaspec}, such that the PSGLD algorithm \eqref{eq:dt_sgld} is within any arbitrarily small desired 2-Wasserstein distance \eqref{eq:wassbound} to the Gibbs distribution \eqref{eq:gibbsd}. Here $\delta$ acts as a precision parameter; smaller $\delta$ yields a tighter approximation \eqref{eq:wassbound} at the expense of larger number of iterations~$k$ and smaller step size $\step$, kernel scale $\Delta$ and sampling distribution scale $\idw$. 

Recalling  $\gibbs(\alpha) \propto \exp(-\beta\rew(\alpha))$, the cost function $\rew$ can be approximately reconstructed as the logarithm of sample density produced by $\alk$. This reconstruction approaches the true cost function $\rew$ as $\delta \to 0$. This result generalizes the nonasymptotic bound obtained in \cite{raginsky2017non} (equation 3.3) to our \textit{passive} stochastic gradient Langevin dynamics algorithm. 

\vspace{0.1cm}
 
\subsubsection{Parameter Specifications Discussion}
{\color{black} Theorem~\ref{thm:main1} provides substantial insight into how the internal parameters of Algorithm~\ref{alg:psgld} affect its ability to approximate the Gibbs measure. We now discuss the dependencies on these parameter specifications.}
\begin{itemize}
    \item[-] \textit{Specification Intuition}: Observe the parameter specifications \eqref{eq:stepspec} - \eqref{eq:omegaspec} necessary for acheiving a given Wasserstein bound \eqref{eq:wassbound}. Specifications \eqref{eq:stepspec} and \eqref{eq:itspec} are intuitive; as we decrease the step size $\step$ we should decrease the discretization error between algorithm \eqref{eq:dt_sgld} and continuous diffusion \eqref{eq:ct_diff}, and as we increase the iterations $k$ we will decrease the distance from the diffusion \eqref{eq:ct_diff} to its stationary measure $\gibbs$. 

    \item[-] \textit{Role of the Kernel Scale}: 
    \begin{itemize}
        \item[-] \textit{Necessity of Kernel Scaling}: The algorithm \eqref{eq:dt_sgld} is not an exact discretization of diffusion \eqref{eq:ct_diff}, as it has a gradient term governed by the external SGD process \eqref{eq:sgd}. The weighting kernel $\K$ is introduced to control for biases in this SGD-evaluated gradient, but for any non-zero variance of $\K$ there will still be biased gradient evaluations entering the algorithm \eqref{eq:dt_sgld} which prevent it from converging to $\gibbs$. To minimize these, the kernel scale $\Delta$ can be reduced; however note that this should come at the cost of increasing the time needed to reach a specified Wasserstein bound, since "useful" {\color{black}(heavily weighted)} gradient information will be integrated into the passive algorithm less often. We see this as an unavoidable tradeoff, one which necessitates taking $\Delta$ {\color{black}inversely proportional to desired 2-Wasserstein proximity}, but which has not been fully quantified in this work.

        \item[-] \textit{Dependence on Kernel Structure}: \color{black}{Observe that the bound in Theorem~\ref{thm:main1} does not depend on the structure of the kernel $K$ except through its scale parameter $\Delta$. However, this is because the \textit{choice of $\Delta$} depends on the structure of $K$ \eqref{eq:delspec}. Thus, we isolate the bound from the specific choice of kernel: once the kernel structure has been incorporated to determine a suitable choice of $\Delta$ \eqref{eq:delspec}, the bound \eqref{eq:wassbound} is agnostic to this kernel structure. One note for future work is that that specification \eqref{eq:Kspec} exhibits poor scaling with respect to dimension; at least for most classical kernels such as the Gaussian it decreases exponentially with the dimension. It would be interesting to investigate further the sufficient conditions on $\Delta$ allowing for a tractable bound, perhaps making the bound \eqref{eq:Kspec} tight to reveal the best dependence on dimensionality, etc.}
    \end{itemize} 

    \item[-] \textit{Role of the Sampling Distribution}: Note that we also require decreasing the sampling distribution scale $\idw$ to obtain a tighter Wasserstein bound. This has arisen as a quantitative necessity in obtaining bounds in Lemma~\ref{lem:MSEbd} and Lemma~\ref{lem:KL_bd} (which are key developments in the proof of Theorem~\ref{thm:main1}). The intuition is as follows: this specification allows us to control $\CE\|\idistrw(\alk)\|^2$ and $\CE\|\grad\idistrw(\alk)\|^2$, such that the influence of $\idistrw(\alk)$ in the algorithm ~\eqref{eq:dt_sgld} does not outweigh that of $\Kd(\thk,\alk)\gradn\rew(\thk)$; (as in the previous paragraph explanation) as $\Delta$ gets smaller, $\Kd(\thk,\alk)\gradn\rew(\thk)$ \color{black}{surpasses any given positive threshold} less often, and the contribution of $\idistrw(\alk)$ should balance this. Notice that we also must choose the base distribution $\idistr$ wide enough (such that $\idistmax =  1$), and scale parameter not \textit{too} small (lower bounded by $\step^2$), so that there is always some non-zero probability of sampling from any point in the domain, allowing for sufficient exploration. {\color{black}Note also that the inverse learner requires knowledge of $\pi_{0,\gamma}$. In \cite{krishnamurthy2021langevin}, a multi-kernel PSGLD algorithm is presented which does not require such knowledge. Analysis of this multi-kernel algorithm is an interesting point for future research but is out of the scope of this paper.}

\end{itemize}

\vspace{0.1cm}
{\color{black}\subsubsection{Dependence on Forward Process Discussion}
\label{subsec:fwdprocdep}

\begin{itemize}
    \item[-] \textit{Dependence on Algorithm~\ref{alg:sgd} Structure}: Observe that the constant $C_4$ appearing in the bound \eqref{eq:wassbound}, and displayed in Appendix~\ref{ap:bd_consts}, contains within it the term \[C_5 := \left[\frac{1/\sgdtm}{\hat{\mu}_{sgd}} + \frac{1}{\sgdtm^2} \right]\] Here $\sgdtm$ is the SGD stopping threshold in Algorithm~\ref{alg:sgd}, and $\hat{\mu}_{sgd}$ is a strictly positive constant defined in Lemma~\ref{lem:sgdvarbd}. Let us unpack this. The term $\hat{\mu}_{sgd}$ is a lower bound on $\textrm{Var}(\grad\rew(\thk)) | \theta_{k-1})$, as can be seen by inserting the SGD update \eqref{eq:sgd} and expanding the Variance definition. This is utilized in the proof of Lemma~\ref{lem:KL_bd}, since this variance term appears in the denominator. Thus, the bound \eqref{eq:wassbound} depends inversely on the stopping threshold $\sgdtm$ and conditional variance $\textrm{Var}(\grad\rew(\thk) | \theta_{k-1})$ magnitudes. 

    \item[-] \textit{Dependence Intuition}: The constant $C_5$ comprising $\sgdtm$ and $\hat{\mu}_{sgd}$ nicely captures dependencies on the stopping threshold and exploratory movement of the SGD \eqref{eq:sgd}. Firstly, as $\sgdtm$ goes to zero, the bound \eqref{eq:wassbound} will increase arbitrarily. From a design perspective, \eqref{eq:sgd} must produce stochastic gradients which are \textit{not too small}, so that the contribution to the PSGLD evolution \eqref{eq:dt_psgld} does not become negligable. Also, the bound \eqref{eq:wassbound} cannot be decreased arbitrarily by increasing this threshold arbitrarily, since the SGD \eqref{eq:sgd} will only update (and contribute a stochastic gradient to the PSGLD process) when this threshold is surpassed. Secondly, the conditional variance $\textrm{Var}(\grad\rew(\thk) | \theta_{k-1})$ captures the variability of iterate $\thk$ and thus reflects the exploratory movement in $\reals^{\dimn}$. Thus, increased 'exploratory behavior' of SGD \eqref{eq:sgd} decreases the approximation bound \eqref{eq:wassbound} by providing more information about the cost function. 

    \item[-] \textit{Exploration-Exploitation Tradeoff}: Observe that the constant $\etsup$ appearing in the bound \eqref{eq:wassbound} within $C_0$, and displayed in Appendix~\ref{ap:bd_consts}, is defined as a bound on $\sup_{k\geq 0}\CE\|\thk\|^2$. Thus, the bound \eqref{eq:wassbound} scales with this supremum; the forward process cannot 'explore' \textit{too much}, since this would decrease the expected value of the kernel weighting evaluation $K(\thk,\alk)$ ($\thk$ being far from $\alk$) such that the PSGLD incorporations of stochastic gradients $\gradn\rew(\thk)$ become negligable. Contrasting this with the motivation for SGD 'exploration' (corresponding to $\textrm{Var}(\grad\rew(\thk)|\theta_{k-1})$) reveals the classical exploration-exploitation tradeoff. The forward process must explore the domain sufficiently well to provide ample information about the cost function, while not straying too far from the PSGLD iterates.

    \item[-] \textit{Generalizing Beyond SGD}: This discussion also motivates the consideration of which other forward processes, besides the form of stochastic gradient descent in Algorithm~\ref{alg:sgd}, will suffice. The answer is precisely those processes which admit lower bounds on $\CE\|\gradn\rew(\thk)\|$ and $\textrm{Var}(\grad\rew(\thk) | \theta_{k-1})$ and an upper bound on $\sup_{k\geq 0} \CE\|\thk\|^2$. This is exactly why Assumption~\ref{as:fwdlrn} is necessary. This may include more sophisticated and practical optimization schemes such as the Adam algorithm \cite{kingma2014adam}, classical SGLD, or even non-re-initializing ergodic stochastic processes, although we do not investigate these here.

    \item[-] \textit{Handling of Re-Setting}: It is also important to highlight that the re-setting of Algorithm~\ref{alg:sgd} is handled automatically by the weighting kernel function $K$; that is, the PSGLD algorithm does not need to incorporate any information about when the SGD will reset. This is a nice property from a design perspective, and further motivates the application to more general forward processes which satisfy those properties listed in the previous paragraphs.

   \item[-] \textit{Handling of Termination}: Observe that in practical instances, Algorithm~\ref{alg:sgd} will terminate at some finite iteration $\ks$. In this case, one may not achieve an \textit{arbitrarily} precise sampling measure through \eqref{eq:wassbound}. However, we still may determine \textit{how close} the sampling measure of Algorithm~\ref{alg:psgld} can be made to the Gibbs measure after this finite number of iterations. Specifically, choosing $\epsilon$ and $\delta$ such that \eqref{eq:stepspec} and \eqref{eq:itspec} hold, with $k=\ks$, gives this proximity through \eqref{eq:wassbound}.
\end{itemize}}

\vspace{0.1cm}
\textit{Remark. Comparison to bound in \cite{raginsky2017non}}: \cite{raginsky2017non} derives non-asymptotic bounds in 2-Wasserstein distance between the Gibbs measure and the classical (non-passive) stochastic gradient Langevin dynamics (SGLD) algorithm. The bound \eqref{eq:wassbound} compares to that derived in \cite{raginsky2017non} scales as, for iteration number $k$, step size $\step$, and stochastic gradient noise variance $\gradnvar$, as $\mathcal{O}(k\left(\step^{1/4} + \gradnvar^{1/4}\right) + \exp(-k\step))$. In contrast, our bound scales as $\mathcal{O}(k\sqrt{\step} + \exp(-k\step))$, see \eqref{eq:wasstotal}. Note that we provide a better dependence on the step size $\step$, and remove dependence on the noise variance $\gradnvar$. {\color{black}Specifically, while the bound \eqref{eq:wasstotal} depends on $\gradnvar$, it does so through a constant multiplicative term scaling with $\step$. This noise dependence in \cite{raginsky2017non} cannot be decreased by scaling $\step$, and thus cannot be absorbed into the big-O notation. However, our bound  relies on more strict control of the passive algorithm parameters, including the kernel function \eqref{eq:Kspec} and sampling distribution \eqref{eq:omegaspec} specifications. In a sense there are more free parameters that must be tuned appropriately in order for our bound to be achieved. Intuitively, this is to be expected since the PSGLD algorithm \eqref{eq:dt_psgld} incorporates noisy and \textit{mis-specified} gradients; it must incorporate more complexity in its design. It turns out that appropriately handling this complexity, through specifying the kernel scale $\Delta$ and sampling distribution $\idistrw$, gives enough quantitative control to not only produce a useful 2-Wasserstein bound but to ignore the dependence on noise $\gradnvar$, at least within the scaling law. }

The bound \eqref{eq:wassbound} also mirrors that in \cite{raginsky2017non} in having exponential dependence on the dimension $\dimn$, from the logarithmic-Sobolev constant bound \eqref{eq:lsconst}. \cite{krishnamurthy2021langevin} proposes a multi-kernel PSGLD algorithm which has better performance in high dimensions; an interesting line of future research is to extend the analysis of this paper to the multi-kernel algorithm. 
\vspace{0.1cm}

\textcolor{black}{\section{Main Result II. Reconstructing the Cost Function via MCMC} \label{sec:cfr}}

Recall the aim of adaptive inverse reinforcement learning is to recover a cost function $\rew$ by first using PSGLD algorithm \eqref{eq:dt_psgld} to sample from the Gibbs measure $\gibbs$ \eqref{eq:as_samp} encoding the cost function $\rew$, and then reconstruct $\rew$ by taking the logarithm of the sample density \eqref{eq:Jrec}. Thus far, we have provided bounds on the proximity of the finite-sample PSGLD empirical measure $\dtlaw$ to the Gibbs measure $\gibbs$ in 2-Wasserstein distance. \textit{Standing alone, these bounds contribute to the field of passive stochastic gradient algorithms \cite{YY96}, \cite{nazin1989passive}, \cite{revesz1977apply}, \cite{hardle1987nonparametric}, by providing sample-complexity analysis for the generalized passive Langevin algorithm \eqref{eq:dt_psgld}}. However, for our adaptive IRL purposes it still remains to be shown how the cost function $\rew$, or an approximation $\hat{\rew}$ thereof, can be recovered. \textcolor{black}{In this section we develop a MCMC algorithm for reconstructing an approximation $\hat{\rew}$ of $\rew$, and provide a concentration bound for its distance to $\rew$ in terms of the $L^1$ norm. We also provide several examples and show that they satisfy all of our assumptions.}
\vspace{0.1cm}



\textit{Kernel Density Estimation}: Perhaps the most ubiquitous method for density estimation  from samples \cite{silverman2018density} is that of kernel density estimation, where the density estimate is formed as the weighted sum of smoothing kernels evaluated at sample points. \cite{vogel2013uniform} provides uniform concentration inequalities for non-parametric multidimensional kernel density estimators. Motivated by this, we construct an algorithm for approximately recovering the cost function $\rew$ from samples using a nonparametric multidimensional kernel density estimator. We provide a concentration inequality quantifying the proximity in $L^1$ of the reconstructed cost function $\hat{\rew}$ to the true cost function $\rew$. A key condition for this technique is that samples must be restricted to a compact set; we next outline this sampling restriction.

\subsection{Sampling Procedure and Cost Reconstruction Accuracy} 
\textcolor{black}{Now we discuss the procedure for MCMC sampling and cost function reconstruction, accomplished in Algorithm~\ref{alg:costrec}}. Recall Theorem~\ref{thm:main1} gives us a way of choosing, for any $\wassprox > 0$, a step size $\step$ and iteration $k$ such that Algorithm~\ref{alg:psgld} samples from a measure $\pi_k$ with $\wass(\dtlaw,\gibbs) \leq \wassprox$. \textcolor{black}{Algorithm~\ref{alg:costrec} exploits this to sample from Algorithm~\ref{alg:psgld} at carefully chosen iterations in order to reconstruct the cost function with arbitrary accuracy. Next we outline the operation of Algorithm~\ref{alg:costrec} in more detail, then provide an explicit concentration inequality quantifying its reconstruction accuracy. First, we provide details for a technical condition imposed in the sampling operation of Algorithm~\ref{alg:costrec}; that is, we restrict to sampling over a compact set.}
\subsubsection{Restricting Reconstruction to a Compact Set} We reconstruct an approximation $\hat{\rew}$ of the cost function $\rew$ on a compact set $\Theta \subset \reals^{\dimn}$. This restriction is motivated in practice by bounded computational resources such as memory. It is motivated in theory since we must impose global Lipschitz-continuity of a logarithmic transformation of the sample measure, which cannot be achieved on $\reals^{\dimn}$ as the measure tails tend to zero. This restriction in no way affects the results of previous sections. In particular, Algorithms \ref{alg:sgd} and \ref{alg:psgld} still operate on an unbounded domain; we only restrict to a bounded set when reconstructing the cost function via MCMC in Algorithm~\ref{alg:costrec}.



\textcolor{black}{In Algorithm~\ref{alg:costrec} we only take MCMC samples from Algorithm~\ref{alg:psgld} which lay in the compact set $\Theta$. Thus, we consider the restricted sample measure of Algorithm~\ref{alg:psgld} on $\Theta$. Specifically, we denote $\hat{\pi}_k$ the sample measure $\dtlaw$ when restricted to $\Theta$:
\[\bar{\pi}_k(x) = \begin{cases} \dtlaw(x)/Z_{\dtlaw} \, &x\in \Theta \\ 0, &x\notin \Theta \end{cases}, \quad Z_{\dtlaw} := \int_{\Theta}\dtlaw(x)dx\]
where $Z_{\dtlaw}$ is a normalizing constant. Similarly, we let $\hat{\pi}_{\infty}$ be the Gibbs measure $\gibbs$ when restricted to $\Theta$,
\begin{equation}
\label{eq:gibbsrest}
\bar{\pi}_{\infty}(x) = \begin{cases} \gibbs(x)/Z_{\gibbs} \, &x\in \Theta \\ 0, &x\notin \Theta \end{cases},\quad Z_{\gibbs} := \int_{\Theta}\gibbs(x)dx,
\end{equation}
Now, we must quantify the proximity $\wass(\bar{\pi}_k,\bar{\pi}_{\infty})$ between these restricted measures, i.e., how much does $\wass(\bar{\pi}_k,\bar{\pi}_{\infty})$ differ from $\wass(\dtlaw,\gibbs)$, which we can bound by Theorem~\ref{thm:main1}? Lemma~\ref{lem:wsubset} gives us:
\begin{equation}
\label{eq:Zint}
\wass(\bar{\pi}_k,\bar{\pi}_{\infty}) \leq \sqrt{2 \zconst^{-1}}\,\wass(\dtlaw,\gibbs), \quad  \zconst:= \int_{\Theta \times \Theta}d\hat{\gamma}(x,y)
\end{equation} where $\hat{\gamma}$ is the optimal transport measure given by (recall the definition \eqref{def:wass}): \[\hat{\gamma} \in \arg\inf_{\gamma\in\Gamma(\dtlaw,\gibbs)}\left(\int_{\reals^{\dimn} \times \reals^{\dimn}} \|x-y\|^2d\gamma(x,y)\right)^{1/2}\]
Lemma~\ref{lem:otsubset} gives us a way of lower bounding $\zconst$ and thus controlling $\wass(\bar{\pi}_k,\bar{\pi}_{\infty})$ through $\wass(\dtlaw,\gibbs)$; we discuss this at the end of the subsection.}

{\color{black} \subsubsection{Algorithm~\ref{alg:costrec} Sampling Operation} 
\label{sec:alg3op}
Now we discuss the operation of Algorithm~\ref{alg:costrec}. Recall Algorithm~\ref{alg:psgld} (PSGLD) is a MCMC algorithm which takes sequential iterations from Algorithm~\ref{alg:sgd} as input, and produces samples according to the dynamics \eqref{eq:dt_psgld}. Recall $\pi_k$ is the sample measure of Algorithm~\ref{alg:psgld} at iterate $k$, and $\gibbs$ \eqref{eq:gibbsd} encodes the cost function $\rew$. By Theorem~\ref{thm:main1}, we can choose Algorithm~\ref{alg:psgld} parameters $k, \step, \Delta, \idw$ such that $\wass(\pi_k,\gibbs)$ is as small as desired. Now, Algorithm~\ref{alg:costrec} acts as a pre- and post-processing procedure for initialization, whereby $T$ Algorithm~\ref{alg:psgld} streams are initialized with appropriately chosen parameters, and sampling and reconstruction, whereby MCMC samples are taken and the cost function estimate is reconstructed. We discuss these two phases:
\begin{enumerate}[label=\roman*)]
    \item \textit{Initialization}: Observe the initialization phase of Algorithm~\ref{alg:costrec} (Lines 1-8). First, Algorithm~\ref{alg:psgld} parameters $\Delta,\gamma,\step$ and iterate $\hk$ are chosen such that we are guaranteed a desired 2-Wasserstein proximity $\wass(\pi_{\hk},\gibbs) \leq \rho$. Then, Algorithm~\ref{alg:psgld} streams are initialized independently with these parameters. Algorithm~\ref{alg:costrec} will initialize and run $\snum$ Algorithm~\ref{alg:psgld} streams sequentially, each with these pre-initialized parameters. 

\item \textit{Sampling and Reconstruction}: Algorithm~\ref{alg:costrec} runs $\snum$ streams of Algorithm~\ref{alg:psgld} sequentially, each with the same pre-chosen parameters. Each stream is sampled after $\hk$ iterations, if it lies in the sampling set $\Theta$, and then the next stream is initialized after the SGD process Algorithm~\ref{alg:sgd} re-initializes. This continues until $\snum$ streams have been sampled from. We let $S$ denote the set of these samples $\alpha_{\hk}^i, i\in[T]$ within $\Theta$. Then we have the following:
\begin{lemma}
\label{lem:iid}
Obtain samples $\{\alpha_{\hk}^i,\,i\in[S]\}$ from Algorithm~\ref{alg:costrec}. Then $\alpha_{\hk}^i \overset{i.i.d.}{\sim} \pi_{\hk}$.
\end{lemma}
\textit{This Lemma follows by observing that each stream of Algorithm~\ref{alg:psgld} is initialized and run independently, and uses independent data streams from SGD process Algorithm~\ref{alg:sgd} as input (since we wait for Algorithm~\ref{alg:sgd} to re-initialize in lines 17-18)}. Observe that by \eqref{eq:Zint}, $\bar{\pi}_{\hk}$ has $\sqrt{2\zconst^{-1}}\wassprox$ 2-Wasserstein proximity to $\bar{\pi}_{\infty}$ \eqref{eq:gibbsrest}, the restriction of Gibbs measure $\gibbs$ to $\Theta$. In order form an empirical estimate $\hat{\pi}$ of this density $\bar{\pi}_{\hk}$, we introduce a kernel $\mk: \reals^{\dimn} \to \reals$ which interpolates these $i.i.d.$ samples $\{\alpha_{\hk}^i \overset{i.i.d.}{\sim} \bar{\pi}_{\hk}\}_{i=1}^{|S|}$ as 
\[\hat{\pi}(\cdot) = \frac{1}{|S|b_{S}}\sum_{i=1}^{|S|}\mk\left(\frac{\cdot - \alpha_{\hk}^i}{b_{S}^{1/{\dimn}}}\right)\]
where $b_{S}$ is the kernel bandwidth which depends on $T$ and $|S|$, see Algorithm~\ref{alg:costrec}. We will provide precise concentration bounds that depend on $b_S$. \textit{The kernel $\mathcal{K}(\cdot)$ is distinct from the kernel $K(\cdot)$ used in the PSGLD algorithm}. We then reconstruct the cost function estimate $\hat{\rew}(\cdot)$ as \[\hat{\rew}(\cdot) = -\frac{1}{\beta}\log(\hat{\pi}(\cdot))\]
\end{enumerate}

Figure~\ref{fig:alg3chart} illustrates the operation of Algorithm~\ref{alg:costrec}: first initializing appropriate parameters of Algorithm~\ref{alg:psgld}, then sampling each sequential stream after an appropriately chosen number of iterations $\hk$, and finally reconstructing the cost function estimate $\hat{\rew}$. In Section~\ref{sec:costrec}, specifically Theorem~\ref{thm:kernest}, we provide a concentration bound on the reconstruction accuracy of $\hat{\rew}$. This bound depends on the parameters chosen in the initialization phase, and Section~\ref{subsec:costrec} provides details on how to choose these parameters to make the reconstruction accuracy as small as desired.
}

\begin{algorithm}[t]
{\color{black}
\caption{\color{black}{Cost Reconstruction by Sequential-Sampling Kernel Density Estimation}}\label{alg:costrec}
\begin{algorithmic}[1]
\State Initialize 2-Wasserstein proximity $\wassprox > 0$, number of sampling streams $\snum$, sampling region $\Theta \subset \reals^{\dimn}$
\State \textcolor{black}{Set kernel function $\mk$ and $T$-dependent bandwidth $b_T$.}
\State Take $\delta \in \textcolor{black}{\left[0, \exp\left(-\frac{1}{\beta \LSconst}\right) \right]}$ s.t. $\delta\left[C_4 + \sqrt{2\LSconst C_3} \right] + \delta\sqrt{10\LSconst\dimn\log\left(1/\delta\right)} \leq \wassprox$, \,\,$\step \leq \left(\frac{\delta}{\log(1/\delta)} \right)^2$, $\hk$ as \eqref{eq:itspec} 
\State Initialize parameters $\Delta, \gamma$ as \eqref{eq:delspec}, \eqref{eq:omegaspec}, respectively.
\State Set $k=0, \,S=\emptyset, \,\algtdum=0$.
\For{$i=1: \textcolor{black}{T}$}
\State \textcolor{black}{Initialize $i$'th PSGLD (Algorithm~\ref{alg:psgld}) stream by $\alpha_0^i \overset{\textcolor{black}{i.i.d.}}{\sim} \idistrw$}.


\While{$k \leq \textcolor{black}{\algtdum +} \hk$}   
    \State $k = k+1$ \textcolor{black}{\algorithmiccomment{Iterate $i$'th Algorithm 2 stream, taking $\gradn\rew(\thk)$ from Algorithm~\ref{alg:sgd} as input.}}
    \If{$k = \textcolor{black}{\algtdum +} \hk$}
        \State \textcolor{black}{Take sample $\alpha_k^i$} \textcolor{black}{\algorithmiccomment{Sample from $\pi_{\hk}$. By initializations (Lines 1-3) we have $\wass(\pi_{\hk},\gibbs) \leq \wassprox$.}}
        \If{$\alpha_k^i \in \Theta$}
            \State \textcolor{black}{$S \leftarrow S \cup \alk^i$}
        \EndIf
    \EndIf
\EndWhile
\textcolor{black}{\While{$\|\gradn\rew(\thk)\| > \sgdtm$} \algorithmiccomment{Wait until SGD (Algorithm~\ref{alg:sgd}) re-initializes.}
\EndWhile}
\State \textcolor{black}{$\algtdum = k$} \textcolor{black}{\algorithmiccomment{$\algtdum$ is a dummy variable holding the re-initialization iterate value}}
\EndFor
\textcolor{black}{\If{$|S| \geq 1$}
    \State \textcolor{black}{Set $b_S = b_T\sqrt{\frac{T}{|S|}}$, and construct density estimate $\hat{\pi}(\cdot) = \frac{1}{b_S |S|}\sum_{\alk^i \in S}\mk\left(\frac{\cdot - \alpha_k^i}{b_S^{1/\dimn}} \right)$.}
\State \textcolor{black}{Construct cost function $\hat{\rew}(\cdot) = -\frac{1}{\beta}\log(\hat{\pi}(\cdot))$, for $\beta$ in Algorithm~\ref{alg:psgld}.}
\EndIf}
\end{algorithmic}
}
\end{algorithm}

\begin{figure}[h!]
    \centering
    \begin{tikzpicture}
  \draw (0,0) rectangle (3,1);
  \node at (1.5,0.75) {Forward Learner};
  \node at (1.5,0.3) {SGD};
  \draw (5,0) rectangle (8,1);
  \node at (6.5,0.7) {Algorithm~\ref{alg:psgld}}; 
  \node at (6.5,0.25) {stream $i$}; %



  \draw (8.75,0) rectangle (12.25,1);
  \node at (10.5,0.75) {MCMC Samples};
  \node at (10.5,0.3) {$S = \{\alpha_{\hk}^i\}_{i\in[1,T]}$};

  \node at (3.7,0.75) {$\gradn\rew(\thk)$};


  \node[red] at (11,1.35) {Algorithm~\ref{alg:costrec}};
  
  \draw[dotted, red, thick] (4.5,-2.75) rectangle (16.5,1.6);
  \draw[-] (3,0.5) -- (4.35,0.5); 


  \draw[->] (4.35,0.5) |- (5,0.5); 

  \draw[->] (8,0.5) -- (8.75,0.5); 
  \draw[->] (12.25,0.5) -- (12.75,0.5);

  \draw[-] (8.35,0.5) |- (4.65,-0.25);
  \draw[->] (4.65,-0.25) -- (4.65, 0.5);
  \node at (6,-0.5) {\small If $k=\algtdum + \hk$:};
  \node at (6.4,-0.9) {\small sample $\alpha_{\hk}^i$};
  \node at (7.35, -1.3) {\small While $\|\gradn \rew(\thk)\| > \sgdtm$:};
  \node at (6.5, -1.7) {\small end; $\algtdum = k$};
  \node at (7.3,-2.1) {\small initialize stream $i+1$};
  \node at (5.4, -2.5) {\small end};



    
 
 
 
  \draw (12.75,0) rectangle (16.25,1);

  \node at (14.5,0.75) {Cost Reconstruction};
  \node at (14.5,0.25) {$\hat{\rew}(\cdot) = -\frac{1}{\beta}\log\left(\hat{\pi}(\cdot)\right)$};


\end{tikzpicture}
    \caption{\small \textcolor{black}{Schematic illustrating the operation of Algorithm~\ref{alg:costrec}. Recall Algorithm~\ref{alg:psgld} provides a MCMC technique for generating data points from sample measure $\pi_k$, and Theorem~\ref{thm:main1} provides guarantees on the 2-Wasserstein distance $\wass(\dtlaw,\gibbs)$ between $\dtlaw$ and the stationary Gibbs measure $\gibbs$. Algorithm~\ref{alg:costrec} acts as a pre- and post-processing procedure for reconstructing the cost function $\rew$ using samples from Algorithm~\ref{alg:psgld}. It pre-processes by initializing streams of Algorithm~\ref{alg:psgld} with appropriately chosen parameters. It post-processes by acquiring MCMC samples $\{\alpha_{\hk}^i\}_{i\in[1,\snum]}$ at a specified iterate $\hk$, reconstructing the sample measure through kernel density estimation, and recovering the cost function by logarithmically transforming this estimated measure.}
    }\label{fig:alg3chart}
\end{figure}

{\color{black} \subsection{Algorithm 3 Cost Reconstruction Accuracy}
\label{sec:costrec}
Here we provide our second main result, Theorem~\ref{thm:kernest}, which gives a concentration bound for the reconstruction accuracy of the cost estimate $\hat{\rew}$ computed  in Algorithm~\ref{alg:costrec}. We also discuss implementation details regarding the attained reconstruction accuracy and iteration complexity. These considerations reveal how to choose the algorithmic parameters in Algorithm~\ref{alg:costrec} such that the reconstruction error is as small as desired.

{\subsubsection{Cost Reconstruction Concentration Inequality}
\label{subsec:costrec}
We first introduce the following assumptions on the kernel function}}
\begin{assumption}
\label{as:kernspec}
Assume that  kernel $\mk(\cdot)$  satisfies
\begin{align}
\begin{split}
\label{eq:kernas}
    &\int_{\reals^{\dimn}}|\mk(u)|du <\infty, \,\,  \int_{\reals^{\dimn}}\mk(u) du= 1, \,\, \int_{\reals^{\dimn}}u_i \mk(u)du = 0 \, \forall i\in[\dimn], \\&\sup_{u\in\reals^{\dimn}}|\mk(u)| = \mk_1 < \infty, \,\, \int_{\reals^{\dimn}}|\mk(u)|du = \mk_2 < \infty, \,\,  \int_{\reals^{\dimn}}|u|^2\mk(u)du = \mk_4 <\infty
\end{split}
\end{align}
where $u_i$ is the $i'th$ element of $u$.
\end{assumption}
Observe that any symmetric density function satisfies the above conditions. We can then construct  an estimate $\hat{\rew}$ of the cost function as $\hat{\rew}(\cdot) = -\frac{1}{\beta}\log(\hat{\pi}(\cdot))$. We utilize Theorem 2 of \cite{vogel2013uniform} to produce a concentration inequality on the $L^1$ distance between the reconstructed cost function $\hat{\rew}$ to the true cost function $\rew$.

First we introduce following specification. Observe that since $\rew$ is non-negative (s.t. $\exp(-\beta\rew) \leq 1$), and is $\lipJ$-Lipschitz continuous and $\lipGJ$-smoooth (by Assumption~\ref{ass:Msmooth}), we have 

\begin{equation}
\label{eq:Vf}
\gibbs \in C^2(\Theta), \, \, \sup_{\alpha\in\Theta}\bigg|\frac{\del^2\gibbs(\alpha)}{\del\alpha_i\del\alpha_j}\bigg| \leq \mk_3 := \beta \lipGJ + \beta^2\lipJ^2 < \infty
\end{equation}

Now we provide our second main result, Theorem~\ref{thm:kernest}, on the cost function reconstruction approximation.

\begin{theorem}
\label{thm:kernest}
    Consider the reconstructed cost function $\hat{\rew}(\cdot)$ from Algorithm~\ref{alg:costrec}. Then, given Assumptions \ref{ass:Msmooth} to \ref{as:kernspec}, the following $L^1$ reconstruction error concentration inequality holds, valid for $\kfv > \wassprox\sqrt{ 2\csix}(\sqrt{\zconst}|\Theta|)^{-1}$:
    \begin{align}
     \begin{split}
     \label{eq:phidef}
        &\PR\left(\int_{\Theta}|\hat{J}(\theta) - J(\theta)|d\theta \geq  \phi_{\kfv,{\color{black}\snum }}\right) \leq 1 - (1 - 2\exp(-\psi_{\kfv,{\color{black}{\color{black}\snum }}}^2))(1-2\exp(-2y^2/T^3))\\
        &\text{ where} \;\;\phi_{\kfv,{\color{black}\snum }} = 2 L|\Theta|\left(\frac{\kfv}{\sqrt{{\color{black}\snum }}b_T} + \frac{\mk_2}{(2\pi)^{\dimn}\sqrt{{\color{black}\snum }}b_T} + \frac{1}{2}\mk_3\mk_4 \left(b_T\textcolor{black}{\sqrt{\frac{1}{P_{\Theta}}+y}}\right)^{2/\dimn} \right), \\
       & \quad \quad \psi_{\kfv,{\color{black}\snum }} = \frac{\kfv|\Theta|\sqrt{\zconst} - \sqrt{2\csix}\wassprox}{\sqrt{\zconst}|\Theta|\sqrt{b_T}}.
    \end{split}
\end{align}
    Here $\Theta$ is the compact set over which the parameters are restricted, \textcolor{black}{with size $|\Theta| := \int_{\Theta}du$},  $L=
\exp(\beta T_1)$ where $T_1$ is an upper bound on $\rew$ in $\Theta$ \textcolor{black}{(existing by \ref{ass:Msmooth})}, $\dimn$ is the dimension, constants $\mk_1, \mk_2,\mk_3,\mk_4$ are defined in \ref{as:kernspec}, $b_T$ is the kernel bandwidth, $\zconst$ is defined in \eqref{eq:Zint},  $\csix$ is defined in A\ref{ap:bd_consts}, and $x$ and $y$ are a free variable over which the concentration inequality may range.  
\end{theorem}
\begin{proof}
See Appendix~\ref{pf:kernest}.
\end{proof}

Theorem~\ref{thm:kernest} provides guarantees on the $L^1$ distance between the cost reconstruction $\hat{\rew}$ and the true cost $\rew$, on the compact sampling set $\Theta$. Next we provide insight into how the algorithmic parameters can be controlled in order to obtain a desired reconstruction proximity through~\eqref{eq:phidef}. 

\textcolor{black}{\subsubsection{Implementation Considerations}
\label{sec:impcon}
Here we discuss practical considerations for working with Algorithm~\ref{alg:costrec} and the concentration bound \eqref{eq:phidef}. First we provide details for obtaining an arbitrarily tight cost function reconstruction through ~\eqref{eq:phidef}. Then we discuss the total iteration complexity and computational complexity incurred by Algorithm \ref{alg:costrec}.}
\begin{itemize}
\item[-] \textit{Controlling the Reconstruction Proximity} In order to obtain useful tail behavior from this concentration inequality, we can set the following kernel bandwidth specification: 
\[\textrm{(S1)} \quad \textrm{Take } b_T \textrm{ such that }  \lim_{T\to\infty}b_T = 0,\, \lim_{T\to\infty}Tb_T^2 = \infty\]
We also must control the quantity $\zconst$, which can be done by introducing the following specification: 
\[\textrm{(S2)}\quad \textrm{Take compact domains }\Theta',\Theta : \,\, \Theta' \subset \Theta \subset \reals^{\dimn}, \,\, \inf_{x\in\del\Theta,y\in\del\Theta'}\|x-y\| > \Delta, \,\, \int_{\Theta'}d\gibbs(x) \geq \alpha\] where $\alpha>0$, $\Delta > \sqrt{\frac{\rho}{\alpha}}$, $\rho$ is the proximity in Algorithm~\ref{alg:costrec}, and $\del\Omega$ denotes the boundary of a compact domain $\Omega\subset\reals^{\dimn}$. \textcolor{black}{
\begin{corollary}
\label{cor:tight}
    Under (S1) and (S2), one may take $\kfv$ and $T$ appropriately, with $\wassprox \leq (2 b_T {\color{black}\snum })^{-1/2}$, such that both $\phi_{\kfv,T}$ and $\exp(-\psi^2_{\kfv,T})$ in \eqref{eq:phidef} become arbitrarily small.
\end{corollary}
\begin{proof}
Letting $\Theta$ in Algorithm~\ref{alg:costrec} be constructed such that (S2) holds, we have that $\zconst \in \left[(\alpha - \rho\Delta^{-2})^2,1\right]$ by Lemma~\ref{lem:otsubset}. One may then take $x$ and $y$ such that $\psi_{x,T}$ and $\exp(-2y^2/T^3)$ become \textit{arbitrarily large}, and then $T$ and $b_T$ such that $\phi_{x,T}$ becomes \textit{arbitrarily small}.
\end{proof}}
\textcolor{black}{We introduce Corollary~\ref{cor:tight} to show that in principle the bound \eqref{eq:phidef} can be made arbitrarily tight. For brevity we do not expand on the quantitative choices of parameters which ensure arbitrarily tight bounds; this will depend on factors such as the size of $\Theta$, which will inevitably be application-dependent.}

{\color{black}However, it should be clear how (S1) and (S2) can be satisfied in practice.}
A necessary condition for (S2) is that one must specify a region $\Theta'$ such that $\int_{\Theta'}d\gibbs(x) \geq \alpha$; \textcolor{black}{then $\Theta$ can simply be constructed by extending the boundary of $\Theta'$ by length $\Delta$}.  \textcolor{black}{A minimal sufficient condition for this is that the inverse learner knows at least one subset of the support of $\gibbs$ with non-zero Lebesgue measure. This immediately implies that $\int_{\Theta'}d\gibbs \geq \alpha$ is satisfied for some $\alpha>0$.} 

\item[-] {\color{black}\textit{Computational Complexity}. Observe that our reconstruction necessitates initialization of $\snum$ sequential streams of Algorithm~\ref{alg:psgld}. Thus, the total iteration complexity of achieving a specified reconstruction accuracy \eqref{eq:phidef} is $\mathcal{O}(\snum\hk)$, where $\hk$ is the necessary number of Algorithm~\ref{alg:psgld} iterations necessary to achieve a $\wassprox$ 2-Wasserstein distance $\wass(\pi_{\hk},\gibbs)$ under parameters initialized in lines 1-3 of Algorithm~\ref{alg:costrec}. So the sample complexity is linear in $\snum$. The total \textit{space complexity} is $O(T)$, since we only need to store $T$ total samples in memory, and the operation of each sequential algorithm requires only $O(1)$ space complexity. Also, observe that the bound \eqref{eq:phidef} gets tighter as $T$ increases; thus necessitating a trade-off between the cost reconstruction accuracy and computational cost incurred.}
\end{itemize}

\section{Illustrative Applications of Adaptive IRL}
\textcolor{black}{Here we show how the theory developed in the previous sections can be applied to practical learning frameworks. We   provide three examples of suitable "forward learning" processes \eqref{eq:sgd} that can be interfaced with our inverse learning techniques, Algorithms \ref{alg:psgld} and \ref{alg:costrec}. We motivate each example and show how they satisfy our necessary assumptions.}
\label{sec:examples}
\subsection{Adaptive Regularized MDP Inverse Reinforcement Learning}
\label{sec:mdpex}
Here we illustrate how the PSGLD algorithm (Alg~\ref{alg:psgld}), and finite-sample guarantees of Theorems~\ref{thm:main1} and \ref{thm:kernest}, interface with a policy gradient reinforcement learning (RL) algorithm. We first outline the details of the policy gradient RL algorithm, then show how our PSGLD algorithm can be used to achieve inverse RL with nonasymptotic accuracy guarantees. 


\subsubsection{\textit{Regularized Policy Gradient Algorithm}}
Here we present \textcolor{black}{a regularized version of} a policy gradient algorithm for reinforcement learning (RL), the Reinforce algorithm. Consider the following notation
\begin{itemize}
    \item[-] discount factor $\gamma \in [0,1)$, discrete-time index $t\in\N$
    \item[-] states $s_t \in \CS$, actions $a_t \in \CA$, cost $c_t = C(s_t,a_t),\, C:\CS \times \CA \to \reals_+, \, \,$ \textcolor{black}{where $\CS$ and $\CA$ are finite state and action spaces, respectively.}
    \item[-] state transition probabilities $P_{s,s'}^a = P(s_{t+1} = s' | s_t = s,a_t = a)$
    \item[-] initial state probability distribution $\rho_0(s) = \PR(s_0 = s)$
    \item[-] \textcolor{black}{Randomized Markovian} policy $\pi(s,a;\theta) = \PR[a_t=a | s_t=s; \theta]$ under policy parameter $\theta \in \reals^d$
\end{itemize}
Now let 
\begin{equation}
\label{eq:mdpJ}
J(\theta) := \CE_{\pi}\left[\sum_{t=0}^{\infty}\gamma^t c_t \right] + \lambda f(\theta) : \reals^d \to \reals_+
\end{equation}
be the expected discounted cumulative cost, where the expectation is taken with respect to the probability distribution of trajectories $(s_0,a_0,\dots,s_t,a_t,\dots)$ induced by $\rho_0, \, P_{(s,s')}^a,$ and $\pi(s,a;\theta)$. The RL goal is to find a policy parameter $\theta$ such that $J(\theta)$ is minimized. The regularization term $\lambda f(\theta)$ is typically introduced to penalize large values of the policy parameter $\theta$; when the policy is evaluated via a neural network this regularization prevents overfitting by controlling the network's expressive complexity \cite{cheng2019control},\cite{kumar2021dr3}. \textcolor{black}{Within our framework, we do not require conditions on the magnitude of $\lambda$ as long as it is positive. This is because Assumption~\ref{as:diss} is concerned only with \textit{asymptotic} growth properties of $\rew(\theta)$.}

\textcolor{black}{\textit{Remark. Entropy Regularization}. A common technique \cite{liu2019policy}, \cite{neu2017unified} for regularized reinforcement learning is to treat the regularization function $f(\theta)$ as the negative (since we minimize $\rew$) entropy $-H(\pi(s,\cdot;\theta)) = \sum_{a\in\CA}p(a|s;\theta)\log p(a|s;\theta)$. This entropy regularization can be derived as a special case of the general regularization \eqref{eq:mdpJ}. Indeed, here we provide one such constructive method for implementing this entropy regularization.
\begin{itemize}
    \item[-] \textit{Explicit Trigonometric Parametrization}: First observe that, since since $\CS$ and $\CA$ are finite spaces we may, without loss of generality, fully parametrize the policy function $\pi(s,a;\theta)$ by considering $\theta$ in some compact domain $C \subset \reals^{(|\CA|-1)|\CS|}$. This is because for each $s\in\CS$, the transition probabilities form $|\CA|$-dimensional vector in the unit simplex which can be parameterized in $|\CA|-1$ variables. So, we may form $\theta = [\theta_{1,1},\dots,\theta_{1,|\CA|-1},\theta_{2,1},\dots,\theta_{|\CS|,|\CA|-1}] \in C$, and for all $i\in|\CS|$ decompose the probability mass over $\CA$ by a trigonometric expansion. For example if $|\CA| = 3$, then we may take, for $\theta \in [0,\pi]^{2|\CS|}$,
    \begin{align*}
        \PR(a_1 | s_i; \theta) = \sin^2(\theta_{i,1})\sin^2(\theta_{i,2}),\quad
        \PR(a_2 | s_i; \theta) = \sin^2(\theta_{i,1})\cos^2(\theta_{i,2}),\quad
        \PR(a_3 | s_i; \theta) = \cos^2(\theta_{i,1})
    \end{align*}
    This construction can be made to accommodate any size action space $\CA$ by appending expansion terms such that $\sum_{j=1}^{|\CA|}\PR(a_j | s_i; \theta) = 1\,\,\forall i\in|\CS|$.
    \item[-] \textit{Satisfaction of Assumptions}: An advantage of the parametrization constructed above is that the entropy function $H(\pi(s,\cdot;\theta)) = -\sum_{a\in\CA}p(a|s;\theta)\log p(a|s;\theta)$ is Lipschitz continuous and smooth (satisfying Assumption~\ref{ass:Msmooth}) on the compact domain $C$. This can be easily checked analytically; we omit this for brevity. 
\end{itemize}
}



A local stationary point of $\rew(\cdot)$ can be achieved by sequentially updating $\theta$ via stochastic gradient descent. This is the methodology of the canonical Reinforce algorithm, which has sequential policy parameter updates given by
\begin{equation}
\label{alg:reinf}
\thkk = \thk - \eta\left[\sum_{t=0}^T\left\{\gamma^t\grad_{\theta}\log\pi(s_t,a_t;\theta)\sum_{k=t}^T \gamma^{k-t}\,r_k\right\}\textcolor{black}{ 
+\lambda \nabla_{\theta}f(\thk)}\right] = \thk - \eta\gradn J(\theta)
\end{equation}
where, crucially, $\gradn J(\theta)$ is an unbiased estimate of $\grad J(\theta)$ by the Policy Gradient Theorem \cite{sutton1999policy}. 


\subsubsection{Adaptive Inverse Reinforcement Learning. PSGLD}
Observe that the policy gradient algorithm~\eqref{alg:reinf} employs a stochastic gradient descent in the space of policies. This suggests that we can employ our PSGLD algorithm, taking as input the sequential policy evaluations and outputting the expected discounted cumulative cost \eqref{eq:mdpJ}. 

Suppose we observe an agent performing policy gradient RL by enacting sequential sample paths $\{(s_0,a_0, \dots,s_T,a_T)_k \sim \pi(\cdot,\cdot;\theta_k)\}_{k=1}^K$, and updating $\theta_k$ according to Algorithm~\ref{alg:reinf}. Notice that by observing a sample path $(s_0,a_0,\dots,s_T,a_T) \sim \pi(\cdot,\cdot;\theta_k)$ we can obtain an unbiased estimate {\color{black}$\hat{\theta}_k$ of $\theta_k$} by simply taking the empirical distribution of observed state-action pairs. Then form 
$\tilde{\grad} J(\theta_k) = \frac{\hat{\theta}_{k} - \hat{\theta}_{k-1}}{\eta}$, and notice that 
\[\CE\left[\tilde{\grad}J(\thk) |\, \thk,\,\thkm \right] = \gradn J(\thk) \]
i.e., $\tilde{\grad}J(\thk)$ is a second order stochastic gradient estimate; it is an unbiased estimate of the stochastic gradient $\gradn J(\thk)$, and $\gradn J(\thk)$ is fully determined by the (random) evaluations $\thk,\,\thkk$.  So we have 
\[\CE\left[\tilde{\grad}J(\thk)\right] = \CE\left[\CE\left[\tilde{\grad}J(\thk) | \,\thk,\,\thkk \right] \right] = \grad J(\thk)\]
and thus $\tilde{\grad}J(\thk)$ can be utilized in our passive stochastic gradient Langevin dynamics algorithm (Alg ~\ref{alg:psgld}) since it is an unbiased estimate of $\grad J(\thk)$. 

{\color{black} We now show that the policy gradient algorithm~\eqref{alg:reinf} satisfies the assumptions necessary to employ our PSGLD algorithm. 
\begin{itemize}
    \item[-] For the cost function $\rew$ \eqref{eq:mdpJ} to satisfy Assumption~\ref{ass:Msmooth} it is sufficient that for each $a\in\CA,s\in\CS$, $\pi(s,a;\theta)$ and $f(\theta)$ satisfy the regularity assumptions of Assumption~\ref{ass:Msmooth} with respect to $\theta$. We also observed that the policy function can be parametrized such that entropy regularization is Lipchitz continuous.
    \item[-] \textcolor{black}{Observe that since $\CS$ and $\CA$ are finite spaces, we may fully parametrize the policy function $\pi(s,a;\theta)$ by taking $\theta$ in the unit ball in $\reals^{|\CS|(|\CA|-1)}$, $B_1(\reals^{|\CS|(|\CA|-1)}) = \{x\in \reals^{|\CS|(|\CA|-1)} : \, \|x\|_1\leq 1\}$, with $\|\cdot\|_1$ the $L_1$ norm. Thus, Assumption~\ref{as:diss} is unnecessary since we do not consider $\theta$ on an unbounded domain.}
    \item[-] Assumption~\ref{ass:gradnoise} is satisfied, with noise variance $\gradnvar$ inversely proportional to the sample path length $T$. \item[-] Assumptions~\ref{ass:idistexp} and \ref{as:Bd_Der} are satisfied when the forward learner initializes its policy parameter $\theta_0$ according to any sub-Gaussian distribution, and Assumptions~\ref{ass:Kspec}, and \ref{ass:beta} can be satisfied by design.
\end{itemize}

Thus, Theorem~\ref{thm:main1} applies and we can control the non-asymptotic Wasserstein distance arbitrarily. Fixing any $\delta>0$, Algorithm~\ref{alg:psgld}, with parameters specified by \eqref{eq:stepspec} and with $\thk$ replaced by $\hat{\theta}_k$ above, produces iterates $\alpha_k \sim \dtlaw$, such that \[\wass(\dtlaw.\gibbs) \leq \mathcal{O}(\delta + \delta\sqrt{\log(1/\delta})), \quad \gibbs(\theta) \propto \exp(-\beta J(\theta))\] 

{\color{black} Then, Assumption~\ref{as:kernspec} can also be satisfied by design, so $\rew$ can be approximately recovered by employing Algorithm~\ref{alg:costrec}, with $L^1$ approximation guarantees given by Theorem~\ref{thm:kernest}.}

Note that traditional IRL methods aim to reconstruct $C(s,a)$, rather than $\rew(\theta)$, given \textit{optimal} policy demonstrations. In our case $C(s,a)$ can be recovered up to a constant multiplicative factor once $\rew(\theta)$, the MDP transition dynamics, {\color{black} and the regularization function (commonly $l^2$ norm)} are known, since $\rew(\theta)$ is the expectation of $C(s,a)$ with respect to the stationary measure induced by the policy $\pi(\cdot,\cdot;\theta)$ and the dynamics $P_{s,s'}^a$. Furthermore, in contrast to traditional methods \cite{ziebart2008maximum}, \cite{ng2000algorithms}, we operate in the transient regime where the observed agent is \textit{in the process of learning} an optimal policy.

\subsection{\textcolor{black} Bayesian Learning}
 Let $\theta$ denote a parameter with Bayesian prior $p(\theta)$, and with a data $x$ generation likelihood $p(x|\theta)$. The posterior distribution given a set $X = \{x_i\}_{i=1}^n$ of $n$ data points is $p(\theta | X) \propto p(\theta)\prod_{i=1}^n p(x_i|\theta)$. The goal of Bayesian learning is to find a maximum a-posteriori parameter $\theta^*$ such that $\theta^* \in \arg\max_{\theta} p(\theta)\prod_{i=1}^n p(x_i|\theta)$. \cite{welling2011bayesian} propose to use classical Langevin dynamics to accomplish this opimization problem, employing gradients of the form 
\[\grad_{\theta}\log p(\theta) + \sum_{i=1}^n \grad_{\theta}\log p(x_i | \theta).\]

Learning the posterior distribution $p(\theta|X)$ from observations of this "forward" optimization process has numerous practical applications, see \cite{krishnamurthy2021langevin} for details. We (as the "inverse" learner) may reconstruct this posterior via Algorithms~\ref{alg:psgld} and \ref{alg:costrec}. Indeed, we now show that our assumptions can be satisfied by this forward learning process. First, notice that the maximization 
 of $\log \left(p(\theta)\prod_{i=1}^n  p(x_i|\theta)\right)$ is equivalent to the minimization of $-\log \left(p(\theta)\prod_{i=1}^n p(x_i|\theta)\right)$. Assumption~\ref{ass:Msmooth} is satisfied when the prior $p(\theta)$ and likelihood functions $p(x|\theta)$ are both (sufficiently) smooth with respect to $\theta$. Assumption~\ref{as:diss} is satisfied since $\grad_{\theta}\log p(\theta)$ eventually tends to infinity super-linearly when $p(\theta)$ is Gaussian, for example. Assumption~\ref{ass:gradnoise} is satisfied in standard SGLD implementations, Assumptions~\ref{ass:idistexp}, \ref{as:Bd_Der} are satisfied by Gaussian initializations, and Assumptions~\ref{ass:Kspec}, \ref{ass:beta}, \ref{as:kernspec} can be satisfied by design. Furthermore, as per the "Dependence on Forward Process" discussion in Section~\ref{sec:main}, Theorems~\ref{thm:main1} and \ref{thm:kernest} still hold when we implement forward learning \textit{SGLD} iterates rather than \textit{SGD} iterates in Algorithms~\ref{alg:psgld} and \ref{alg:costrec}.

\vspace{0.1cm}
\subsection{\textcolor{black} Empirical Risk Minimization}
Consider the stochastic optimization problem 
\[\textrm{minimize } \,\, F(w) := \CE_{P}[f(w,Z)] = \int_Z f(w,z)P(dz)\]
where $w \in \reals^d$ and $Z$ is a random vector with unknown probability law $P$. Under empirical risk minimization, one has access to an $n$-tuple $\bold{Z} = (z_1,\dots,z_n)$ of i.i.d. samples from $P$, and aims to minimize 
\[F_{\bold{Z}}(w) := \frac{1}{n}\sum_{i=1}^n f(w,z_i).\]
 \cite{raginsky2017non} proposes to use classical SGLD to approximately minimize this (non-convex) empirical loss function. Then, suppose we as the "inverse" learner have access to sequential SGLD iterates and stochastic gradients\footnote{Observe that in this case we must have access to both, since one cannot be recovered from the other due to the additional noise in SGLD.} from this "forward" learning process. We may use Algorithms~\ref{alg:psgld} and \ref{alg:costrec} to obtain an estimate of $F_{\bold{Z}}(w)$ with proximity guarantees given by Theorems~\ref{thm:main1} and \ref{thm:kernest}. This may also serve as an approximation of $F(w)$ (with approximation error quantified in terms of $f$ regularity and dataset $\bold{Z}$ richness; we do not investigate this here). 
 
It is clear upon inspection that our assumptions can be satisfied in this case. Indeed, Assumption A.2 and A.3 (imposed via e.g., $l^2$ regularization) of \cite{raginsky2017non} are equivalent to our Assumptions~\ref{ass:Msmooth}, \ref{as:diss}. The remaining assumptions are satisfied for the same reasons as in the previous example.

\section{Preliminaries for Proof of Theorem~\ref{thm:main1}}
\label{sec:proofprelim}
This short section briefly summarizes the main tools required for the proof of Theorem~\ref{thm:main1}. 

\subsection{Infinitesimal Generator}
 Let $X_t$ be an $\reals^{\dimn}$-valued diffusion defined by the stochastic differential equation
\begin{equation}
\label{eq:itodiff}
    dX_t = b(X_t)\, dt + \sigma(X_t)\, d\brownt, \quad X_0 = x \in \reals^{\dimn}
\end{equation}
where $b: \reals^{\dimn} \to \reals^{\dimn}$ is the drift function, $\sigma: \reals^{\dimn} \to \reals$ is the volatility function, and $\brownt$ is standard $\dimn$-dimensional Brownian motion. Fixing a point $x \in \reals^{\dimn}$, let $P^x$ denote the law of $X_t$ given $X_0 = x$, and $\CE^x$ denote expectation with respect to $P^x$. 
Let $\gen$ be the \textit{infinitesimal generator} of $X_t$, defined by its action on compactly-supported $C^2$ functions $f:\reals^{\dimn}\to\reals$ in domain $\gendom \subseteq C^2(\reals^{\dimn})$, as
\begin{align}
    \begin{split}
        \gen f(x) &= \lim_{t \downarrow 0}\frac{\CE^x[f(X_t) - f(x)]}{t} = \sum_{i=1}^n b_i(x) \frac{\del f}{\del x_i}(x) + \frac{1}{2}\sum_{i,j} \sigma^2(x) \frac{\del^2 f}{\del x_i \del x_j}(x)
    \end{split}
\end{align}
where $b_i(x)$ is the $i$'th element of $b(x) \in \reals^{\dimn}$. Thus $\gen$ is an operator acting on $f \in C^2(\reals^{\dimn})$ as \[\gen f = \frac{1}{2}\sigma^2 \Delta f+ \langle b, \grad f \rangle \]
where $\Delta := \grad \cdot \grad$ denotes the standard Laplacian operator. We say $\pi$ is an invariant probability measure w.r.t $\gen$ if and only if $\int_{\reals^{\dimn}}\gen g d\pi = 0$ for all $g \in \gendom$. 

In this work we consider the diffusion which solves the stochastic differential equation \eqref{eq:ct_diff}, which has:
\[ b(x) = -\frac{\beta}{2} \idistrw^2(x) \grad \rew(x) \,\color{black}{+}\,\idistrw(x) \grad \idistrw(x), \quad \sigma(x) = \idistrw(x)\]

Thus, the infinitesimal generator of our diffusion process \eqref{eq:ct_diff} is given as 
\begin{equation}
\label{eq:diffgen}
    \gen f = \frac{1}{2}\idistrw^2 \Delta f - \frac{\beta}{2} \idistrw^2 \langle \grad \rew , \grad f \rangle  \,\color{black}{+}\, \idistrw \langle \grad \idistrw, \grad f\rangle
\end{equation}
and note that by assumptions (A2),(A6) and by Theorem 2.5 of \cite{karatzas1991brownian}, we have that \eqref{eq:ct_diff} admits a unique strong solution. 
\subsection{Poincar\'e and Logarithmic Sobolev Inequalities}
\label{sec:plsineq}


Considering a general infinitesimal generator $\gen$, with stationary measure $\pi$, we can define the \textit{Dirichlet form}
\[ \dirform(g) := -\int_{\reals^{\dimn}}g\gen g d\pi\]
and the \textit{spectral gap} $\specgap$ as 
\begin{align}
\begin{split}
\label{eq:sgap}
\specgap &:= \inf\bigl\{\frac{\int_{\reals^{\dimn}}\dirform(g) d\pi}{\int_{\reals^{\dimn}}g^2 d\pi} \ : \ g \in C^1(\reals^{\dimn})\cap L^2(\pi),\,g \neq 0,\, \int_{\reals^{\dimn}}gd\gibbs = 0 \bigr\} \\
\end{split}
\end{align}

Let us consider a Markov process $X_t$ with unique invariant distribution $\pi$ and infinitesimal generator $\gen$. We say that $\pi$ satisfies a \textit{Poincar\'e (spectral gap) inequality} with constant $c$ if 
\begin{equation}
\label{eq:poinc_ineq}
    \chi^2(\mu || \pi) \leq c\, \dirform\left(\sqrt{\frac{d\mu}{d\pi}}\right)
\end{equation} for all probability measures $\mu \ll \pi$ ($\mu$ absolutely continuous w.r.t $\pi$),
where \\ $\chi^2(\mu || \pi) := ||\frac{d\mu}{d\pi} - 1||^2_{L^2(\pi)}$ is the $\chi^2$ divergence between $\mu$ and $\pi$. If \eqref{eq:poinc_ineq} is satisfied for some $c$, then we have $\frac{1}{c} \leq \specgap$ where $\specgap$ is the spectral gap given in \eqref{eq:sgap}. In particular, letting $\pconst$ denote the \textit{Poincar\'e constant}, given as the smallest $c$ such that \eqref{eq:poinc_ineq} holds,
\[\pconst = \inf\{c :  \chi^2(\mu || \pi) \leq c\, \dirform\left(\sqrt{\frac{d\mu}{d\pi}}\right) \ \forall \mu \ll \pi\}\]
where $\ll$ denotes absolute continuity, then we have $\frac{1}{\pconst} = \specgap$, and the eigenspectrum of $-\gen$ is contained in $\{0\}\cup[\frac{1}{\pconst},\infty)$.

We say that $\pi$ satisfies a \textit{logarithmic Sobolev inequality} with constant $c$ if \[\KL(\mu || \pi) \leq 2\,c\,\dirform\left(\sqrt{\frac{d\mu}{d\pi}} \right)\] for all $\mu \ll \pi$,
where $\KL(\mu || \pi) = \int d\mu \log \frac{d\mu}{d\pi}$ is the Kullback-Leibler divergence. 

One of the main efforts of this work will be to show that the diffusion \eqref{eq:ct_diff} satisfies a log-Sobolev inequality; this then allows us to utilize several useful properties in the non-asymptotic analysis. Specifically, letting $\{X(t)\}_{t\geq 0}$ be a Markov process with stationary distribution $\pi$ and Dirichlet form $\dirform$, then we have:

\begin{lemma}[Exponential decay of entropy 
\cite{bakry2014analysis}, Th. 5.2.1]
\label{lem:expdecay}
Let $\mu_t := \law(X(t))$. If $\pi$ satisfies a logarithmic-Sobolev inequality with constant $c$, then 
\begin{equation*}
    \KL(\mu_t || \pi) \leq \KL(\mu_0 || \pi)e^{-2t/c}
\end{equation*}
\end{lemma}

\begin{lemma}[Otto-Villani theorem \cite{bakry2014analysis}, Th. 9.6.1] 
If $\pi$ satisfies a logarithmic-Sobolev inequality with constant $c$, then, for any $\mu \ll \pi$
\label{lem:OVthm}
\begin{equation*}
    \wass(\mu,\pi) \leq \sqrt{2c\KL(\mu||\pi)}
\end{equation*}
\end{lemma}

The following two results give sufficient conditions for a measure $\pi$ to satisfy Poincare and logarithmic-Sobolev inequalities, using Lyapunov function criteria. 

\begin{proposition}[Bakry 2008 \cite{bakry2008simple}]
\label{prop:Bakry}
Let $\pi(dx) = \exp(-H(x))dx$ be a probability measure on $\reals^{\dimn}$ with $H \in C^2(\reals^{\dimn})$ and lower bounded. Let $\gen$ be the infinitesimal generator of a Markov process with stationary measure $\pi$. Suppose there exist constants $\bc, \bcc > 0$, $\bccc\geq 0$ and a $C^2$ function $V: \reals^{\dimn} \to [1,\infty)$ such that
\begin{equation}
\label{eq:lyapcondb}
    \frac{\gen V(w)}{V(w)} \leq -\bcc + \bc\ind\{\|w\|\leq \bccc\}
\end{equation}
Then $\pi$ satisfies a Poincar\'e inequality with constant 
\begin{equation*}
    \pconst \leq \frac{1}{\bcc}\left(1 + C\bc\bccc^2\exp(O_{\bccc}(H)) \right)
\end{equation*}
where $C>0$ is a universal constant and $O_{\bccc}(H) := \max_{\|w\|\leq \bccc}H(w) - \min_{\|w\|\leq \bccc}H(w)$
\end{proposition}

\begin{proposition}[Cattiaux et. al. (2010) \cite{cattiaux2008note}]
\label{prop:cat}
Let $\pi(dx) = \exp(-H(x))dx$ be a probability measure on $\reals^{\dimn}$ with $H \in C^2(\reals^{\dimn})$ and lower bounded. Let $\gen$ be the infinitesimal generator of a Markov process with stationary measure $\pi$.
Suppose the following conditions hold:
\begin{enumerate}
\item There exist constants $\kappa, \gamma > 0$ and a $C^2$ function $V: \reals^d \rightarrow [1,\infty)$ such that 
\begin{equation}
\label{eq:lyapcond}
    \frac{\gen V(w)}{V(w)} \leq \kappa - \gamma \|w\|^2  \ \forall w \in \reals^d
\end{equation}
\item $\gibbs$ satisfies a Poincar\'e inequality with constant $\pconst$.
\item There exists some constant $K \geq 0$, such that $\grad^2 H \succcurlyeq -K I_d$
\end{enumerate}

Let $Z_1, Z_2$ be defined, for some $\zeta>0$ as
\begin{equation*}
    Z_1 = \frac{2}{\gamma}\left(\frac{1}{\zeta} + \frac{K}{2} \right) + \zeta, \quad Z_2 = \frac{2}{\gamma}\left(\frac{1}{\zeta} + \frac{K}{2} \right)\left( \kappa + \gamma\int_{\reals^{\dimn}}
    \|w\|^2 \pi(dw)\right)
\end{equation*}
Then $\pi$ satisfies a logarithmic Sobolev inequality with constant $\LSconst = Z_1 + (Z_2 + 2)\pconst$.
\end{proposition}

We will be interested in showing that the invariant (Gibbs) measure $\gibbs$ \eqref{eq:gibbsd} of our particular diffusion \eqref{eq:ct_diff} satisfies a log-Sobolev inequality, so that we can apply Lemmas \ref{lem:expdecay} and \ref{lem:OVthm} to obtain exponential convergence of $\wass(\ctlaw,\gibbs)$. To show the log-Sobolev inequality holds, we will show that the conditions of Proposition~\ref{prop:cat} hold, using Proposition~\ref{prop:Bakry} as an intermediate step. More details on this procedure will be outlined in Section~\ref{sec:resultspf}.

The following result is unrelated to Poincar\'e and log-Sobolev inequalities, but gives a way to bound a general Wasserstein distance once a KL-divergence is known. We will utilize this in the bound on $\wass(\dtlaw,\ctlaw)$.
\begin{corollary}[Bolley and Villani 2005 \cite{bolley2005weighted} Cor. 2.3]
\label{cor:Villani}
For any two Borel probability measures $\mu, \nu$ on $\reals^{\dimn}$, 
\[ \wass(\mu,\nu) \leq 2 \inf_{\lambda > 0}\left(\frac{1}{\lambda}\left(\frac{3}{2} + \log\int_{\reals^{\dimn}}e^{\lambda\|w\|^2}\nu(dw)\right)\right)^{1/2}\left[\sqrt{\KL(\mu || \nu)} + \left(\frac{\KL(\mu || \nu)}{2}\right)^{1/4}\right]\]
\end{corollary}


\section{Proof of Main Result (Theorem~\ref{thm:main1}). Outline}
\label{sec:resultspf}
Here we provide the proof structure for our bound on $\wass(\dtlaw,\gibbs)$, provided as \eqref{eq:wassbound} in Theorem~\ref{thm:main1} (the complete proof details can be found in the Appendix). The block diagram in Figure~\ref{fig:blockdiag} displays the relations between our main supporting results in the proof of Theorem~\ref{thm:main1}. The high level proof structure is as follows: We bound $\wass(\dtlaw,\gibbs) \leq \wass(\dtlaw,\ctlaw) + \wass(\ctlaw,\gibbs)$, i.e., we first control the discretization error between passive algorithm~\ref{alg:psgld} and diffusion~\ref{eq:ct_diff}, then control the convergence rate of this diffusion to its stationary distribution $\gibbs$. All complete proofs and supporting technical Lemmas can be found in the Appendix.

In order to achieve a useful bound on the former, scaling as $\mathcal{O}(k\step\sqrt{\step})$, we employ a Girsanov change of measure (controlling the KL-divergence), given as Lemma~\ref{lem:KL_bd}, followed by Corollary~\ref{cor:Villani} (to relate back to 2-Wasserstein distance), as in \cite{raginsky2017non}. This procedure relies crucially on the exponential integrability of the diffusion \eqref{eq:ct_diff}, which we prove as Lemma~\ref{lem:exp_int}. To handle lack of measure absolute continuity, as discussed below, we must introduce an intermediate process (with law $\ytlaw$), perform the above procedure on the error between $\ytlaw$ and $\ctlaw$, then bound $\wass(\dtlaw,\ctlaw) \leq \wass(\dtlaw,\ytlaw) + \wass(\ytlaw,\ctlaw)$. The result providing a bound on $\wass(\dtlaw,\ytlaw)$, completing this approach, is given as Lemma~\ref{lem:MSEbd}.

To bound $\wass(\ctlaw,\gibbs)$, we first show that $\gibbs$ satisfies a logarithmic-Sobolev inequality, by satisfying the conditions of Proposition~\ref{prop:cat} \cite{cattiaux2008note}. This result is given as Proposition~\ref{prop:logsob}. We then apply exponential decay of entropy \cite{bakry2014analysis}, given as Lemma~\ref{lem:expdecay}, and the Otto-Villani Theorem \cite{bakry2008simple}, given as \ref{lem:OVthm}. This procedure provides an exponentially decaying bound on $\wass(\ctlaw,\gibbs)$.

\begin{figure}
\centering
{\hspace*{-0.2cm}\begin{tikzpicture}[node distance=1.2cm, block/.style={rectangle, draw, minimum width=2cm, minimum height=0.6cm},
    arrow/.style={single arrow, draw, minimum height=0.5cm, minimum width=1cm, shape border rotate=-90},
    bentarrow/.style={->, bend right=90}]
    \tikzstyle{split}=[rectangle split,rectangle split parts=2,draw,text centered]
    \tikzstyle{every node}=[font=\footnotesize]

    \node[split] (A) {Diffusion Approximation \nodepart{second} $\wass(\dtlaw,\ctlaw) \leq \dots$};
    \coordinate (BelowA) at ($(A.south)+ (0,-0.7cm)$);
    \node[split, left=0.7cm of BelowA] (B) {Lemma~\ref{lem:MSEbd} \nodepart{second} $\wass(\dtlaw,\ytlaw) \leq \dots$};
    \draw[bentarrow] (B.north) |- (A.west);

    
    \node[split,right=0.8cm of B] (C44) {Corollary~\ref{cor:Villani} \cite{bolley2005weighted} \nodepart{second} $\wass(\ytlaw,\ctlaw) \leq C \,f(\KL(\ytlaw\|\ctlaw))$};
    \coordinate (BelowC44) at ($(C44.south)+ (0,-0.7cm)$);
    \draw[bentarrow] (C44.north) |- (A.east);
    
    \node[split, below=0.2cm of C44] (L75) {Lemma~\ref{lem:exp_int} \nodepart{second} $C \leq \dots$};
    \node[split, left=1.3cm of BelowC44] (L52) {Lemma~\ref{lem:KL_bd} \nodepart{second} $\KL(\ytlaw\|\ctlaw) \leq \dots$};

    \draw[bentarrow] (L52.north) |- (C44.west);
    \draw[->] (L75.north) -- (C44.south);



    \node[split, right=5cm of A] (D) {Diffusion Convergence  \nodepart{second} $\wass(\ctlaw,\gibbs) \leq \dots$};
    \node[split, below=0.15cm of D] (E) {Lemma~\ref{lem:OVthm} \cite{bakry2014analysis} \nodepart{second} $\wass(\ctlaw,\gibbs) \leq \sqrt{2\LSconst\KL(\ctlaw\|\gibbs)}$};
    \node[split, below=0.2cm of E] (F) {Lemma~\ref{lem:expdecay} \cite{bakry2014analysis} \nodepart{second} $\KL (\ctlaw|| \gibbs) \leq \KL(\idistrw || \gibbs)e^{-2k\step/\beta \LSconst}$};
    \coordinate (BelowF) at ($(F.south)+ (0,-0.6cm)$);
    \node[split, left=1cm of BelowF] (G) {Proposition~\ref{prop:logsob} \nodepart{second} $\gibbs$ satisfies log-Sobolev inequality, with $\LSconst \leq \dots$};

    \draw[bentarrow] (G.east) -| (F.south);

    \node[split, above=0.25cm of $(A.north)!0.5!(D.north)$] (Process) {Main result (Theorem~\ref{thm:main1}) \nodepart{second} $\wass(\dtlaw,\gibbs) \leq \dots$};

    \draw[bentarrow] (A.north) |- (Process.west);
    \draw[bentarrow] (D.north) |- (Process.east);


    \draw[->] (F.north) -- (E.south);
    \draw[->] (E.north) -- (D.south);
    
\end{tikzpicture}}
\caption{\small Theorem~\ref{thm:main1} proof structure. First the 2-Wasserstein distance between discrete-time algorithm \eqref{eq:dt_sgld} (with measure $\dtlaw$) and continuous-time diffusion \eqref{eq:ct_diff} (with measure $\ctlaw$) is bounded. We must introduce an intermediate process (with law $\ytlaw$). Lemma~\ref{lem:MSEbd} bounds the Wasserstein distance between $\dtlaw$ and $\ytlaw$. Lemma~\ref{lem:KL_bd} bounds the KL-divergence between $\dtlaw$ and $\ytlaw$. Corollary~\ref{cor:Villani} is then used, along with Lemma~\ref{lem:exp_int} to relate this KL bound to a 2-Wasserstein bound. Proposition~\ref{prop:logsob} is the key tool in bounding $\wass(\ctlaw,\gibbs)$, establishing that $\gibbs$ satisfies a log-Sobolev inequality. We then employ exponential decay of entropy (Lemma~\ref{lem:expdecay}) and the Otto-Villani Theorem (Lemma~\ref{lem:OVthm}) to obtain exponential decay of $\wass(\ctlaw,\gibbs)$.}
\label{fig:blockdiag}
\end{figure}

\subsection{2-Wasserstein Bound for Diffusion Approximation}
\label{subsec:diffapprox}
Here we obtain a bound on $\wass(\dtlaw,\ctlaw)$. 
 Consider the continuous-time interpolation of the process~\eqref{eq:dt_sgld}: 
 \begin{align}
 \begin{split}
 \label{eq:dt_interp}
 \altbar = \alpha_0 &- \int_0^t \left[\Kd(\theta_{\sbar}, \alsbbar) \frac{\beta}{2} \gradn\rew(\theta_{\sbar}) \,\color{black}{-}\, \grad\idistrw(\alsbbar)\right]\idistrw(\alsbbar)ds + \int_0^t\idistrw(\alsbbar)d\browns
\end{split}
\end{align}
where $\sbar = \lfloor s/\step \rfloor \step$, and $\theta_{\sbar} := \thk$ for $k=\lfloor s/\step \rfloor$. Note that, for each $k$, $\bar{\alpha}(k\step)$ and $\alk$ have the same probability law $\dtlaw$. We aim to relate this process to the diffusion \eqref{eq:ct_diff} through a Girsanov change of measure; but the process \eqref{eq:dt_interp} is not Markovian and is therefore not an It\^o diffusion. However, by \color{black}{Theorem 4.6} of \cite{gyongy1986mimicking}, the process $\altbar$ has the same one-time marginals as the It\^o process $\yt$, where
\begin{gather}
\begin{aligned}
\begin{split}
\label{eq:xtgyongy}
\yt &= \alpha_0 - \int_0^t g_s(\theta_{\sbar},Y(s))ds + \int_0^t\CE\left[ \idistrw(\alsbbar) \,\vert \,\alsbar = Y(s) \right]d\browns\\
g_s(\theta_{\sbar}, Y(s)) &= \CE\left[\left(\Kd(\theta_{\sbar},\alsbbar)\frac{\beta}{2}\gradn \rew(\theta_{\sbar}) \,\color{black}{-}\,\grad\idistrw(\alsbbar)\right)\idistrw(\alsbbar) \, \bigg| \, \alsbar = Y(s) \right]
\end{split}
\end{aligned}\raisetag{-1.3\baselineskip}
\end{gather}
i.e., $\textrm{Law}(Y(k\step)) = \dtlaw \,\forall k \in \N$, and it is apparent that \eqref{eq:xtgyongy} is Markovian. 
However, we cannot apply Girsanov's formula to relate \eqref{eq:xtgyongy} and \eqref{eq:ct_diff} because the volatility functions are different; so the measures $\dtlaw$ and $\ctlaw$ are not absolutely continuous.

To solve this, we introduce the intermediate process
\begin{equation}
\label{eq:ytgyongy}
\xt = \alpha_0 - \int_0^t \hat{g}_s(\theta_{\sbar},X(s))ds + \int_0^t \idistrw(X(s))d\browns
\end{equation}
where 
$\hat{g}_s(\theta_{\sbar},X(s)) = \left(\Kd(\theta_{\sbar},X(s))\frac{\beta}{2}\gradn \rew(\theta_{\sbar}) + \grad\idistrw(X(s))\right)\idistrw(X(s))$. 
Let $\ytlaw$ denote the law of \eqref{eq:ytgyongy} at time $t=k \step$. 
This process \eqref{eq:ytgyongy} is similar enough to \eqref{eq:xtgyongy} to allow a tractable bound on $\wass(\dtlaw,\ytlaw)$, and since \eqref{eq:ytgyongy} has the same volatility function as \eqref{eq:ct_diff} we can relate these two via Girsanov's formula to obtain a desirable bound on $\wass(\ytlaw,\ctlaw)$. Then we simply bound $\wass(\dtlaw,\ctlaw) \leq \wass(\dtlaw,\ytlaw) + \wass(\ytlaw,\ctlaw)$.


The following Lemma provides a bound on $\wass(\dtlaw,\ytlaw)$.
\begin{lemma}
\label{lem:MSEbd}
    Fixing the step size $\step$ and time horizon $k \step$, take the kernel scale parameter $\Delta$ small enough to satisfy \eqref{eq:delspec}, and sampling distribution scale parameter $\idw$ small enough to satisfy \eqref{eq:omegaspec}. Then we have 
    \[\wass(\dtlaw,\ytlaw) \leq 6(k\step)\step\sqrt{12C_0 + 3} + 3\sqrt{2(k\step)\step}\]
\end{lemma}
\color{black}{Here} $C_0$ is a constant provided in \color{black}{Appendix}~\ref{ap:bd_consts}, 
$\etsup$ is a uniform bound on $\CE\|\thk\|^2$, see Lemma~\ref{lem:unifL2}, $\lipGJ$ is the Lipchitz constant for $\grad\rew$, see Assumption~\ref{ass:Msmooth}, $B =\|\grad\rew(0)\|$,  and $\gradnvar$ is the uniform noise variance bound in Assumption~\ref{ass:gradnoise}.
Now, the following Lemma provides a bound on $\KL(\ytlaw\| \ctlaw)$, the Kullback-Leibler (KL) divergence between measures $\ytlaw$ and $\ctlaw$, using a similar Girsanov change of measure as presented in \cite{raginsky2017non}.
\begin{lemma}
\label{lem:KL_bd}
Fixing the step size $\step$ and time horizon $k \step$, taking  the kernel scale parameter $\Delta$ small enough to satisfy \eqref{eq:delspec}, and sampling distribution scale parameter $\idw$ small enough to satisfy \eqref{eq:omegaspec}. Then we have:
\[\KL(\ytlaw \| \ctlaw) \leq {\color{black}(k\step)^3\,\step^3\,\left[2\beta C_5 \lipGJ^2\left(72C_0 + 6\sqrt{C_0} + 18+\sqrt{2} \right) \right] + 2\beta\,(k\step)\,\step \,C_5\,(\lipJ^2 + \gradnvar)}\]
\end{lemma}

Now we relate this KL divergence to a Wasserstein distance through Corollary~\ref{cor:Villani}, with $\mu = \ytlaw, \,\nu = \ctlaw, \,\lambda = 1$.
By Lemma~\ref{lem:exp_int} (Appendix) we have exponential integrability of the diffusion~\eqref{eq:ct_diff}:
\begin{align*}
    \begin{split}
        \log \int_{\reals^{\dimn}}e^{\lambda\|w\|^2}\ctlaw(dw) &\leq\knw + ((\beta \dissb + \dimn)2\step + 2 \idiffcc) k\step
    \end{split}
\end{align*}
Now, since $k\step \geq 1$ by \eqref{eq:delmax} and \eqref{eq:itspec}, and $\knw \leq \kn$, we can bound 
\begin{align*}
    \wass (\ytlaw,\ctlaw) &\leq 2\sqrt{\frac{3}{2} + (\kn + (\beta \dissb + \dimn)2\step + 2 \idiffcc)k \step }\biggl(\sqrt{\KL(\ytlaw || \ctlaw)} + \left(\KL(\ytlaw || \ctlaw)\right)^{1/4}\biggr)
\end{align*}
Applying Lemma~\ref{lem:KL_bd}, we have:
\begin{align*}
    \wass (\ytlaw,\ctlaw)
    & \leq {\color{black} 4\sqrt{\frac{3}{2} + C_1 k \step }\sqrt{k\step}\sqrt{\step}\left(\sqrt{2\beta C_5 \lipGJ^2 C_2 } + \sqrt{2\beta C_5 (\lipJ^2 + \gradnvar)} \right)}
\end{align*}
with $C_0, C_1, C_2, C_5$ in Appendix~\ref{ap:bd_consts}. So finally,
\begin{align}
\begin{split}
\label{eq:wassdt}
\wass(\dtlaw,\ctlaw) 
&\leq {\color{black} k\step\sqrt{\step}\left[6\sqrt{12C_0 + 3} + 3\sqrt{2} + 4\sqrt{3/2 + C_1}\left(\sqrt{2\beta C_5 \lipGJ^2 C_2} + 2\sqrt{2\beta C_5(\lipJ^2 + \gradnvar)} \right)\right]}
\end{split}
\end{align}
So we achieve a discretization error bound $\wass(\dtlaw,\ctlaw)$ which scales as $\mathcal{O}(k\step\sqrt{\step})$. In fact, this is tighter than the bound obtained in \cite{raginsky2017non}, which scales as $\mathcal{O}(k\step\,\step^{1/4})$. We represent this bound in terms of distinct units $k\step$ and $\sqrt{\step}$ (rather than $k\step^{3/2}$) since in our final analysis we will take $k\step$ large enough (but fixed), then $\step$ small enough, so that $\wass(\dtlaw,\ctlaw)$ decreases arbitrarily. We will need to first take $k\step$ large enough to control the diffusion \eqref{eq:ct_diff} proximity to the Gibbs measure. The following presents this proximity in terms of exponentially decaying distance $\wass(\ctlaw,\gibbs)$.

\subsection{2-Wasserstein Distance for Diffusion Convergence}
\label{sec:DC}

Here we describe the method to bound $\wass(\ctlaw,\gibbs)$. The strategy is as follows:
\begin{enumerate}[label=\roman*)]
    \item Show that $\gibbs$ satisfies a logarithmic-Sobolev inequality.
    \item Apply exponential decay of entropy, given as Lemma~\ref{lem:expdecay}, with the relative entropy bound in Lemma~\ref{lem:relentropy}, to derive a bound on $\KL(\ctlaw\|\gibbs)$
    \item Apply the Otto-Villani Theorem, given as Lemma~\ref{lem:OVthm}, to relate this to a bound on $\wass(\ctlaw,\gibbs)$.
\end{enumerate}
We accomplish (i) in the following proposition, establishing that the Gibbs measure $\gibbs$ satisfies a log-Sobolev inequality:
\begin{proposition}
\label{prop:logsob}
For $\beta$ satisfying Assumption~\ref{ass:beta}, the Gibbs measure $\gibbs$ satisfies a logarithmic Sobolev inequality with constant $\LSconst$:
\begin{align}
    \begin{split}
    \label{eq:lsconst}
        &0 \leq \LSconst \leq \frac{2\beta \lipGJ}{\gamma} + \frac{2}{\beta \lipGJ}  + \frac{1}{\specgap} \biggl(\frac{2\beta \lipGJ}{\gamma}\left( \kappa + \gamma\left(\kn + \frac{(\beta \dissb + \dimn)\idistmaxw + 2\idiffc}{(\dissm\beta)\idistmaxw}\right)\right) + 2\biggr)
    \end{split}
\end{align}
\end{proposition}
where 
\begin{align}
    \begin{split}
    \label{eq:kg}
        &\frac{1}{\specgap} \leq \frac{1}{2\kappa} \biggl(1 + \frac{4C\kappa^2}{\gamma}\exp\biggl(\beta\biggl(\frac{(\lipGJ+\constB)\kappa}{\gamma} + \constA + \constB\biggr)\biggr) \biggr)\\
         &\kappa = \biggl(\frac{1}{2} \beta\dissm \dimn  + \beta \dissm \idiffc \biggr) + \frac{1}{2}\biggl[ \beta^2\dissm \dissb  + \left(\beta\dissm\tailbound\right)^2 \biggr], \quad \gamma = \frac{1}{2}\left((\beta\dissm)^2 + \left( 1-\frac{1}{\idistmax^2+1}\right)\right) 
    \end{split}
\end{align}
and $\idistmaxw = \sup_{x}\idistrw(x), \, \, \idistmax = \sup_{x}\idistr(x)$.\\
\textit{Proof Sketch}:
The full proof is available in Appendix~\ref{ap:resultspf}. The key tool we use is the main Theorem in \cite{cattiaux2008note}, reproduced as Proposition~\ref{prop:cat}. To satisfy condition (1) of Proposition~\ref{prop:cat} we show that the Lyapunov function \[V(w) = \exp\left(\frac{\beta \dissm\|w\|^2 }{2(\idistmaxw^2 +1)}\right)\]
and the infinitesimal generator \eqref{eq:diffgen} satisfy \eqref{eq:lyapcond}, with $\kappa$ and $\gamma$ given in \eqref{eq:kg}. Then, Proposition~\ref{prop:Bakry} is used to show that condition (2) is satisfied. Condition (3) is satisfied with $K = \beta \lipGJ$ by assumption~\ref{ass:Msmooth}. 

Now since $\KL (\nu_0|| \gibbs) = \KL (\idistr || \gibbs) < \infty$ by Lemma \ref{lem:relentropy}, we can apply the exponential decay of entropy (Lemma~\ref{lem:expdecay}) to obtain
\begin{equation}
\label{eq:CTbound}
\KL (\ctlawt|| \gibbs) \leq \KL(\idistrw || \gibbs)e^{-2t/\beta \LSconst} 
\end{equation}
Then by the Otto-Villani Theorem and Lemma~\ref{lem:relentropy}, we have 
\begin{align}
\begin{split}
\label{eq:ct_bd}
&\wass(\ctlawt,\gibbs) \leq \sqrt{2\LSconst \relentbdw} e^{-t/\beta \LSconst} 
\end{split}
\end{align}
where $\relentbdw$ is the relative entropy bound given in \eqref{eq:relentropy} and $\LSconst$ is bounded in \eqref{eq:lsconst}.

\subsection{Controlling the 2-Wasserstein Distance}
Combining the bounds \eqref{eq:wassdt} and \eqref{eq:ct_bd} yields
\begin{gather}
\begin{aligned}
\begin{split}
\label{eq:wasstotal}
    \wass(\dtlaw,\gibbs) &\leq k\step\sqrt{\step}\left[6\sqrt{12C_0 + 3} + 3\sqrt{2} + 4\sqrt{\frac{3}{2} + C_1 }{\color{black}\left(\sqrt{2\beta C_5 \lipGJ^2 C_2} + 2\sqrt{2\beta C_5(\lipJ^2 + \gradnvar)} \right)}\right] \\
    &\quad \quad + \sqrt{2\LSconst \relentbdw} e^{-{k\step}/\beta \LSconst}
\end{split}
\end{aligned}\raisetag{-3\baselineskip}
\end{gather}

The strategy to control \eqref{eq:wasstotal} is to take $k\step$ large enough so that the exponential term dies away, then (fixing $k\step$) take $\step$ small enough so that the first term decreases arbitrarily. However, we encounter a subtle problem: the term $\relentbdw$ may depend inconveniently on $\idw$, and thus on $\step$, since we take $\idw$ satisfying \eqref{eq:omegaspec} in order to obtain Lemmas~\ref{lem:MSEbd},~\ref{lem:KL_bd}.

Let us investigate this. Lemma~\ref{lem:relentropy}, with
$\idw \leq 1$ and $\idistmaxw$ expanded, gives
\begin{align*}
    \relentbdw 
    & \leq \log(\idistmax) + \log\frac{1}{\idw^{\dimn}} +  \frac{\dimn}{2} \log\frac{3\pi}{\dissm\beta} + \frac{\beta \dissb}{2}\log3 + \beta\left(\frac{\lipGJ}{3}\kn + \constB\sqrt{\kn} + \constA\right)
\end{align*}
so we see that $\relentbdw$ depends on $\idw$ as $\dimn\log\left(\frac{1}{\idw}\right)$. Observe that taking $k\step$ as \eqref{eq:itspec} and $\step$ as \eqref{eq:stepspec} yields 
\begin{gather}
\begin{aligned}
\begin{split}
    \wass(\dtlaw,\gibbs) \leq &\,\delta\left[6\sqrt{12C_0 + 3} + 3\sqrt{2} + 4\sqrt{\frac{3}{2} + C_1 }{\color{black}\left(\sqrt{2\beta C_5 \lipGJ^2 C_2} + 2\sqrt{2\beta C_5(\lipJ^2 + \gradnvar)} \right)}\right] \\
    &\quad \quad+ \delta\,\sqrt{2\LSconst \dimn\log\left(\frac{1}{\idw}\right)} + \delta\, \sqrt{2\LSconst C_3}
\end{split}
\end{aligned}\raisetag{-3\baselineskip}
\end{gather}
where $C_3 := \log(\bar{\pi}) + \frac{\dimn}{2} \log\frac{3\pi}{\dissm\beta} + \frac{\beta \dissb}{2}\log3 + \beta\left(\frac{\lipGJ}{3}\kn + \constB\sqrt{\kn} + \constA\right)$.
Then, since $\idw \in [ \step^2, \step^{3/2}]$ and $\step \leq \left(\frac{\delta}{\log(1/\delta)} \right)^2$ we have 
 \[\log(\frac{1}{\idw}) \leq \log\left(\frac{1}{\step^2}\right) \leq \log\left(\frac{\log(1/\delta)}{\delta}\right)^4 \leq \log\left(\frac{1}{ \delta^5}\right) = 5\log\left(1/\delta \right)\]
 where we use that $\left(\frac{\log(1/\delta)}{\delta}\right)^4 \leq \delta^{-5} $ for all $\delta\leq 1$, satisfied by the feasible $\delta$ range \eqref{eq:delmax}. 
Then we obtain the bound displayed in Theorem~\ref{thm:main1}. 


\section{Conclusion}
\label{sec:conc}
We derived non-asymptotic (finite-sample) bounds for a passive stochastic gradient Langevin dynamics algorithm. These results complement recent asymptotic weak convergence analysis of the passive Langevin algorithm in~\cite{krishnamurthy2021langevin}. The passive Langevin algorithm analyzed in this paper uses sequential evaluations of a stochastic gradient descent by an external agent (forward learner), and reconstructs the cost function being optimized. Thus it achieves \textit{real-time} (adaptive) inverse reinforcement learning, in that we (the inverse learner) reconstruct the cost function while it is in the process of being optimized. Specifically, we have provided finite-sample bounds on the 2-Wasserstein distance between the sample distribution induced by our algorithm and the Gibbs measure encoding the cost function to be reconstructed. {\color{black}We have also provided a kernel density estimation algorithm for reconstructing the cost function from empirical samples, and an associated concentration bound on this reconstruction.} Our paper builds on~\cite{raginsky2017non} and utilizes techniques in the analysis of Markov Diffusion Operators \cite{bakry2014analysis} to achieve the bound.

In \cite{krishnamurthy2021langevin} a multi-kernel passive stochastic gradient Langevin dynamics is presented which achieves variance reduction in high dimensions. this algorithm can also handle the scenario when the forward learner's initial sample distribution is unknown; it thus generalizes the framework considered in this paper. An interesting line of future work under consideration is the non-asymptotic analysis of this multi-kernel passive algorithm. 

\
\bibliographystyle{siamplain}
\bibliography{Bibliography.bib}

\appendix

\subsection{Bound Constants}
\label{ap:bd_consts}
The following bound constants appear primarily in the main bound, given as Theorem~\ref{thm:main1}.
\begin{align*}
    \begin{split}
        &C_0 := 3\lipGJ^2(\etsup + 2B^2\etsup) + B^2 + \gradnvar \\
        &C_1 := \kn + (\beta \dissb + \dimn)2\step + 2 \idiffcc \\
        &C_2 := \beta \lipGJ^2\left(72C_0 + 6\sqrt{C_0} + 18+\sqrt{2}\right) \\
        &C_3 := \log(\bar{\pi}) + \frac{\dimn}{2} \log\frac{3\pi}{\dissm\beta} + \frac{\beta \dissb}{2}\log3 + \beta\left(\frac{\lipGJ}{3}\kn + \constB\sqrt{\kn} + \constA\right)\\
        &C_4 := \left[6\sqrt{12C_0 + 3} + 3\sqrt{2} + 4\left(\frac{3}{2} + C_1 \right)^{1/2}{\color{black}\left(\sqrt{2\beta C_5 \lipGJ^2 C_2} + 2\sqrt{2\beta C_5(\lipJ^2 + \gradnvar)} \right)}\right]\\
        &\etsup = \kn + 2\left(1 \vee \frac{1}{\dissm}\right)\left(\dissb + 2\constB^2\right) \\
        &{\color{black}C_5 := \left[\frac{1/\sgdtm}{\hat{\mu}_{sgd}} + \frac{1}{\sgdtm^2} \right]}
    \end{split}
\end{align*}

\subsection{Technical Results}
\label{subsec:techres}
Here we present several technical Lemmas which are necessary for the results derived in Section~\ref{sec:resultspf}.

\label{ap:techres}

The proofs for all of these can be found in Section~\ref{ap:techlempfs}.We denote $\idistmax := \sup_{x}\idistr(x)$ and $\idistmaxw := \sup_{x}\idistrw(x)$. $A = \|\rew(0)\|, B = \|\grad\rew(0)\|$, and $\idiffc, \idiffcc$ are constants provided in Lemma~\ref{lem:lem11}.

{\color{black}
\begin{lemma}
\label{lem:gendiffstat}
The generalized Langevin diffusion \eqref{eq:stram_diff} has \eqref{eq:gibbsd} as its stationary measure. 
\end{lemma}
}

\begin{lemma}[$\idistrw$ exponential integrability]
\label{lem:idistr}
 For all $\idw \leq 1$, $\idistrw$ has a bounded and strictly positive density with respect to the Lebesgue measure on $\reals^{\dimn}$, and
\begin{equation}
\label{eq:kn}
 \knw := \log \int_{\reals^{\dimn}}e^{\|x\|^2}d\idistrw(x) < \infty
 \end{equation}
 and denote $\kn := \knw |_{\idw=1}$ so that $\knw \leq \kn \ \forall \idw \leq 1$.
\end{lemma}

\begin{lemma}[relative entropy bound]
\label{lem:relentropy}
\begin{align}
    \label{eq:relentropy}
    \begin{split}
       \relentbdw := \KL(\idistrw || \gibbs) \leq &\log\idistmaxw + \frac{\dimn}{2} \log\frac{3\pi}{\dissm\beta} + \frac{\beta \dissb}{2}\log3 \\&\quad + \beta\left(\frac{\lipGJ}{3}\knw + \constB\sqrt{\knw} + \constA\right)
    \end{split}
\end{align}
\end{lemma}

\begin{lemma}[exponential integrability of Langevin diffusion]
\label{lem:exp_int}
\[\log\CE[e^{\|\alt\|^2]}] \leq \knw + ((\beta \dissb + \dimn)2\step + 2 \idiffcc) t \]
where $\kn$ is given in \eqref{eq:kn} and $\idiffcc$ is given in \eqref{eq:lem11}.
\end{lemma}

\begin{lemma}[uniform $L^2$ bound on SGD]
\label{lem:unifL2}
For $\sgdstep \in (0, 1 \wedge \frac{\dissm}{4\lipGJ^2})$, and $\thk$ from to process~\eqref{eq:sgd}, 
\begin{equation}
\sup_{k\geq 0}\CE\|\thk\|^2 \leq \kn + 2\left(1 \vee \frac{1}{\dissm}\right)\left(\dissb + 2\constB^2\right) =: \etsup
\end{equation}
\end{lemma}

\begin{lemma}[$L^2$ bound on Langevin diffusion]
\label{lem:diffbound}
\[\CE \|\alt\|^2 \leq \knw + \frac{(\beta \dissb + \dimn)\idistmaxw + 2\idiffc}{(\dissm\beta)\idistmaxw}\]
\end{lemma}

\begin{lemma}
\label{lem:Kbound} Taking 
    \begin{equation}
    \label{eq:delch}
       \Delta \leq \inf_{x \in [\step,\ksup_{\step}]}\frac{K^{-1}(\frac{\ksup_1\sqrt{2\pi}}{2\step}e^{x^2/2})}{K^{-2}(x\step^{2\dimn})}
    \end{equation}
    gives 
    \[\CE \|\Kd(\thk,\alk)\gradn \rew(\thk)\|^2\leq  12\step(\lipGJ^2(\etsup + 2B^2\etsup) + B^2 + \gradnvar)\]
\end{lemma}

\begin{lemma}
\label{lem:wcond}
Taking $\idistr$ such that $\idistmax = 1$, and 
\begin{align}
\begin{split}
\label{eq:wcond}
\idw \leq\step^{3/2}
\end{split}
\end{align}
gives
\[\CE\|\idistrw(\alk)\|^2 \leq \step , \quad \CE\|\grad\idistrw(\alk)\|^2 \leq \step \]
\end{lemma}

\begin{lemma}
\label{lem:Esplit}
    Let $R$ be an $\dimn$-dimensional random variable on the same probability space as $\alsbar$. Then 
    \[\CE \|\idistrw(\alsbar)R\|^2 \leq 2\CE \left[|\idistrw(\alsbar)|^2\right] \CE \left[\|R\|^2\right]\]
\end{lemma}

\begin{lemma}[$\idistrw$ Quadratic Decay]
\label{lem:taildecay}
$\idistrw$ has tail value decay $\mathcal{O}(\|x\|^{-2})$ and differential decay $\mathcal{O}(\|x\|^{-1})$. Specifically,  
\[\exists\, \tailbound \in \N, \, \idiffconst > 0 : \, \idistrw(x) \leq \frac{2}{(\beta m^*)^2\|x\|^2}, \ \, \|\grad\idistrw(x)\| \leq \frac{\idiffconst}{\|x\|}\ \ \forall \|x\| > \tailbound\]
\end{lemma}

\begin{lemma}[quadratic bounds on $\rew$]
\label{lem:quad_bd}
For all $w \in \reals^d$,
\[\|\grad \rew(w)\| \leq \lipGJ \|w\| + \constB \]
and
\[\frac{\dissm}{3}\|w\|^2 - \frac{\dissb}{2} \leq \rew(w) \leq \frac{\lipGJ}{2}\|w\|^2 + \constB\|w\| + \constA\]
\end{lemma}

\begin{lemma}[Uniform Gradient Inner Product Bound] 
\label{lem:lem11}
We have
    \begin{align}
    \begin{split}
    \label{eq:lem11}
        &\exists \idiffc >0 : \|\langle x, \grad\idistrw(x) \rangle\| \leq \idiffc \ \forall x \in \reals^{\dimn}, \idw \in \reals_+ \\
        &\exists \idiffcc >0 : \|\langle x, \idistrw(x)\grad\idistrw(x) \rangle\| \leq \idiffcc \ \forall x \in \reals^{\dimn}, \idw \in \reals_+
    \end{split}
    \end{align}
\end{lemma}

\color{black}{\begin{lemma}
\label{lem:sgdvarbd}
The quantity $\sgdvarbd := \min_{c\in\reals^{\dimn}}\int_{\reals^d} \left(\grad\rew(\eta \gradn\rew(c)) - \grad\rew(\eta\grad\rew(c))\right)d\mu_{sgd}$, with $\mu_{sgd}$ defined in \ref{ass:gradnoise}, is strictly positive.
\end{lemma}}

\color{black}{\begin{lemma}[SGD Gradient Variance Lower Bound] 
\label{lem:sgdvarbd}
Given Algorithm~\ref{alg:sgd}, we have 
\[\textrm{Var}(\grad\rew(\theta_{k+1}) | \thk) \geq \hat{\mu}_{sgd}\]
\end{lemma}}

\textcolor{black}{\begin{lemma}[2-Wasserstein Subset Bound]
\label{lem:wsubset}
Consider two probability densities $\mu,\nu$ on $\reals^n$. Then, consider a compact subset $\Theta \subset \reals^n$, and construct
\begin{equation}
\label{eq:wsbo}
\hat{\mu}(x) = \begin{cases} \mu(x) / Z_{\mu}, \,&x \in \Theta \\ 0, &x\notin \Theta \end{cases}, \quad \hat{\nu}(x) = \begin{cases} \nu(x) / Z_{\nu}, \,&x \in \Theta \\ 0, &x\notin \Theta \end{cases} \quad Z_{\mu} = \int_{\Theta} \mu(x)dx,\,\,Z_{\nu} = \int_{\Theta}\nu(x)dx
\end{equation}
Here $Z_{\mu}$ and $Z_{\nu}$ are normalizing constants such that $\hat{\mu}$ and $\hat{\nu}$ represent the densities conditioned on being restricted to $\Theta$. Choose (refer to definition \eqref{def:wass}) $\hat{\gamma} \in \Gamma(\mu,\nu)$ as 
\begin{equation}
    \label{eq:wsboo}
    \arg\inf_{\gamma \in \Gamma(\mu,\nu)}\left(\int_{\reals^n \times\reals^n} \|x-y\|^2d\gamma(x,y) \right)^{1/2}
\end{equation} Then, \eqref{eq:wsbo}, \eqref{eq:wsboo} imply that 
\[\wass(\hat{\mu},\hat{\nu}) \leq \sqrt{2\zconst^{-1}}\wass(\mu,\nu), \quad \zconst = \int_{\Theta \times \Theta} d\hat{\gamma}(x,y)\]
where $\zconst = \int_{\Theta \times \Theta}d\hat{\gamma}(x,y)$ is the volume of optimal transport measure $\hat{\gamma}$ in region $\Theta \times \Theta$.
\end{lemma}
\begin{proof}
Begin by expanding the definition of 2-Wasserstein distance 
\begin{align*}
\wass(\mu,\nu) &= \inf_{\gamma \in \Gamma(\mu,\nu)}\left( \int_{\Theta \times \Theta}\|x-y\|^2d\gamma(x,y) + \int_{(\Theta \times \Theta)^C}\|x-y\|^2d\gamma(x,y)\right)^{1/2}\\ 
&\geq \inf_{\gamma \in \Gamma(\mu,\nu)}\frac{1}{\sqrt{2}}\left(\sqrt{\int_{\Theta \times \Theta}\|x-y\|^2d\gamma(x,y)} + \sqrt{\int_{(\Theta \times \Theta)^C}\|x-y\|^2d\gamma(x,y)} \right) \\
&\geq \frac{1}{\sqrt{2}}\left(\inf_{\gamma \in \Gamma(\mu,\nu)}\sqrt{\int_{\Theta \times \Theta}\| x-y\|^2d\gamma(x,y)} + \sqrt{\int_{(\Theta \times \Theta)^C}\|x-y\|^2d\tilde{\gamma}(x,y)}\right) \\
&\geq \frac{1}{\sqrt{2}}\left(\inf_{\gamma \in \Gamma(\hat{\mu},\hat{\nu})}\sqrt{\int_{\Theta \times \Theta}\zconst\|x-y\|^2d\gamma(x,y)} + \sqrt{\int_{(\Theta \times \Theta)^C}\|x-y\|^2d\tilde{\gamma}(x,y)}\right)
\end{align*}
where $\tilde{\gamma}$ is the $\arg\inf$ of the second line infimum. 
Thus, we have $\wass(\hat{\mu},\hat{\nu}) \leq \sqrt{2\zconst^{-1}}\wass(\mu,\nu)$
\end{proof}}

{\color{black}
\begin{lemma}[Optimal Transport Measure Subset Bound]
\label{lem:otsubset}
Consider probability densities $\mu$, $\nu$ on $\reals^{\dimn}$ and compact subsets $\Theta',\Theta \subset \reals^{\dimn}$. Denote $\del\Omega$ the boundary of set $\Omega \subset \reals^{\dimn}$. Suppose the following properties, with $\alpha,\kappa > 0, \,\Delta > \sqrt{\frac{\kappa}{\alpha}}$, are satisfied:
\begin{enumerate}[label=\roman*)]
    \item $\int_{\Theta'} \mu(x)dx \geq \alpha$
    \item $\wass(\mu,\nu) \leq \kappa$
    \item $\Theta' \subset \Theta \subset \reals^{\dimn}$,  $\inf_{x\in\del\Theta,y\in\del\Theta'}\|x-y\| \geq \Delta$
\end{enumerate}
Let $\hat{\gamma}$ denote the optimal transport measure coupling $\mu$ and $\nu$ \eqref{def:wass},
\[\hat{\gamma} \in \arg\inf_{\gamma\in\Gamma(\mu,\nu)}\left(\int_{\reals^{\dimn}\times\reals^{\dimn}}\|x-y\|^2d\gamma(x,y)\right)^{1/2}\]
Then, we have that $\int_{\Theta\times\Theta}d\hat{\gamma}(x,y) \geq (\alpha - \kappa\Delta^{-2})^2$
\end{lemma}

\begin{proof}

Let $V$ be the total density transported from $\mu$ in $\Theta'$ to $\nu$ in $\reals^{\dimn}/\Theta$ under transportation map $\hat{\gamma}$. It is well known, under the Monge optimal transport reformulation, that the non-degenerate joint measure $\hat{\gamma}$ is equivalent to a map $T: \reals^{\dimn} \to \reals^{\dimn}$, such that $\int_{\reals^{\dimn}\times\reals^{\dimn}}\|x-y\|^2d\hat{\gamma}(x,y) = \int_{\reals^{\dimn}}\|x-T(x)\|^2\mu(x)d(x), \, T_*(\mu) = \nu$, where $T_*(\mu)$ is the push-forward of $\mu$. It can be seen that, given $(iii)$, $V$ is maximized under a geometry where $\inf_{y\in\del\Theta}\|x-y\| = \Delta\,\forall x\in \del\Theta'$, a transport map $T$ such that $T(x) = \inf_{y\in\del\Theta}\|x-y\| \, \forall x\in \Omega \subseteq \Theta', \,T(x) = x \, \forall c\in\reals^{\dimn}/\Omega$. Here $\Omega$ is taken necessarily small enough so that $\int_{\reals^{\dimn}}\|x-T(x)\|^2\mu(x)d(x) = \int_{\Omega}\|x-T(x)\|^2\mu(x)d(x) \leq \kappa^2$ by $(ii)$. Then, for any $x\in\Omega$, we have $\|x-T(x)\|\geq \Delta$ and $V = \int_{\Omega}\mu(x)d(x) \leq \kappa\Delta^{-2}$. Thus, letting $\zeta = \arg\min_{\beta\in\{\mu,\nu\}}\int_{\Theta}d\beta(x)$, we have $\int_{\Theta \times \Theta}d\hat{\gamma}(x,y) \geq \left(\int_{\Theta}d\zeta(x)\right)^2 \geq (\alpha - \kappa\Delta^{-2})^2$ by the previous reasoning and $(i)$. 
\end{proof}}

{\color{black}The consequence of this lemma is that it enables us to enact a lower bound on $\int_{\Theta \times \Theta} d\hat{\gamma}(x,y)$, which is necessary for control of the concentration inequality \eqref{eq:phidef}.}

\subsection{Proofs of Technical Lemmas}
\label{ap:techlempfs}

{\color{black}
\subsection*{Proof of Lemma \ref{lem:gendiffstat}}
\begin{proof}
The proof is outlined in \cite{krishnamurthy2021langevin} at the bottom of Page 10. Recall \cite{karatzas1991brownian} that for a generic diffusion process denoted by $dx(t) = f(x)dt + \sigma(x)dW(t)$, the stationary distribution $p$ satisfies 
\begin{equation}
\label{eq:statsuff}
    \mathcal{L}^*p = \frac{1}{2}\textrm{Tr}[\grad^2(\Sigma p)] - \textrm{div}(fp) = 0, \quad \textrm{where } \Sigma = \sigma\sigma'
\end{equation}
and $\mathcal{L}^*$ is the forward operator. From \eqref{eq:stram_diff}, $f(\alpha) = [-\frac{\beta}{2}\pi(\alpha)\grad_{\alpha}\rew(\alpha) + \grad_{\alpha}\pi(\alpha)]\pi(\alpha), \, \sigma = \pi(\alpha)I$.  Then, with $p(\alpha) \propto \exp(-\beta\rew(\alpha))$, 
\[\frac{1}{2}\textrm{Tr}[\grad^2(\Sigma p)] = \sum_i\grad\frac{1}{P}[ -\pi^2(\alpha)\frac{\beta}{2}\exp(-\beta\rew(\alpha))(\grad_{\alpha}\rew(\alpha))_i + \pi(\alpha)(\grad_{\alpha}\pi(\alpha))_i\exp(-\beta\rew(\alpha))] = \textrm{div}(fp)\]
where $P$ is the normalizing factor for $p(\alpha)$ and $(x)_i$ refers to the $i'th$ component of vector $x\in\reals^{\dimn}$. 
\end{proof}}

\subsection*{Proof of Lemma \ref{lem:idistr}}
\begin{proof}
    This follows from Assumption~\ref{ass:idistexp}.
\end{proof}

\subsection*{Proof of Lemma \ref{lem:relentropy}}
\begin{proof}

Recall $\gibbs(w) := \frac{1}{\gibbsnorm} \exp(-\beta \rew(w))$, where $\gibbsnorm = \int_{\reals^{\dimn}}\exp(-\beta\rew(w))dw$. Since $\gibbs > 0$ everywhere, we can write

\begin{align*}
\KL(\idistrw || \gibbs) &= \int_{\reals^{\dimn}}\idistrw(x) \log\left( \frac{\idistrw(x)}{\gibbs(x)}\right)dx \\
& = \int_{\reals^{\dimn}}\idistrw(x)\log\idistrw(x)dx + \log\Lambda + \beta\int_{\reals^{\dimn}}\idistrw(x)\rew(x) dx \\
& \leq \log\|\idistrw\|_{\infty} + \log\Lambda + \beta\int_{\reals^{\dimn}}\idistrw(x)\rew(x) dx \\
& \leq \log\idistmaxw + \log\Lambda + \beta\int_{\reals^{\dimn}}\idistrw(x)\rew(x) dx
\end{align*}
First let us upper bound the normalization constant:
\begin{align*}
\Lambda &= \int_{\reals^{\dimn}}e^{-\beta \rew(x)}dx \\
&\leq e^{\frac{1}{2}\beta \dissb\log3}\int_{\reals^{\dimn}}e^{-\frac{\dissm\beta\|x\|^2}{3}} dx \\
& = 3^{\beta \dissb / 2}\left(\frac{3\pi}{\dissm\beta} \right)^{\dimn/2}
\end{align*}
where the inequality follows from Lemma \ref{lem:quad_bd}. Thus,
\[\log \Lambda \leq \frac{\dimn}{2} \log\frac{3\pi}{\dissm\beta} + \frac{\beta \dissb}{2}\log3\]
By Lemma \ref{lem:quad_bd} we also have 
\begin{align*}
\begin{split}
    \int_{\reals^{\dimn}}\rew(x)\idistrw(x)dx &\leq \int_{\reals^{\dimn}}\idistrw{dx}\left(\frac{\lipGJ}{3}\|x\|^2 + \constB\|x\| + \constA \right) \\
    & \leq \frac{\lipGJ}{3}\knw + \constB\sqrt{\knw} + \constA
\end{split}
\end{align*}
Thus
\[\KL(\idistrw || \gibbs) \leq \log\idistmaxw + \frac{\dimn}{2} \log\frac{3\pi}{\dissm\beta} + \frac{\beta \dissb}{2}\log3 + \beta\left(\frac{\lipGJ}{3}\knw + \constB\sqrt{\knw} + \constA\right) \]

\end{proof}

\subsection*{Proof of Lemma \ref{lem:exp_int}}
    
\begin{proof}
Let us denote $\alts := \alt$ for notational convenience, and define $L(t) := e^{\|\alts\|^2}$. Similarly denote $\Lt = L(t)$. By It\^o's Lemma we have 
\[d\Lt = \{(\grad_{\alts}\Lt)^T\boldsymbol{\mu}_t + \frac{1}{2}\textrm{Tr}[\boldsymbol{G}_t^T(H_{\alts}\Lt)\boldsymbol{G}_t]\}dt + (\grad_{\alts}\Lt)^T\boldsymbol{G}_td\brownt)\]
where from \eqref{eq:ct_diff} we have 
\[\boldsymbol{\mu}_t = -\frac{\beta}{2}\idistrw^2(\alts)\grad \rew(\alts) \,\color{black}{+}\, \idistrw(\alts)\grad \idistrw(\alts), \ \boldsymbol{G}_t = \idistrw(\alts)\]
Thus,
\begin{align*}
    \begin{split}
        &(\grad_{\alts}\Lt)^T\boldsymbol{\mu}_t = - \beta \langle \alts \Lt ,\idistrw^2(\alts)\grad \rew(\alts) \rangle \,\color{black}{+}\, 2 \langle \alts \Lt , \idistrw(\alts)\grad \idistrw(\alts) \rangle
    \end{split}
\end{align*}
and 
\[\frac{1}{2}\textrm{Tr}[\boldsymbol{G}_t^T(H_{\alts}\Lt)\boldsymbol{G}_t] = \frac{1}{2}\textrm{Tr}[\idistrw^2(\alts)H_{\alts}\Lt] = \idistrw^2(\alts)(\|\alts\|^2\Lt + \dimn \Lt)\]
and
\[(\grad_{\alts}\Lt)^T\boldsymbol{G}_t = 2\alts^{*}\idistrw(\alts)\Lt\]
Putting these together and integrating,
\begin{align*}
    \begin{split}
        \Lt = \ &L(0) - \beta \int_0^t \langle \altss \Lss ,\idistrw^2(\altss)\grad \rew(\altss) \rangle ds \,\color{black}{+}\, 2 \int_0^t \langle  \altss \Lss , \idistrw(\altss)\grad \idistrw(\altss) \rangle ds \\
        & + \int_0^t \idistrw^2(\altss)(\|\altss\|^2\Lss + \dimn \Lss) ds + \int_0^t 2\altss^{*}\idistrw(\altss)\Lss d\browns \\
        & = L(0) + \int_0^t (\idistrw^2(\altss)\|\altss\|^2 - \beta \langle \altss,\idistrw^2(\altss)\grad \rew(\altss) \rangle) \Lss ds \\
        &\,\color{black}{+}\, 2 \int_0^t \langle  \altss \Lss , \idistrw(\altss)\grad \idistrw(\altss) \rangle ds + \int_0^t \dimn \Lss \idistrw^2(\altss) ds \\&\quad + \int_0^t 2\altss^{*}\idistrw(\altss)\Lss d\browns
    \end{split}
\end{align*}

Now, from the dissipativity condition \ref{as:diss}, we can obtain the following bound:
\begin{align*}
    \begin{split}
        &\idistrw^2(\altss)(\|\altss\|^2 - \beta\langle \altss,\grad \rew(\altss) \rangle)  \leq \idistrw^2(\altss)(\|\altss\|^2 + \beta[-\dissm\|\altss\|^2 + \dissb]) \\
        & = \|\altss\|^2(\idistrw^2(\altss) - \beta \dissm) + \idistrw^2(\altss)\beta \dissb \leq \idistrw^2(\altss)\beta \dissb 
    \end{split}
\end{align*} 

Making this substitution, we now work with
\begin{align*}
    \begin{split}
        \Lt \leq \ &L(0) + (\beta \dissb + \dimn) \int_0^t \idistrw^2(\als) \Lss ds \,\color{black}{+}\, 2 \int_0^t \langle  \altss \Lss , \idistrw(\altss)\grad \idistrw(\altss) \rangle ds \\& + \int_0^t 2\altss^{*}\idistrw(\altss)\Lss d\browns
    \end{split}
\end{align*}
It can be shown (e.g., proof of Corollary 4.1 in \cite{djellout2004transportation}) that $\int_0^T\CE\|\Lt\alt\|^2dt < \infty \ \forall T \geq 0$. Therefore the It\^o integral $\int \Lss \altss^*d\browns$ is a zero-mean martingale. Thus, taking expectations leaves us with

\begin{align*}
    \begin{split}
        \CE[\Lt] &\leq \ \CE[L(0)] + \CE[(\beta \dissb + \dimn)\int_0^t \idistrw^2(\als) \Lss ds] \\&\quad \quad \,\color{black}{+}\, \CE[2 \int_0^t \langle  \altss, \idistrw(\altss)\grad \idistrw(\altss) \rangle \Lss ds] \\
        &= \CE[L(0)] + (\beta \dissb + \dimn)\int_0^t \CE[\idistrw^2(\als) \Lss]ds \\&\quad \quad \,\color{black}{+}\, 2\int_0^t \CE[\langle  \altss, \idistrw(\altss) \grad \idistrw(\altss) \rangle\Lss] ds \\
        &\leq \CE[L(0)] + (\beta \dissb + \dimn)2  \step \int_0^t \CE[\Lss]ds +
        2 \idiffcc \int_0^t \CE[\Lss] ds \\
        & = \CE[L(0)] + ((\beta \dissb + \dimn)2\step + 2\idiffcc)\int_0^t \CE[\Lss]ds  \\
        &= e^{\knw} + ((\beta \dissb + \dimn)2\step + 2 \idiffcc) \int_0^t \CE[\Lss]ds
    \end{split}
\end{align*}
By application of the Gronwall Inequality, we obtain
\begin{align*}
    \CE[\Lt] &\leq \textrm{exp}(\knw + \int_0^t ((\beta \dissb + \dimn)2\step + 2 \idiffcc)  ds \\&=  \textrm{exp}(\knw + ((\beta \dissb + \dimn)2\step + 2 \idiffcc)  t)
\end{align*}
Thus,
\[\log\CE[e^{\|\alt\|^2]}] \leq \knw + ((\beta \dissb + \dimn)2\step + 2 \idiffcc)  t \]
\end{proof}

\subsection*{Proof of Lemma \ref{lem:unifL2}}
\begin{proof}
    See Lemma 3 of \cite{raginsky2017non}, with $\beta = \infty$.
\end{proof}

\subsection*{Proof of Lemma \ref{lem:diffbound}}
\begin{proof}
We consider the diffusion given by \eqref{eq:ct_diff}. Letting $Y(t) = \|\alt\|^2$, It\^o's Lemma gives 
\begin{align*}
    dY(t) &= \biggl[-2\langle \alt,\frac{\beta}{2}\idistrw^2(\alt)\grad \rew(\alt) \,\color{black}{-}\,  \idistrw(\alt)\grad\idistrw(\alt)\rangle + N\idistrw^2(\alt)\biggr]dt \\&\quad + \idistrw(\alt)\alt^*d\brownt
\end{align*}
where $\alt^*d\brownt := \sum_{i=1}^{\dimn}\alpha_i(t)d\brown_i(t)$. Letting $m:=\frac{\dissm\beta}{2}\idistmaxw^2$, we then form 
\begin{align*}
    d(e^{2mt}Y(t)) &= 2me^{2mt}Y(t) + e^{2mt}dY(t) \\
    &= \biggl[-2e^{2mt}\langle \alt,\frac{\beta}{2}\idistrw^2(\alt)\grad \rew(\alt) \,\color{black}{-}\, \idistrw(\alt)\grad\idistrw(\alt)\rangle \\&\quad + \dimn\idistrw^2(\alt)e^{2mt}  + 2me^{2mt}Y(t)\biggr] dt \\&\quad + e^{2mt}\idistrw(\alt)\alt^*d\brownt
\end{align*}
Then integrating yields
\begin{align*}
    Y(t) &= e^{-2mt}Y(0) - 2\int_0^t e^{2m(s-t)} \langle \als,\frac{\beta}{2}\idistrw^2(\als)\grad \rew(\als) \\&\quad \,\color{black}{-}\, \idistrw(\als)\grad\idistrw(\als)\rangle ds \\
    & \quad + 2m\int_0^t e^{2m(s-t)}Y(s)ds + \int_0^t\dimn\idistrw^2(\als)e^{2m(s-t)}ds \\&\quad+ 2\int_0^t
e^{2m(s-t)}\idistrw(\als)\als^*d\brown(s)
\end{align*}
Then using the dissipativity condition \color{black}{in Assumption~\ref{as:diss}} we get
\begin{align*}
Y(t) &\leq e^{-2mt}Y(0) + \beta\int_0^t\idistrw^2(\als) e^{2m(s-t)}(\dissb - \dissm Y(s))ds \\&\quad+ 2\int_0^t\idistrw(\als)e^{2m(s-t)}\idiffc ds\\
& \quad + \dissm\beta \int_0^t\idistmaxw^2 e^{2m(s-t)}Y(s)ds + \int_0^t\dimn \idistmaxw^2 e^{2m(s-t)}ds \\
&\quad + 2\int_0^t e^{2m(s-t)} \idistrw(\als)\als^* d\browns \\
&\leq e^{-2mt} Y(0) + \beta \dissb \int_0^t \idistmaxw^2 e^{2m(s-t)}ds + 2\idiffc \int_0^t\idistmaxw e^{2m(s-t)}ds \\&\quad+ \int_0^t \dimn \idistmaxw^2 e^{2m(s-t)} ds
\end{align*}
Then grouping terms and evaluating the integral yields
\begin{align*}
    Y(t) &\leq e^{-2mt}Y(0) + \frac{(\beta \dissb + \dimn)\idistmaxw^2 + 2\idiffc\idistmaxw}{2m}\left(1-e^{-2mt} \right) \\&\quad \quad + \idistmaxw\int_0^te^{2m(s-t)}\als^*d\browns 
\end{align*}
Now taking expectations, and by the Martingale property of the It\^o integral, we have
\begin{align*}
    \CE[\|\alt\|^2] &\leq e^{-2mt}\CE\|\alpha(0)\|^2 + \frac{(\beta \dissb + \dimn)\idistmaxw^2 + 2\idiffc\idistmaxw}{2m} \left(1-e^{-2mt} \right) \\
    &\leq e^{-2mt}\CE\|\alpha(0)\|^2 + \frac{(\beta \dissb + \dimn)\idistmaxw^2 + 2\idiffc\idistmaxw}{2m}
\end{align*}
and from \eqref{eq:kn}, and using $m = \frac{\dissm\beta}{2}\idistmaxw^2$, and taking the maximum over $t > 0$ gives

\[\CE \|\alt\|^2 \leq \knw + \frac{(\beta \dissb + \dimn)\idistmaxw + 2\idiffc}{(\dissm\beta)\idistmaxw}\]
\end{proof}

\subsection*{Proof of Lemma~\ref{lem:Kbound}}
\begin{proof}
\begin{align*}
\begin{split}
&\CE \|\Kd(\thk,\alk)\gradn \rew(\thk)\|^2 = \CE |\Kd(\thk,\alk)|^2 \|\gradn \rew(\thk)\|^2  \\ & = \CE |\Kd(\thk,\alk)|^2\CE \|\gradn \rew(\thk)\|^2 + \left[\textrm{Cov}\left(|\Kd(\thk,\alk)|^2, \|\gradn \rew(\thk)\|^2\right)\right]
\end{split}
\end{align*}
By Cauchy-Schwarz,
\begin{align*}
    \begin{split}
        &\textrm{Cov}\left(|\Kd(\thk,\alk)|^2, \|\gradn \rew(\thk)\|^2\right)\leq \textrm{Var}\left(|\Kd(\thk,\alk)|^2\right)\textrm{Var}\left(\|\gradn \rew(\thk)\|^2\right) \\
        & \quad \leq \CE \left[|\Kd(\thk,\alk)|^2\right] \CE \left[\|\gradn \rew(\thk)\|^2\right] \leq \CE |\Kd(\thk,\alk)|^2 \CE \|\gradn \rew(\thk)\|^2
    \end{split}
\end{align*}
so
\[\CE \|\Kd(\thk,\alk)\gradn \rew(\thk)\|^2 \leq 2\CE |\Kd(\thk,\alk)|^2 \CE \|\gradn \rew(\thk)\|^2\]
Then we bound
\begin{align*}
    &\CE \left[\|\gradn \rew(\thk)\|^2\right] \\&\leq 3(\CE\|\lipGJ|\thk| + B\|^2 + \gradnvar) \leq 3(\CE \left[\lipGJ^2|\thk|^2\right] + 2\CE\sup_{k\in\N}\left[\lipGJ^2B^2|\thk|^2 \right] + B^2 + \gradnvar) \\
    & \leq 3(\lipGJ^2 (\etsup + 2B^2 \etsup) + B^2 + \gradnvar) =: C_0
\end{align*}
and thus we have 
\begin{align}
\label{eq:esq_bd}
    &\CE \|\Kd(\thk,\alk)\gradn \rew(\thk)\|^2 \leq 6\CE\|\Kd(\thk,\alk)\|^2\left(\lipGJ^2 (\etsup + 2B^2 \etsup) + B^2 + \gradnvar\right)
\end{align}
so we can control this quantity directly by controlling $\CE\|K_{\Delta}(\thk,\alk)\|^2$, which is done as follows:
Notice that
\[\PR\left(\frac{1}{\Delta^{2\dimn}} K^2\left(\frac{|\thk-\alk|}{\Delta}\right) > x \right) = \PR \left(K(\thk,\alk) > K(\Delta K^{2^{-1}}(x\Delta^{2\dimn})) \right)\]
By Markov's Inequality we have 
\[\PR \left(K(\thk,\alk) > K(\Delta K^{2^{-1}}(x\Delta^{2\dimn}))) \right) \leq \frac{\CE\|K(\thk,\alk)\|}{K(\Delta K^{2^{-1}}(x\Delta^{2\dimn})))} \leq \frac{\ksup}{K(\Delta K^{2^{-1}}(x\Delta^{2\dimn})))}\]
Then, for all $x \in [\step,\ksup_{\step}]$, by choosing $\Delta$ as \eqref{eq:delch} we have
\[\PR\left( \frac{1}{\Delta^{2\dimn}}K^2\left(\frac{|\thk-\alk|}{\Delta}\right) > x \right) \leq  \frac{2\step}{\sqrt{2\pi}} e^{-x^2}\]
So now observe 
\begin{align*}
    \begin{split}
        &\CE\|\Kd(\thk,\alk)\|^2 = \CE\left[\frac{1}{\Delta^{2\dimn}} K^2\left(\frac{|\thk-\alk|}{\Delta}\right)\right] \\
        &= \int_0^{\infty}\PR\left(\frac{1}{\Delta^{2\dimn}}K^2\left(\frac{|\thk-\alk|}{\Delta}\right) > x \right)dx \\
        & \quad = \int_0^{\step}\PR\left(\frac{1}{\Delta^{2\dimn}}K^2\left(\frac{|\thk-\alk|}{\Delta}\right) > x \right)dx  + \int_{\step}^{\ksup_{\step}}\PR\left(\frac{1}{\Delta^{2\dimn}}K^2\left(\frac{|\thk-\alk|}{\Delta}\right) > x \right)dx \\
        & \quad \leq \step + \int_{\step}^{\ksup_{\step}} \frac{2\step}{\sqrt{2\pi}} e^{-x^2} dx \leq 2\,\step
    \end{split}
\end{align*}
\end{proof}

\subsection*{Proof of Lemma~\ref{lem:wcond}}
\begin{proof}
\begin{align*}
\begin{split}
    \CE\left[\|\idistrw(\alk)\|^2\right] = \int_0^{\infty}\PR\left( \idistrw^2(\alk) > x\right)dx
\end{split}
\end{align*}
We define the level set \[\levxw := \{y\in\reals^{\dimn} : \idistrw^2(y) > x\}\]
so that \[\PR\left(\idistrw^2(\alk) > x\right) = \PR(\alk \in \levxw)\]
Now split this term as
\begin{align*}
\PR(\alk \in \levxw) &= \PR(\alk \in \levxw | \alkm \in \levxw)\PR(\alkm \in \levxw) \\& \quad + \PR(\alk \in \levxw | \alkm \notin \levxw)\PR(\alkm \notin \levxw) \\
& \leq \PR(\alk \in \levxw | \alkm \in \levxw) + \PR(\alk \in \levxw | \alkm \notin \levxw)
\end{align*}
Now, from \eqref{eq:dt_sgld}, denote
\[\grad_k := \step\biggl[\kernel \frac{\beta}{2} \gradn \rew (\thk) + \grad \idistrw(\alk)\biggr]\idistrw(\alk), \quad  \tilde{w}_k := \sqrt{\step}\idistrw(\alk) w_k\]
so that $\tilde{w}_k \sim \gaus(0,\step\idistrw(\alk)^2)$. Then observe that, given $\tilde{w}_k$ is symmetric with mean zero,
\begin{align*}
    \PR(\alk \in \levxw | \alkm \in \levxw) &= \PR(\alkm - \grad_{k-1} + \tilde{w}_k \in \levxw | \alkm \in \levxw) \\&\leq \PR(\alk \in \levxw | \alkm - \grad_{k-1} \in \levxw)
\end{align*}
Similarly,
\begin{align*}
    \PR(\alk \in \levxw | \alkm \notin \levxw) \leq \PR(\alk \in \levxw | \alkm - \grad_k \in \levxw)
\end{align*}
so that \[\PR(\alk \in \levxw) \leq 2\PR(\alk \in \levxw | \alkm - \grad_{k-1} \in \levxw)\]

Now, let $\wtz := (\step \idistrw^2(z))^{-1}$, define \[\levxwp(z) := \{y\in \reals^{\dimn} : y/z \in \levxw \}\] and notice that 
\begin{align*}
    \begin{split}
        &\PR\left(\alk \in \levxw | \alkm - \grad_{k-1} \in \levxw\right) \\&\quad = \PR\left(\wtk\alk \in \levxwp(\wtk) \,| \,\wtk\alkm - \wtk\grad_{k-1} \in \levxwp(\wtk)\right) \\
        &\quad = \PR\biggl(\wtk(\alkm - \grad_{k-1} + \sqrt{\step} \idistrw(\alkm) w_k) \in \levxwp(\wtk) \, | \\&\quad \quad \quad \quad \quad  \wtk( \alkm - \grad_{k-1}) \in \levxwp(\wtk) \biggr)
    \end{split}
\end{align*}
 Then, since $\wtk \sqrt{\step} \idistrw(\alkm) w_k \sim \gaus(0,(\wtk)^2 \,\step \,\idistrw^2(\alkm))$ we have that,
 \begin{align*}
    \begin{split}
    \label{eq:ballint}
         &\PR\left(\alk \in \levxw | \alkm - \grad_{k-1} \in \levxw\right) \\&\quad \leq \int_{\levxwp(\wtk)}\gaus(\gamma; \hat{c},(\wtk)^2\, \step \,\idistrw^2(\alkm))d\gamma = \int_{\levxwp(\wtk)}\gaus(\gamma; \hat{c},\wtk)d\gamma
    \end{split}
 \end{align*}
 Now, crucially, observe that the volume of $\levxwp(\wtz)$ scales, w.r.t $z$, at the same rate as the variance of $\gaus(\cdot,\hat{c},\wtz)$. Thus we have \[\int_{\levxwp(\zeta^{\step,\idw}_{z_1})}\gaus(\gamma; \hat{c},\zeta^{\step,\idw}_{z_1} )d\gamma = \int_{\levxwp(\zeta^{\step,\idw}_{z_2})}\gaus(\gamma; \hat{c}, \zeta^{\step,\idw}_{z_2})d\gamma \quad \forall \, z_1,z_2 > 0\]
In particular, take $z$ such that $\wtz = (\step \idistrw^2(z))^{-1} = \step^2$. Note that this necessitates $\idistmaxw \geq (\frac{1}{\step})^{3/2}$, which is given from the condition $\idw \leq \idistmax\,\step^{3/2}$. Then we have 
\begin{align*}
    \PR\left(\alk \in \levxw | \alkm - \grad_{k-1} \in \levxw\right) &\leq \int_{\levxwp(\step^2)}\gaus(\gamma; \hat{c},\step^2)d\gamma \\&\leq \int_{\levxwp(\step^2)}\frac{1}{\step\sqrt{2\pi}}d\gamma = \int_{\levxwp(1)}\frac{\step}{\sqrt{2\pi}}d\gamma
\end{align*}
Then  
\begin{align*}
    \begin{split}
    \CE\left[\|\idistrw(\alk)\|^2\right] &= \int_{0}^{\infty}\PR(\idistrw^2(\alk) > x)dx\\
    &\leq 2\int_0^{\infty}\int_{\levxw}\frac{\step}{\sqrt{2\pi}}d\gamma dx  =\sqrt{\frac{2}{\pi}}\step\int_{0}^{\infty}\int_{\reals^{\dimn}}\mathds{1}\{\idistrw(\gamma)>\sqrt{x}\}d\gamma dx \\
    \end{split}
\end{align*}
but observe that, for $\idw \leq 1$,  \[\int_{0}^{\infty}\int_{\reals^{\dimn}}\mathds{1}\{\idistrw(\gamma)>\sqrt{x}\}d\gamma\, dx \leq \int_{0}^{\infty}\int_{\reals^{\dimn}}\mathds{1}\{\idistr(\gamma)>\sqrt{x}\}d\gamma \,dx\]
So now
\begin{align*}
    \begin{split}
         \CE\left[\|\idistrw(\alk)\|^2\right] 
         &\leq \sqrt{\frac{2}{\pi}}\step\int_{0}^{\infty}\int_{\reals^{\dimn}}\mathds{1}\{\idistr(\gamma)>\sqrt{x}\}d\gamma \,dx \\
         &= \sqrt{\frac{2}{\pi}}\step\biggl[\int_{0}^{1}\int_{\reals^{\dimn}}\mathds{1}\{\idistr(\gamma)>\sqrt{x}\}d\gamma \,dx \\&\quad \quad \quad \quad + \int_{1}^{\idistmaxw^2}\int_{\reals^{\dimn}}\mathds{1}\{\idistr(\gamma)>\sqrt{x}\}d\gamma \,dx\biggr] \\&\leq \sqrt{\frac{2}{\pi}}\step\biggl[\int_{0}^{1}\int_{\reals^{\dimn}}\mathds{1}\{\idistr(\gamma)> x\}d\gamma \,dx + \idistmaxw^2\,V(\{\idistr > 1\})\biggr] \\
         &\leq \step\left[1 + V(\{\idistr>1\})\right] =: \step \ibd
    \end{split}
\end{align*}
where $V(\{\idistr>1\})$ is shorthand for $\int_{\reals^{\dimn}}\mathds{1}\{\idistr(\gamma)>1\}d\gamma$.

Define, analagously,  
\[\Gamma_x^{\idw} = \{y \in \reals^{\dimn}  : \|\grad\idistrw(y)\|^2 > x\}\]
By the same procedure as above, we can obtain
\[\PR(\alk \in \Gamma_x^{\idw}) \leq 2\int_{\Gamma_x^{\idw}}\frac{\step}{\sqrt{2\pi}}d\gamma\]
and so
\begin{align*}
    \begin{split}
        \CE\|\grad\idistrw(\alk)\|^2 &= \int_0^{\infty}\PR(\|\grad\idistrw(\alk)\|^2 > x)dx \leq 2\int_0^{\infty}\int_{\Gamma_x^{\idw}}\sqrt{\frac{2}{\pi}}\step d\gamma dx \\
        &\leq  \step\left[ 1 + \idistgmax V(\{\|\grad\idistr\|>1\})\right] =:\step \igbd
    \end{split}
\end{align*}
where $\idistgmax = \sup_{x\in\reals^{\dimn}}\grad\idistr(x)$, $V(\{\|\grad\idistr\|>1\}) = \int_{\reals^{\dimn}}\mathds{1}\{\|\grad\idistr(\gamma)\|>1\}d\gamma$.
Taking $\idistr$ such that $\idistmax = 1$  gives
\[ \CE\left[\|\idistrw(\alk)\|^2\right] \leq \step,  \quad \CE\|\grad\idistrw(\alk)\|^2 \leq \step\]
\end{proof}

\subsection*{Proof of Lemma \ref{lem:Esplit}}
\begin{proof}
First take 
\begin{align*}
\begin{split}
&\CE \|\idistrw(\alsbar)R\|^2 = \CE \left[|\idistrw(\alsbar)|^2\|R\|^2 \right] \\ & = \CE \left[|\idistrw(\alsbar)|^2\right] \CE \left[\|R\|^2\right] + \textrm{Cov}\left(|\idistrw(\alsbar)|^2, \|R\|^2\right)
\end{split}
\end{align*}
By the Cauchy-Schwarz inequality,
\begin{align*}
    \begin{split}
        &\textrm{Cov}\left(|\idistrw(\alsbar)|^2, \|R\|^2\right) \leq \textrm{Var}\left(|\idistrw(\alsbar)|^2\right)\textrm{Var}\left(\|R\|^2\right) \\
        & \quad \leq \CE \left[|\idistrw(\alsbar)|^2\right] \CE \left[\|R\|^2\right]
    \end{split}
\end{align*}
So we have 
\begin{align}
\begin{split}
\label{eq:ttbound}
\CE \|\idistrw(\alsbar)R\|^2 &\leq 2\CE \left[|\idistrw(\alsbar)|^2\right] \CE \left[\|R\|^2\right] \\
\end{split}
\end{align}
\end{proof}

\subsection*{Proof of Lemma \ref{lem:taildecay}}
\begin{proof}
    This follows from Assumption~\ref{ass:idistexp}
\end{proof}

\subsection*{Proof of Lemma \ref{lem:quad_bd}}
\begin{proof}
Lemma \ref{lem:quad_bd} equivalent to Lemma 2 of \cite{raginsky2017non}, and follows from Assumptions A1, A2, and A3. 
\end{proof}

\subsection*{Proof of Lemma \ref{lem:lem11}}
\begin{proof}
First note that for all $x\in \reals^{\dimn}$ we have $\grad\idistrw(x) = \frac{\grad\idistr(x)}{\idw}$, so 
\[\arg\max_{y\in\reals^{\dimn}}\|\langle y,\, \grad \idistrw(y) \rangle\| = \arg\max_{y\in\reals^{\dimn}}\|\langle \idw y,\, \frac{\grad \idistr(y)}{\idw} \rangle\| =  \arg\max_{y\in\reals^{\dimn}}\|\langle y,\, \grad \idistr(y) \rangle\|\]
and
\begin{align*}
    &\arg\max_{y\in\reals^{\dimn}}\|\langle y,\, \idistrw(y)\grad \idistrw(y) \rangle\| = \arg\max_{y\in\reals^{\dimn}}\|\langle \idw y,\, \idistrw(y)\frac{\grad \idistr(y)}{\idw} \rangle\|\\& \quad =  \arg\max_{y\in\reals^{\dimn}}\|\langle y,\, \idistr(y)\grad \idistr(y) \rangle\|
\end{align*}

 Thus \eqref{eq:lem11} is equivalent to:
\begin{align}
    \label{eq:lem11eq}
    &\exists \idiffc >0 : \|\langle x, \grad\idistr(x) \rangle\| \leq \idiffc \ \forall x \in \reals^{\dimn}
\end{align}
\begin{align}
\label{eq:lem11eqq}
    &\exists \idiffcc >0 : \|\langle x, \idistr(x)\grad\idistr(x) \rangle\| \leq \idiffcc \ \forall x \in \reals^{\dimn}
\end{align}
Now we prove by reductio ad absurdum: Suppose \eqref{eq:lem11eq}, \eqref{eq:lem11eqq} do not hold. Then we have:
\begin{align}
\begin{split}
\label{eq:ras}
    &\forall y >0 \ \exists x \in \reals^{\dimn}  : \ \|\langle x, \grad\idistr(x) \rangle\| > y \\
    &\forall y >0 \ \exists x \in \reals^{\dimn}  : \ \|\langle x, \idistr(x)\grad\idistr(x) \rangle\| > y 
\end{split}
\end{align}
Recall $\tailbound$ as defined in Assumption~\ref{ass:idistexp}. Denote \[D_1 = \{x\in \reals^{\dimn} : \|x\| \leq \tailbound \vee 1\}, \,D_2 = D_1^C = \reals^{\dimn}\backslash D_1\] Recall that we assume Lipshitz-continuity of $\idistr(\cdot)$ in (\ref{as:Bd_Der}), with Lipschitz constant $\Dbound$. Thus, $\|\grad\idistr(x)\|$ is bounded by $\Dbound$ for all $x\in\reals^{\dimn}$. In particular notice that since $\|\grad\idistr(x)\|$, $\|\idistr(x)\|$, and $\|x\|$ are bounded for $x \in D_1$, there exists some $M^*$,$M^{**}$ such that \[\|\langle x, \grad\idistr(x) \rangle\| \leq M^* \ \forall x \in D_1\]
\[\|\langle x, \idistr(x)\grad\idistr(x) \rangle\| \leq M^{**} \ \forall x \in D_1\]
Then \eqref{eq:ras} requires both:
\begin{align}
\begin{split}
\label{eq:ras2}
    &\forall y > M^* \ \exists x \in D_2 : \ \|\langle x, \grad \idistr(x) \rangle\| > y \\
    &\forall y > M^{**} \ \exists x \in D_2 : \ \|\langle x, \grad \idistr(x)\grad\idistr(x) \rangle\| > y 
\end{split}
\end{align} 
But by assumption~\ref{lem:taildecay} $\grad \idistr(x)$ decays as $\mathcal{O}(\|x\|^{-1})$ for $x \in D_2$, and in Lemma~\ref{lem:taildecay} $\idistr(x)$ decays as $\mathcal{O}(\|x\|^{-2})$. Specifically, there exists $\idiffconst$ such that \[\|\grad\idistr(x)\| \leq \frac{\idiffconst}{\|x\|^{-1}}, \quad \|\idistr(x)\| \leq \frac{2}{(\beta m^*)^2\|x\|^2},\, \forall x \in D_2\] Thus we have that \[\|\langle x, \grad \idistr(x) \rangle\| \leq \idiffconst \ \forall x \in D_2\]
\[\|\langle x, \idistr(x)\grad \idistr(x) \rangle\| \leq \frac{2\idiffconst}{(\beta m^*)^2} \ \forall x \in D_2\]which contradicts \eqref{eq:ras2} and thus \eqref{eq:ras} is refuted. So \eqref{eq:lem11} holds.
\end{proof}

{\color{black}\subsection*{Proof of Lemma~\ref{lem:sgdvarbd}}
This follows from assumptions \ref{as:diss} and \ref{ass:gradnoise}.}

\subsection{Proofs of Main Body Results}
\label{ap:resultspf}

{\color{black}\subsection*{Proof of Theorem~\ref{thm:kernest}}
\label{pf:kernest}
\begin{proof}
{\color{black}Consider the $\Theta$-restricted distributions $\bar{\pi}_k,\bar{\pi}_{\infty}$, given by
\[\bar{\pi}_k(x) = \begin{cases} \dtlaw(x)/Z_{\dtlaw} \, &x\in \Theta \\ 0, &x\notin \Theta \end{cases}, \quad Z_{\dtlaw} := \int_{\Theta}\dtlaw(x)dx\]
\[\bar{\pi}_{\infty}(x) = \begin{cases} \gibbs(x)/Z_{\gibbs} \, &x\in \Theta \\ 0, &x\notin \Theta \end{cases},\quad Z_{\gibbs} := \int_{\Theta}\gibbs(x)dx,\]
where $Z_{\dtlaw}$ and $Z_{\gibbs}$ are normalizing constants.}
 We begin by showing that the sampling scheme in Algorithm~\ref{alg:costrec} produces $i.i.d.$ samples from a measure $\bar{\pi}$ such that {\color{black}$\wass(\bar{\pi},\bar{\pi}_{\infty}) \leq \sqrt{2\zconst^{-1}}\wassprox$, where $\zconst$ is defined in \eqref{eq:Zint}}. First observe that by our choice of $\delta$, we have 
 \[\delta\left[C_4 + \sqrt{2\LSconst C_3} \right] + \delta\sqrt{10\LSconst\dimn\log\left(1/\delta\right)} \leq \wassprox \]
 and thus, with $\step \leq \left(\frac{\delta}{\log(1/\delta)} \right)^2$ and $\hk$ taken as \eqref{eq:itspec} we have $\wass(\pi_{\hk},\gibbs) \leq \wassprox$ by Theorem~\ref{thm:main1}, and $\wass(\bar{\pi}_{\hk},\bar{\pi}_{\infty}) \leq \sqrt{2\zconst^{-1}}\wassprox$ by Lemma~\ref{lem:wsubset}. 
 
 \textcolor{black}{Now observe (by Lemma~\ref{lem:iid}) that the empirical sample set $S = \cup_{i=1}^{|S|}\alpha_{\hk}^i$ from Algorithm~\ref{alg:costrec} is composed of $i.i.d.$ samples. This, along with Assumption~\ref{as:kernspec} and \eqref{eq:Vf}, allows us to employ Theorem 2 of \cite{vogel2013uniform} to derive:
 \begin{equation}
 \label{eq:cione}
\PR\left(\sup_{\theta\in\reals^{\dimn}}|\hat{\pi}(\theta) - \bar{\pi}_{\hk}(\theta)| \geq \beta_{\kfv,T} \right) \leq \mathcal{H}(\kfv) 
 \end{equation}
where 
\begin{equation*}
    \beta_{\kfv,T} = \frac{\kfv}{\sqrt{|S|}b_S} + \frac{\mk_2}{(2\pi)
^{\dimn}\sqrt{|S|}b_S} + \frac{1}{2}\mk_3\mk_4b_S^{2/\dimn}, \quad \mathcal{H}(\kfv) = 2e^{\frac{-\kfv^2}{2\mk_1^2}}
\end{equation*}
Then, since $b_S = b_T\sqrt{\frac{T}{|S|}}$, we have 
\begin{equation}
\label{eq:bhdef}
    \beta_{\kfv,\snum} = \frac{\kfv}{\sqrt{\snum}b_T} + \frac{\mk_2}{(2\pi)
^{\dimn}\sqrt{\snum}b_T} + \frac{1}{2}\mk_3\mk_4\left(b_T\sqrt{\frac{T}{|S|}}\right)^{2/\dimn}, \quad \mathcal{H}(\kfv) = 2e^{\frac{-\kfv^2}{2\mk_1^2}}
\end{equation}}

Next we utilize a result of \cite{peyre2018comparison} which allows us to bound the $L^1$ distance between $\bar{\pi}_{\hk}(\cdot)$
and $\bar{\pi}_{\infty}$ in terms of the 2-Wasserstein bound already achieved. Specifically, defining the negative-order Sobolev norm
\[\|f\|^2_{\dot{H}^1} = \int_{\reals^{\dimn}}\|\grad f(x)\|^2 dx\] and for a measure $\mu$ on $\reals^{\dimn}$ the dual norm
\[\|\mu\|_{\dot{H}^{-1}} = \sup\left\{|\langle f,\mu\rangle| \bigg|\|f\|_{\dot{H}^1} \leq 1\right\}\] and given measures $\mu,\nu$ on $\reals^{\dimn}$, bounded by $C < \infty$, we have 
\begin{equation*}
    \|\mu - \nu\|_{\dot{H}^{-1}} \leq \sqrt{C}\wass(\mu,\nu).
\end{equation*}

\textcolor{black}{Now consider the following bounds on $\bar{\pi}_{\hk}(\theta)$ and $\bar{\pi}_{\infty}$. Let $T_1$ denote an upper bound on $\rew$ in sampling region $\Theta$ (this exists by \ref{ass:Msmooth}). Then, we have \begin{equation}
    \sup_{\theta\in\reals^{\dimn}}\bar{\pi}_{\infty}(\theta) \leq (\exp(-\beta T_1)|\Theta|)^{-1}
\end{equation}
Notice now that by construction we can produce an upper bound on $\hat{\pi}$ as $\sup_{\theta \in \reals^{\dimn}}\hat{\pi}(\theta) \leq \frac{\mk_1}{b_S} = \frac{\mk_1}{b_T \sqrt{T}/|S|} \leq \frac{\mk_1}{b_T}$. Then, using \eqref{eq:cione} we may produce the concentration bound:
\begin{equation}
\PR\left(\sup_{\theta\in\reals^{\dimn}}\bar{\pi}_{\hk} (\theta) \leq \frac{\mk_1}{b_T} + \beta_{x,T}\right) \geq 1-\mathcal{H}(x)
\end{equation}}

In particular, \textcolor{black}{letting 
\begin{equation}
\label{eq:cdef}
    \csix = \left(\frac{\mk_1}{b_T} + \beta_{x,T}\right) \vee \left( \frac{1}{\exp(-\beta T_1 |\Theta|)}\right)
\end{equation} with probability at least $1-\mathcal{H}(x)$ we have 
\begin{equation}
\label{eq:l1recbd}
    \|\bar{\pi}_{\hk}(\theta) - \bar{\pi}_{\infty}(\theta)\|_{1} = \int_{\Theta}| \bar{\pi}_{\hk}(\theta) - \bar{\pi}_{\infty}(\theta)| d\theta  = |\langle 1, \bar{\pi}_{\hk} - \bar{\pi}_{\infty}\rangle | \leq \|\bar{\pi}_{\hk} - \bar{\pi}_{\infty}\|_{\dot{H}^{-1}} \leq \sqrt{\csix}\wass(\bar{\pi}_{\hk},\bar{\pi}_{\infty}) \leq \sqrt{\frac{2}{\zconst}\csix} \wassprox
\end{equation}}
This will be useful for us momentarily. Now observe, since $\rew$ is bounded, we can enforce Lipschitz-continuity of the logarithmic transformation ($\|\log(x) - \log(y)\| \leq L\|x-y\| \, \forall x,y \geq \exp(-\beta T_1)$), to obtain
\textcolor{black}{\begin{align*}
    &\PR\left(\int_{\Theta}|\hat{J}(\theta) - J(\theta)|d\theta \geq  \phi_{\kfv,\textcolor{black}{\snum}}\right) \leq \PR\left(\int_{\Theta}|\hat{\pi}(\theta) - \bar{\pi}_{\infty}(\theta)|d\theta \geq  \phi_{\kfv,\textcolor{black}{\snum}}/L\right) \\
    &\overset{1}{\leq} \PR\left(\int_{\Theta}|\hat{\pi}(\theta) - \bar{\pi}_{\hk}(\theta)|d\theta + \int_{\Theta}|\bar{\pi}_{\hk}(\theta) - \bar{\pi}_{\infty}(\theta)| d\theta \geq  \phi_{\kfv,\textcolor{black}{\snum}}/L \right) \\
    & \overset{2}{\leq} \PR\left(\int_{\Theta}|\hat{\pi}(\theta) - \bar{\pi}_{\hk}(\theta)|d\theta + \int_{\Theta}|\bar{\pi}_{\hk}(\theta) - \bar{\pi}_{\infty}(\theta)| d\theta \geq  \phi_{\kfv,\snum}/L \bigg| \int_{\Theta}|\bar{\pi}_{\hk}(\theta) - \bar{\pi}_{\infty}(\theta)| d\theta \leq \sqrt{\csix}\wass(\bar{\pi}_{\hk}, \bar{\pi}_{\infty})\right)\\
    &\quad \cdot \PR\left(\int_{\Theta}|\bar{\pi}_{\hk}(\theta) - \bar{\pi}_{\infty}(\theta)| d\theta \leq \sqrt{\csix}\wass(\bar{\pi}_{\hk},\bar{\pi}_{\infty})\right)\\
    & \quad + \PR\left(\int_{\Theta}|\hat{\pi}(\theta) - \bar{\pi}_{\hk}(\theta)|d\theta + \int_{\Theta}|\bar{\pi}_{\hk}(\theta) - \bar{\pi}_{\infty}(\theta)| d\theta \geq  \phi_{\kfv,\snum}/L \bigg| \int_{\Theta}|\bar{\pi}_{\hk}(\theta) - \bar{\pi}_{\infty}(\theta)| d\theta \geq \sqrt{\csix}\wass(\bar{\pi}_{\hk}, \bar{\pi}_{\infty})\right)\\
    & \quad \cdot \PR\left(\int_{\Theta}|\bar{\pi}_{\hk}(\theta) - \bar{\pi}_{\infty}(\theta)| d\theta \geq \sqrt{\csix}\wass(\bar{\pi}_{\hk},\bar{\pi}_{\infty})\right)\\
    & \leq \PR\left(\int_{\Theta}|\hat{\pi}(\theta) - \bar{\pi}_{\hk}(\theta)|d\theta + \sqrt{\csix}\wass(\bar{\pi}_{\hk}, \bar{\pi}_{\infty}) \geq  \phi_{\kfv,\snum}/L \right) + \PR\left(\int_{\Theta}|\bar{\pi}_{\hk}(\theta) - \bar{\pi}_{\infty}(\theta)| d\theta \geq \sqrt{\csix}\wass(\bar{\pi}_{\hk},\bar{\pi}_{\infty})\right)\\
    &\overset{3}{\leq} \PR\left(\int_{\Theta}|\hat{\pi}(\theta) - \bar{\pi}_{\hk}(\theta)|d\theta + \sqrt{\csix}\sqrt{2 \zconst^{-1}}\wassprox \geq  \phi_{\kfv,\snum}/L \right) + \mathcal{H}(x)\\
    &\leq \PR\left(\int_{\Theta}|\hat{\pi}(\theta) - \bar{\pi}_{\hk}(\theta)|d\theta \geq  \phi_{\kfv,\snum}/L - \sqrt{\csix}\sqrt{2 \zconst^{-1}}\wassprox\right) + \mathcal{H}(x)\\
     &\leq \PR\left(\sup_{\theta\in\Theta}|\hat{\pi}(\theta) - \bar{\pi}_{\hk}(\theta)| \geq \beta_{\kfv,\textcolor{black}{\snum}} - \frac{\sqrt{2\csix}\wassprox}{\sqrt{\zconst}|\Theta|}\right) + \mathcal{H}(x) \overset{4}{\leq} 2\left(\exp(-\psi_{\kfv,\snum}^2)  + \exp(-x^2 / 2\mk_1^2)\right)
\end{align*}
where 
\begin{equation}
\label{eq:psixt}
\psi_{\kfv,{\color{black}\snum }} = \frac{\kfv|\Theta|\sqrt{\zconst} - \sqrt{2\csix}\wassprox}{\sqrt{\zconst}|\Theta|\sqrt{b_T}}
\end{equation} 
Here inequality (1) follows by triangle inequality, (2) follows from the law of total probability, (3) follows by the bound \eqref{eq:l1recbd}, and (4) follows by augmenting $\beta_{\kfv,T},\, \mathcal{H}(\kfv)$ \eqref{eq:bhdef} to obtain $\psi_{\kfv,\snum}$.
Thus we now work with:
\begin{align*}
    \begin{split}
        &\PR\left(\int_{\Theta}|\hat{J}(\theta) - J(\theta)|d\theta \geq  \tilde{\phi}_{\kfv,{\color{black}\snum }}\right) \leq 2\exp(-\psi_{\kfv,{\color{black}{\color{black}\snum }}}^2)\\
        &\text{ where} \;\;\tilde{\phi}_{\kfv,{\color{black}\snum }} = 2 L|\Theta|\left(\frac{\kfv}{\sqrt{{\color{black}\snum }}b_T} + \frac{\mk_2}{(2\pi)^{\dimn}\sqrt{{\color{black}\snum }}b_T} + \frac{1}{2}\mk_3\mk_4 \left(b_T\textcolor{black}{\sqrt{\frac{T}{|S|}}}\right)^{2/\dimn} \right), \,\, \psi_{x,T} \text{ given in } \eqref{eq:psixt}.
    \end{split}
\end{align*}}
\textcolor{black}{Observe that $|S|$ is a random variable. We would like to remove this randomness in the concentration bound. We begin by denoting
\[P_{\Theta} = \int_{\Theta}\pi_{\hk}(\theta)d\theta\]
and deriving the following Chernoff bound:
\[\PR\left(\big||S| - \CE|S| \big| > y\right) = \PR\left(\big||S| - TP_{\Theta}\big|>y \right) \leq 2\exp\left(-2y^2/\snum\right)\]
for $y>0$. Now observe that since $|S| \geq 1$, we have $\big|\frac{1}{|S|} - \frac{1}{TP_{\Theta}}\big| \leq \big||S| - TP_{\Theta}\big|$ and thus
\begin{align*}
    &\PR\left(\bigg|\frac{T}{|S|} - \frac{1}{P_{\Theta}}\bigg| \geq t\right) = \PR\left( \bigg| \frac{1}{|S|} - \frac{1}{TP_{\Theta}}\big| \geq y/T\right) \leq \PR\left( \big| |S| - TP_{\Theta}\big| \geq y/T\right) \leq 2\exp(-2y^2/T^3)\\
    & \Rightarrow \PR\left(\frac{T}{|S|} \notin (\frac{1}{P_{\Theta}} - t,  \frac{1}{P_{\Theta}}+ t)\right) \leq 2\exp(-2y^2/T^3)
\end{align*}
Now we proceed as follows:
\begin{align*}
    &\PR\left(\int_{\Theta}|\hat{J}(\theta) - J(\theta)|d\theta \leq   2 L|\Theta|\left(\frac{\kfv}{\sqrt{\snum }b_T} + \frac{\mk_2}{(2\pi)^{\dimn}\sqrt{{\color{black}\snum }}b_T} + \frac{1}{2}\mk_3\mk_4 \left(b_T\sqrt{\frac{1}{P_{\Theta}}+y}\right)^{2/\dimn} \right)\right)  \\
    & \geq \PR\left(\left\{\int_{\Theta}|\hat{J}(\theta) - J(\theta)|d\theta \leq \tilde{\phi}_{x,T}\right\}\cap\left\{\frac{T}{|S|} \leq \frac{1}{P_{\Theta}}+t\right\}\right) \geq \left(1-2\exp(-\psi_{x,T}^2) \right)\left(1-2\exp(-2y^2/T^3)\right)
\end{align*}
and thus we conclude:
\begin{align*}
     \begin{split}
        &\PR\left(\int_{\Theta}|\hat{J}(\theta) - J(\theta)|d\theta \geq  \phi_{\kfv,{\color{black}\snum }}\right) \leq 1 - (1 - 2\exp(-\psi_{\kfv,{\color{black}{\color{black}\snum }}}^2))(1-2\exp(-2y^2/T^3))\\
        &\text{ where} \;\;\phi_{\kfv,{\color{black}\snum }} = 2 L|\Theta|\left(\frac{\kfv}{\sqrt{{\color{black}\snum }}b_T} + \frac{\mk_2}{(2\pi)^{\dimn}\sqrt{{\color{black}\snum }}b_T} + \frac{1}{2}\mk_3\mk_4 \left(b_T\textcolor{black}{\sqrt{\frac{1}{P_{\Theta}}+y}}\right)^{2/\dimn} \right), \,\, \psi_{x,T} \text{ given in } \eqref{eq:psixt}.
    \end{split}
\end{align*}}

\end{proof}}

\subsection*{Proof of Lemma \ref{lem:MSEbd}}
\begin{proof}
    We begin by bounding this Wasserstein distance by the mean-square error between the processes \eqref{eq:xtgyongy} and \eqref{eq:ytgyongy}. Recall that $Y(k\step)$ has probability law $\dtlaw$.
\begin{align*}
    \begin{split}
        \wass(\dtlaw,\ytlaw) &= \inf_{\gamma \in \Gamma(\dtlaw,\ytlaw)} \left(\CE_{(x,y)\sim \gamma}\|x-y\|^2\right)^{1/2} \\
        & \quad \leq \sqrt{\CE_{x\sim\dtlaw, \, y\sim \ytlaw }\|x-y\|^2}
    \end{split}
\end{align*}
Then we take $t = k\step$, and bound

\begin{align}
\begin{split}
    &\CE\|\yt - \xt\|^2 \\
    &\leq 3\CE\|\int_0^t g_s(\theta_{\sbar},Y(s))-\hat{g}_s(\theta_{\sbar},X(s))ds\|^2 \\&
        \quad + 3\CE\|\int_0^t\CE\left[ \idistrw(\alsbbar) | \alsbar = Y(s)\right] - \idistrw(X(s))d\browns\|^2
\end{split}
\end{align}
First we bound, using Jensen's inequality:
\begin{align*}
    &\CE\|\int_0^t g_s(\theta_{\sbar},Y(s))-\hat{g}_s(\theta_{\sbar},X(s))ds\|^2 \\
    & = \CE\|\int_0^t\CE\left[\left(\Kd(\theta_{\sbar},\alsbbar)\frac{\beta}{2}\gradn \rew(\theta_{\sbar}) \,\color{black}{-}\, \grad\idistrw(\alsbbar)\right)\idistrw(\alsbbar) \, | \, \alsbar = Y(s) \right] \\
    &\quad \quad - \CE\left[\left(\Kd(\theta_{\sbar},\alsbbar)\frac{\beta}{2}\gradn \rew(\theta_{\sbar}) \,\color{black}{-}\, \grad\idistrw(\alsbbar)\right)\idistrw(\alsbbar) | \alsbar = X(s)\right]ds\|^2 \\
    &\leq \CE\|\int_0^t \sqrt{\CE\left[ 
    \|\Kd(\theta_{\sbar},\alsbbar)\frac{\beta}{2}\gradn \rew(\theta_{\sbar})\idistrw(\alsbbar)\|^2| \alsbar = Y(s)\right] }\\
    &\quad \quad - \sqrt{\CE\left[ \|\Kd(\theta_{\sbar},\alsbbar)\frac{\beta}{2}\gradn \rew(\theta_{\sbar})\idistrw(\alsbbar)\|^2 | \alsbar = X(s)\right]} \\
    &\quad \quad \,\color{black}{-}\, \sqrt{\CE\left[\|\idistrw(\alsbbar)\grad\idistrw(\alsbbar)\|^2  | \alsbar =Y(s)\right]} \\
    &\quad \quad \,\color{black}{+}\, \sqrt{\CE\left[\|\idistrw(\alsbbar)\grad\idistrw(\alsbbar)\|^2 | \alsbar = X(s)\right]}ds\|^2\\
    &\leq\CE\|\int_0^t 4\,\step\sqrt{2C_0} + 2\sqrt{2}\,\step ds\|^2 \leq (k\step)^2\,\step^2\,(96C_0 + 24)
\end{align*}
We bound the second term, using the It\^o Isometry:
\begin{align*}
\begin{split}
    &\CE\|\int_0^t\CE\left[ \idistrw(\alsbbar) | \alsbar = Y(s)\right] - \idistrw(X(s))d\browns\|^2 \\
    & = \CE\|\int_0^t\CE\left[ \idistrw(\alsbbar) - \idistrw(X(s)) | \alsbar = Y(s) \right]d\browns\|^2 \\
    & = \CE\left[\int_0^t \CE\left[\|\idistrw(\alsbbar) - \idistrw(X(s)) \|^2 |\alsbar = Y(s)  \right]ds\right] \\
    &\leq 6\CE\left[\int_0^t \step ds\right] = 6\,k\step^2  \quad\text{ since $t=k\epsilon$}
\end{split}   
\end{align*}
Thus, we have
\begin{align*}
    \wass(\dtlaw,\ytlaw) &\leq \sqrt{72(k\step)^2\,\step^2\,(4C_0 + 1) + 18(k\step)\step} \\
    &\leq 6\,(k\step)\step\sqrt{12C_0 + 3} + 3\sqrt{2\,(k\step)\step}
\end{align*}

\end{proof}

\subsection*{Proof of Lemma \ref{lem:KL_bd}}
\begin{proof}
Let $\lawXt := \law(X(s): 0 \leq s\leq t)$ for $X(s)$ in \ref{eq:ytgyongy} and $\lawAt := \law(\alpha(s) : 0 \leq s \leq t)$ for $\alt$ in \eqref{eq:ct_diff}.  The Radon-Nikodym derivative of $\lawAt$ with respect to $\lawXt$ is given by the Girsanov formula:

\begin{align*}
&\frac{d\lawAt}{d\lawXt}(X) = \textrm{exp}\{ \int_0^t \left(G(X(s)) - g_s(\theta_{\sbar},X(s)) \right)^{*}\idistrw(X(s))^{-2} d\browns \\
&-\frac{1}{2}\int_0^t\|G(X(s)) - g_s(\theta_{\sbar},X(s))) \|^2 \idistrw(X(s))^{-2} ds\} 
\end{align*}

where $G(X(s)) = \frac{\beta}{2}\idistrw^2(X(s))\grad\rew(X(s)) + \idistrw(X(s)) \grad \idistrw(X(s))$

Then since $\CE\{\int_0^t \left(G(X(s)) - g_s(\theta_{\sbar},X(s)) \right)^{*}\idistrw(X(s))^{-2} d\browns \} = 0$, because an It\^o integral is a zero mean martingale, it follows that
\begin{align*}
\KL(\lawXt \| \lawAt) &= -\int d\lawXt \log\frac{d\lawAt}{d\lawXt} \\
&= \frac{1}{2}\int_0^t \CE \biggl[ \|G(X(s)) - g_s(\theta_{\sbar},X(s))) \|^2 \idistrw(X(s))^{-2}\biggr] ds \\
&= \frac{1}{2} \int_0^t \CE \biggl[ \| \frac{\beta}{2}\idistrw^2(X(s))\grad \rew(X(s)) \,\color{black}{-}\, \idistrw(X(s)) \grad \idistrw(X(s))\\ & \quad -  g_s(\theta_{\sbar},X(s))\|^2
\idistrw(X(s))^{-2} \biggr] ds \\
&= \frac{1}{2} \int_0^t \CE \biggl[ \| \frac{\beta}{2}\idistrw^2(\alsbar)\grad \rew(\alsbar) \,\color{black}{-}\, \idistrw(\alsbar) \grad \idistrw(\alsbar)\\ & \quad -  g_s(\theta_{\sbar},\alsbar)\|^2
\idistrw(\alsbar)^{-2}\biggr] ds 
\end{align*}
where the last line follows from the fact that $\law(\alsbar) = \law(X(s)) \ \forall s \leq t$.

Now let $t = k\step$ for some $k \in \N$. Then, expanding $g_s(\theta_{\sbar}, Y(s))$ and using Jensen's inequality, we get
{\color{black}
\begin{align}
\begin{split}
\label{eq:KL_bound}
\KL(\lawXt \| \lawAt) &\leq \frac{1}{2} \int_0^t \CE \biggl[ \| \frac{\beta}{2}\idistrw^2(\alsbar)\grad\rew(\alsbar) \,\color{black}{-}\, \idistrw(\alsbar) \grad \idistrw(\alsbar)\\ & \quad - \left(\Kd(\theta_{\sbar},\alsbar)\frac{\beta}{2}\gradn \rew(\theta_{\sbar}) \,\color{black}{-}\, \grad\idistrw(\alsbar)\right)\idistrw(\alsbar) \|^2 \idistrw(\alsbar)^{-2}\biggr] ds \\
    & =  \frac{\beta}{4} \int_0^t \CE  \| \idistrw(\alsbar)\grad \rew(\alsbar)  - \Kd(\theta_{\sbar},\alsbar)\gradn \rew(\theta_{\sbar}) \|^2 ds \\
 &= \frac{\beta}{4} \sum_{j=0}^{k-1}\int_{j\step}^{(j+1)\step} \CE  \| \idistrw(\alsbar)\grad \rew(\alsbar) - \Kd(\theta_{\sbar},\alsbar)\gradn \rew(\theta_{\sbar}) \|^2 ds \\
&\leq \frac{\beta}{2} \sum_{j=0}^{k-1} \int_{j\step}^{(j+1)\step} \CE \|\idistrw(\alsbar)\grad \rew(\alsbar) - \idistrw(\alsbar)\grad\rew(\alsebar)\|^2 ds \\
 &\quad  + \frac{\beta}{2} \sum_{j=0}^{k-1} \int_{j\step}^{(j+1)\step} \CE \| \idistrw(\alsbar)\grad\rew(\alsebar) - \Kd(\theta_{\sbar},\alsbar)\gradn \rew(\theta_{\sbar})\|^2 ds \\
 &\leq  \beta \sum_{j=0}^{k-1} \int_{j\step}^{(j+1)\step} \CE|\idistrw(\alsbar)|^2\CE \|\grad \rew(\alsbar) - \grad\rew(\alsebar)\|^2 ds \\
 \quad &+ \frac{\beta}{2} \sum_{j=0}^{k-1} \int _{j\step}^{(j+1)\step} \CE\|\grad\rew(\theta_{\sbar})\|^2 \cdot \biggl[\CE\|\idistrw(\alsbar)\langle \grad\rew(\theta_{\sbar})^{-1}, \grad\rew(\alsebar) \rangle  \\&\quad \quad \quad \quad - \Kd(\theta_{\sbar},\alsbar)\langle \grad\rew(\theta_{\sbar})^{-1}, \gradn\rew(\theta_{\sbar})\rangle 
 \|^2\biggr] \\
 &\leq  \beta \sum_{j=0}^{k-1} \int_{j\step}^{(j+1)\step} \CE|\idistrw(\alsbar)|^2\CE \|\grad \rew(\alsbar) - \grad\rew(\alsebar)\|^2 ds\\
 &\quad  + \beta \sum_{j=0}^{k-1} \int _{j\step}^{(j+1)\step} C_0  \cdot \CE\|\grad\rew(\theta_{\sbar})^{-1}\|^2 \left[ \step\,\CE\| \grad\rew(\alsebar)\|^2 +\gradnvar \, \CE\|\Kd(\theta_{\sbar},\alsbar) \|^2\right] \\
 &\leq  \beta \sum_{j=0}^{k-1} \int_{j\step}^{(j+1)\step} \CE|\idistrw(\alsbar)|^2\CE \|\grad \rew(\alsbar) - \grad\rew(\alsebar)\|^2 ds\\
 &\quad  + \beta \sum_{j=0}^{k-1} \int _{j\step}^{(j+1)\step}  C_0\left[\frac{1/\sgdtm}{\textrm{Var}(\|\grad\rew(\theta_{\sbar})\|\big| \theta_{\bar{s}_{-1}})} + \left(\CE\|\frac{1}{\|\grad\rew(\theta_{\sbar})\|}\|\right)^2\right] \left[ \step\,\CE\| \grad\rew(\alsebar)\|^2 +\gradnvar \, \CE\|\Kd(\theta_{\sbar},\alsbar) \|^2\right] \\
\end{split}
\end{align}
where $\textrm{Var}(\cdot)$ denotes variance w.r.t. the probability measure inducing the expectation, $\theta_{\bar{s}_{-1}} := \theta_{\lfloor (s-\step) / \step\rfloor }$ and we use the relation $ \textrm{Var}\left(\frac{1}{\|\grad\rew(\theta_{\sbar})\|}\right) \leq \frac{1/\sgdtm}{\textrm{Var}(\|\grad\rew(\theta_{\sbar})\|)} \leq \frac{1/\sgdtm}{\textrm{Var}(\|\grad\rew(\theta_{\sbar})\| | \theta_{\sbar_{-1}})}$. Now we bound each constituent summation term. 

\textbf{First Term:}
By $\lipGJ$-smoothness (Assumption~\ref{ass:Msmooth}), we begin to control the first term as:
\begin{align*}
&\CE\|\grad \rew(\alsbar) - \grad\rew(\alsebar)\|^2
\leq \lipGJ^2 \CE\| (\alsbar - \alsebar) \|^2 
\end{align*}
Then for $s \in [j\step, (j+1)\step]$:\\
\begin{flalign*}
&(\alsbar - \alsebar)\\
& = (s-j\step)\frac{\beta}{2}\Kd(\tjm,\ajm)\gradn\rew(\tjm)\idistrw(\ajm) \\
&\quad  + (s-j\step)\idistrw(\ajm)\grad\idistrw(\ajm)- \idistrw(\ajm)(\brown(s) - \brown(j\step))
\end{flalign*}
Then bound $\CE\|(\alsbar - \alsebar) \|^2$ as:
\begin{align}
\begin{split}
\label{eq:MSEexp}
&\CE\|(\alsbar - \alsebar) \|^2 \\
&\leq 6\step^2\CE\|\Kd(\tjm,\ajm)\frac{\beta}{2}\gradn\rew(\tjm)\|^2\CE\|\idistrw(\bar{\alpha}(j\step))\|^2 \\
&\quad + 3\step^2\CE\|\idistrw(\ajm)\grad\idistrw(\ajm) \|^2  + 3\step\CE\|\idistrw(\ajm)(W(s)-W(j\step)\|^2\\
\end{split}
\end{align}

But observe that
\begin{align}
\begin{split}
\label{eq:MSEexpp}
    \CE\|(\alsbar -& \alsebar)\|^2 \\&= \CE\left[(s-j\step)^2\|\Kd(\tjm,\ajm)\gradn\rew(\tjm)\idistrw(\ajm)\|^2\right]\\
    &\quad + \CE\left[ (s-j\step)^2\|\idistrw(\ajm)\grad\idistrw(\ajm)\|^2\right] \\
    &\quad + 2\CE\left[(s-j\step)^2\langle \Kd(\tjm,\ajm)\gradn\rew(\tjm),\,\idistrw^2(\ajm)\grad\idistrw(\ajm)\rangle\right] \\
    &\quad - \CE\left[ (s-j\step)\langle\Kd(\tjm),\ajm)\gradn\rew(\tjm),\,\idistrw^2(W(s)-W(j\step))\rangle\right]\\
    &\quad - \CE\left[(s-j\step)\langle\grad\idistrw(\ajm),\,\idistrw^2(W(s)-W(j\step)\rangle \right] \\
    &\quad + \CE\left[\idistrw^2\|\ajm)(W(s)-W(j\step)\|^2\right]
\end{split}
\end{align}
then combining \eqref{eq:MSEexpp} with \eqref{eq:MSEexp}, and using the Martingale property of Brownian motion, Jensen's Inequality, and Cauchy-Schwarz, gives:
\begin{align*}
    &\CE\left[\idistrw^2(\ajm)\|W(s)-W(j\step)\|^2\right] \\&\leq 2\step^2\CE\|\Kd(\tjm,\ajm)\gradn\rew(\tjm)\idistrw(\ajm)\|^2 \\
    &\quad + 2\step^2\CE\|\idistrw(\ajm)\grad\idistrw(\ajm)\|^2 \\&\quad + \step^2\sqrt{\CE\left[\|\langle\Kd(\tjm,\ajm)\gradn\rew(\tjm)\idistrw^2(\ajm),\, \idistrw^2(\ajm)\grad\idistrw(\ajm) \rangle\|^2\right]} \\
    &\leq 2\step^2\CE\|\Kd(\tjm,\ajm)\gradn\rew(\tjm)\idistrw(\ajm)\|^2 \\
    &\quad + 2\step^2\CE\|\idistrw(\ajm)\grad\idistrw(\ajm)\|^2 \\&\quad +\step^2\sqrt{\CE\|\Kd(\tjm,\ajm)\gradn\rew(\tjm)\idistrw(\ajm)\|^2} \\&\quad + \step^2\sqrt{\CE\|\idistrw(\ajm)\grad\idistrw(\ajm)\|^2} 
\end{align*}
Thus we have:
\begin{align*}
    &\CE\|(\alsbar - \alsebar) \|^2 \\
    &\leq 18\step^2\CE\|\Kd(\tjm,\ajm)\frac{\beta}{2}\gradn\rew(\tjm)\|^2\CE\|\idistrw(\bar{\alpha}(j\step))\|^2 \\
    &\quad + 9\step^2\CE\|\idistrw(\ajm)\grad\idistrw(\ajm) \|^2 \\
    &\quad + 3\step^2\sqrt{\CE\|\Kd(\tjm,\ajm)\frac{\beta}{2}\gradn\rew(\tjm)\|^2\CE\|\idistrw(\bar{\alpha}(j\step))\|^2} \\
    &\quad  + 3\step^2\sqrt{\CE\|\idistrw(\ajm)\grad\idistrw(\ajm) \|^2} \\
    &\leq \step^4\left(72C_0+ 18\right) + \step^3\left(6\sqrt{C_0} +\sqrt{2}\right)
\end{align*}

where recall $C_0 := 3\lipGJ^2(\etsup + 2B^2\etsup) + B^2 + \gradnvar$.
Consequently, 

\begin{align}
\begin{split}
\label{eq:bd1}
&\beta\sum_{j=0}^{k-1} \int_{j\step}^{(j+1)\step}\CE|\idistrw(\alsbar)|^2 \CE \|\grad \rew(\alsbar) -  \grad\rew(\alsebar)\|^2 ds \\
& \leq \beta\sum_{j=0}^{k-1} \int_{j\step}^{(j+1)\step} \lipGJ^2 \CE|\idistrw(\alsbar)|^2 \CE\|(\alsbar - \alsebar) \|^2 ds \\
& \leq \beta \lipGJ^2 k\step \left(\step^4\left(72C_0+ 18\right) + \step^3\left(6\sqrt{C_0} +\sqrt{2}\right)\right)
\end{split}
\end{align}

\textbf{Second Term}: We can further break up the second term as 
\begin{align*}
    &\beta \sum_{j=0}^{k-1} \int _{j\step}^{(j+1)\step}  \left[\frac{1/\sgdtm}{\textrm{Var}(\|\grad\rew(\theta_{\sbar})\|\big| \theta_{\bar{s}_{-1}})} + \left(\CE\|\frac{1}{\|\grad\rew(\theta_{\sbar})\|}\|\right)^2\right] \left[ \step\,\CE\| \grad\rew(\alsebar)\|^2 +\gradnvar \, \CE\|\Kd(\theta_{\sbar},\alsbar) \|^2\right] \\
    & \leq \beta \sum_{j=0}^{k-1} \int _{j\step}^{(j+1)\step}  \left[\frac{1/\sgdtm}{\hat{\mu}_{sgd}} + \frac{1}{\sgdtm^2} \right] \left[ \step\,\CE\| \grad\rew(\alsebar)\|^2 +\gradnvar \, \CE\|\Kd(\theta_{\sbar},\alsbar) \|^2\right] \\
    &\leq \beta \sum_{j=0}^{k-1} \int _{j\step}^{(j+1)\step}  \gradnvar\,\left[\frac{1/\sgdtm}{\hat{\mu}_{sgd}} + \frac{1}{\sgdtm^2} \right] \CE\|\Kd(\theta_{\sbar},\alsbar) \|^2 ds  \\&\quad + 2\beta \left[\frac{1/\sgdtm}{\hat{\mu}_{sgd}} + \frac{1}{\sgdtm^2} \right]\step \sum_{j=0}^{k-1} \int _{j\step}^{(j+1)\step} \CE\| \grad\rew(\alsebar) - \grad\rew(\bar{\alpha}(0))\|^2ds \\&\quad +  2\beta\left[\frac{1/\sgdtm}{\hat{\mu}_{sgd}} + \frac{1}{\sgdtm^2} \right]\step \sum_{j=0}^{k-1} \int _{j\step}^{(j+1)\step} \CE\|\grad\rew(\bar{\alpha}(0))\|^2ds 
\end{align*}

Then, bound the first sum as 
\begin{equation}
\label{bd:fs}
    \beta \sum_{j=0}^{k-1} \int _{j\step}^{(j+1)\step}  \gradnvar\,\left[\frac{1/\sgdtm}{\hat{\mu}_{sgd}} + \frac{1}{\sgdtm^2} \right] \CE\|\Kd(\theta_{\sbar},\alsbar) \|^2 ds \, \leq \, 2\step \beta (k\step)\gradnvar\,\left[\frac{1/\sgdtm}{\hat{\mu}_{sgd}} + \frac{1}{\sgdtm^2} \right],
\end{equation}
and the middle sum as 
\begin{align}
\begin{split}
\label{eq:bd2}
&2\beta\step \sum_{j=0}^{k-1} \int_{j\step}^{(j+1)\step}\left[\frac{1/\sgdtm}{\hat{\mu}_{sgd}} + \frac{1}{\sgdtm^2} \right] \CE \| \grad\rew(\alsebar) - \grad\rew(\bar{\alpha}(0))\|^2 ds \\
& \quad = \, 2\beta \left[\frac{1/\sgdtm}{\hat{\mu}_{sgd}} + \frac{1}{\sgdtm^2} \right] \sum_{j=0}^{k-1} \step^2 \CE \| \grad\rew(\bar{\alpha}(j\step)) - \grad\rew(\bar{\alpha}(0))\|^2 ds \\
& \quad \leq \, 2\beta \left[\frac{1/\sgdtm}{\hat{\mu}_{sgd}} + \frac{1}{\sgdtm^2} \right]\sum_{j=0}^{k-1} \step^2 \CE \| \sum_{i=0}^{j-1}|\grad\rew(\bar{\alpha}((i+1)\step)) - \grad\rew(\bar{\alpha}(i\step))|\|^2 \\
& \quad \leq 2\beta \left[\frac{1/\sgdtm}{\hat{\mu}_{sgd}} + \frac{1}{\sgdtm^2} \right] \sum_{j=0}^{k-1} \step^2 \, \lipGJ^2 \CE \| \sum_{i=0}^{j-1}|\bar{\alpha}((i+1)\step) - \alpha(i\step)|\|^2 \\
& \quad = 2\beta \left[\frac{1/\sgdtm}{\hat{\mu}_{sgd}} + \frac{1}{\sgdtm^2} \right] \sum_{j=0}^{k-1}\step^2 \lipGJ^2\CE\left[\sum_{i=0}^{j-1}\sum_{l=0}^{j-1}|\bar{\alpha}((i+1)\step) - \bar{\alpha}(i\step)||\bar{\alpha}((l+1)\step) - \bar{\alpha}(l\step)| \right] \\
&\quad \leq 2\beta\left[\frac{1/\sgdtm}{\hat{\mu}_{sgd}} + \frac{1}{\sgdtm^2} \right] \sum_{j=0}^{k-1}\step^2 \lipGJ^2 \left[\sum_{i=0}^{j-1}\sum_{l=0}^{j-1} \CE\left[|\bar{\alpha}((i+1)\step) - \bar{\alpha}(i\step)||\bar{\alpha}((l+1)\step) - \bar{\alpha}(l\step)|\right] \right]\\
&\quad \leq 2\beta \left[\frac{1/\sgdtm}{\hat{\mu}_{sgd}} + \frac{1}{\sgdtm^2} \right] \sum_{j=0}^{k-1}\step^2 \lipGJ^2 \left[\sum_{i=0}^{j-1}\sum_{l=0}^{j-1} \left(\step^4\left(72C_0+ 18\right) + \step^3\left(6\sqrt{C_0} +\sqrt{2}\right)\right)\right] \\
& \quad = 2\beta \left[\frac{1/\sgdtm}{\hat{\mu}_{sgd}} + \frac{1}{\sgdtm^2} \right] \sum_{j=0}^{k-1}\step^2 \lipGJ^2 j^2 \left(\step^4\left(72C_0+ 18\right) + \step^3\left(6\sqrt{C_0} +\sqrt{2}\right)\right) \\
&\quad \leq 2\beta \left[\frac{1/\sgdtm}{\hat{\mu}_{sgd}} + \frac{1}{\sgdtm^2} \right] k\step^2 \lipGJ^2 k^2 \left(\step^4\left(72C_0+ 18\right) + \step^3\left(6\sqrt{C_0} +\sqrt{2}\right)\right) \\
&\quad = 2\beta \left[\frac{1/\sgdtm}{\hat{\mu}_{sgd}} + \frac{1}{\sgdtm^2} \right] (k\step)^3 \lipGJ^2 \left(\step^4\left(72C_0+ 18\right) + \step^3\left(6\sqrt{C_0} +\sqrt{2}\right)\right)
\end{split}
\end{align}
Using \ref{ass:Msmooth} we bound the third term as
\begin{align}
\label{eq:ttbound}
 & 2\beta \left[\frac{1/\sgdtm}{\hat{\mu}_{sgd}} + \frac{1}{\sgdtm^2} \right]\step\sum_{j=0}^{k-1} \int_{j\step}^{(j+1)\step}\CE \|\grad(\rew(\alpha(0))) \|^2 \leq 2\beta \left[\frac{1/\sgdtm}{\hat{\mu}_{sgd}} + \frac{1}{\sgdtm^2} \right] k\step\lipJ^2\step
\end{align}

Combining \eqref{eq:bd1}, \eqref{bd:fs}, \eqref{eq:bd2}, \eqref{eq:ttbound}, in \eqref{eq:KL_bound}, and using Lemma~\ref{lem:sgdvarbd}, we obtain
 
\begin{align*}
        &\KL(\lawXt \| \lawAt) \\& \leq \beta \lipGJ^2 k\step \left(\step^4\left(72C_0+ 18\right) + \step^3\left(6\sqrt{C_0} +\sqrt{2}\right)\right) 
        \\&\quad + 2\step \beta (k\step) \gradnvar\,\left[\frac{1/\sgdtm}{\hat{\mu}_{sgd}} + \frac{1}{\sgdtm^2} \right]
        \\&\quad+ 2\beta \left[\frac{1/\sgdtm}{\hat{\mu}_{sgd}} + \frac{1}{\sgdtm^2} \right](k\step)^3 \lipGJ^2 \left(\step^4\left(72C_0+ 18\right) + \step^3\left(6\sqrt{C_0} +\sqrt{2}\right)\right)\\
        &\quad +  2\beta\left[\frac{1/\sgdtm}{\hat{\mu}_{sgd}} + \frac{1}{\sgdtm^2} \right] k\step\lipJ^2\step \\
        & \leq (k\step)^3\,\step^3\,\left[2\beta \lipGJ^2\left[\frac{1/\sgdtm}{\hat{\mu}_{sgd}} + \frac{1}{\sgdtm^2} +1\right]\left(72C_0 + 6\sqrt{C_0} + 18+\sqrt{2} \right) \right] + 2\beta \left[\frac{1/\sgdtm}{\hat{\mu}_{sgd}} + \frac{1}{\sgdtm^2} \right](k\step)\,\step\,(\lipJ^2 + \gradnvar)
\end{align*}}
Now since $\dtlaw = \law(\alk)$ and $\ctlaw = \law(\alpha(t))$, the KL divergence data-processing inequality yields
\begin{align*}
\KL(\dtlaw \| \ctlaw) & \leq \KL(\lawXt \| \lawAt)
\end{align*}
\end{proof}

\section{Proof of Proposition~\ref{prop:logsob}}
\begin{proof}

    Recall that the continuous time diffusion of interest \eqref{eq:ct_diff} has infinitesimal generator $\gen$ acting on $C^2$ function $f$ as 
    \[\gen f = \frac{1}{2}\idistrw^2 \Delta f - \frac{\beta}{2} \idistrw \langle \grad \rew , \grad f \rangle  \,\color{black}{+}\, \idistrw \langle  \grad \idistrw, \grad f\rangle\]
    We will show that the conditions of Proposition~\ref{prop:cat} hold:
    \begin{enumerate}
        \item Consider the Lyapunov function \[V(\larg) = \exp\left(\frac{\beta \dissm\|\larg\|^2 }{2(\idistmaxw^2+1)}\right)\]
        Then we have
        \begin{align}
        \begin{split}
        \label{eq:Lyapbound}
            \gen V(\larg) &= -\frac{\beta}{2} \idistrw^2(\larg) \langle \grad\rew(\larg),\grad V(\larg)\rangle - \idistrw(\larg) \langle \grad \idistrw(\larg), \grad V(\larg)\rangle \\&\quad+ \frac{1}{2}\idistrw^2(\larg) \Delta V (\larg)\\
            & = \biggl\{-\frac{\beta}{2} \frac{\dissm \beta}{(\idistmaxw^2+1)} \idistrw^2(\larg) \langle \grad \rew(\larg), \larg\rangle \,\color{black}{+}\, \frac{\dissm \beta}{\idistmaxw^2+1} \idistrw(\larg) \langle  \grad \idistrw(\larg), \larg \rangle \\
            & \quad \quad + \frac{1}{2}\idistrw^2(\larg)(\frac{\dissm \beta \dimn}{\idistmaxw^2+1} + (\frac{\dissm\beta}{\idistmaxw^2+1})^2\|\larg\|^2))\biggr\}V(\larg) \\
            & \leq \biggl\{-\frac{\beta}{2} \frac{\dissm \beta}{\idistmaxw^2+1}  \idistrw^2(\larg) (\dissm \|\larg\|^2 - \dissb) + \frac{\dissm \beta \idiffc}{\idistmaxw^2+1} \idistrw(\larg)  \\
            &\quad \quad + \frac{1}{2}\idistrw^2(\larg)(\frac{\dissm \beta \dimn}{\idistmaxw^2+1} + (\frac{\dissm\beta}{\idistmaxw^2+1})^2\|\larg\|^2))\biggr\}V(\larg)\\
            & \leq \biggl\{\biggl(\frac{1}{2} \beta \dissm \dimn +  \dissm \beta \idiffc + \frac{1}{2}\idistrw^2(\larg) \frac{\beta^2 \dissm \dissb}{\idistmaxw^2+1}\biggr) \\& \quad \quad - \biggl(\frac{\beta^2\dissm}{2(\idistmaxw^2+1)}\idistrw^2(\larg)\dissm - \frac{1}{2}\idistrw^2(\larg)(\frac{\beta \dissm}{(\idistmaxw^2+1)})^2 \biggr)\|\larg\|^2 \biggr\}V(\larg) \\
            & \leq \biggl\{\biggl(\frac{1}{2}\beta \dissm \dimn  + \dissm \beta \idiffc\biggr) + \frac{1}{2}\idistrw^2(\larg) \frac{\beta^2 \dissm \dissb}{\idistmaxw^2+1}\\& \quad \quad - \left(\frac{1}{2}\idistrw^2(\larg)\frac{(\beta \dissm)^2}{\idistmaxw^2+1} - \frac{1}{2}\idistrw^2(\larg)(\frac{\beta \dissm}{\idistmaxw^2+1})^2 \right)\|\larg\|^2 \biggr\}V(\larg) \\
            & \leq \biggl\{\biggl(\frac{1}{2} \beta \dissm \dimn  + \dissm \beta \idiffc \biggr)  + \frac{1}{2}\idistrw^2(\larg) \frac{\beta^2 \dissm \dissb}{\idistmaxw^2+1} \\& \quad \quad - \left(\frac{1}{2}\idistrw^2(\larg)\frac{(\beta \dissm)^2}{\idistmaxw^2+1}\left(1 - \frac{1}{\idistmaxw^2+1}\right) \right)\|\larg\|^2 \biggr\}V(\larg) \\
            & \leq \biggl\{\biggl(\frac{1}{2}\beta \dissm \dimn  + \dissm \beta \idiffc \biggr)  + \biggl[ \frac{1}{2}\idistrw^2(\larg) \frac{\beta^2 \dissm \dissb}{\idistmaxw^2+1} \\&\quad \quad + \frac{1}{2}\idistrw^2(\larg)\left(\frac{(\beta \dissm)^2}{\idistmaxw^2 + 1}\tailbound^2\right) + 1 \\&\quad \quad - \left(\frac{1}{2}\idistrw^2(\larg)\frac{(\beta \dissm)^2}{\idistmaxw^2+1}\left(1 - \frac{1}{\idistmaxw^2+1}\right) \right)\|\larg\|^2\biggr] \biggr\}V(\larg) \\
        \end{split}
        \end{align}
        where in the last inequality statement we append \[\frac{1}{2}\idistrw^2(\larg)\left(\frac{(\beta \dissm)^2}{\idistmaxw^2 + 1}\tailbound^2\right) + 1\] for the following reason: By Assumption~\ref{ass:idistexp} we have that $\exists \tailbound : \idistrw(\larg) < \frac{2}{(\beta \dissm)^2\|\larg\|^2} \ \forall \|\larg\| > \tailbound, \,\idw\leq 1$. We aim to show that the term inside brackets in the last inequality line of \eqref{eq:Lyapbound} is positive for all $\larg$. First take $\|\larg\| < \tailbound$: we have that:
        \[\left(\frac{1}{2}\idistrw^2(\larg)\frac{(\beta \dissm)^2}{\idistmaxw^2+1}\left(1 - \frac{1}{\idistmaxw^2+1}\right) \right)\|\larg\|^2 <  \frac{1}{2}\idistrw^2(\larg)\left(\frac{(\beta \dissm)^2}{\idistmaxw^2 + 1}\tailbound^2\right)\]
        Now consider $\|\larg\| \geq \tailbound$. By Lemma~\ref{lem:taildecay} we have 
        \begin{align*}
                &\frac{1}{2}\idistrw(\larg)(\beta \dissm)^2\|\larg\|^2 < 1 \\
                &\Rightarrow \frac{1}{2}\idistrw^2(\larg) \frac{(\beta \dissm)^2}{\idistmaxw^2+1}\|\larg\|^2 < 1 \\
                &\Rightarrow \frac{1}{2}\idistrw^2(\larg) \frac{(\beta \dissm)^2}{\idistmaxw^2+1}\left(1 - \frac{1}{\idistmaxw^2+1}\right)\|\larg\|^2 < 1
        \end{align*} 
        Thus we have:
        \begin{align*}
           & \biggl[ \frac{1}{2}\idistrw^2(\larg) \frac{\beta^2 \dissm \dissb}{\idistmaxw^2+1}  + \frac{1}{2}\idistrw^2(\larg)\left(\frac{(\beta \dissm)^2}{\idistmaxw^2 + 1}\tailbound^2\right) + 1 \\ &\quad - \left(\frac{1}{2}\idistrw^2(\larg)\frac{(\beta \dissm)^2}{\idistmaxw^2+1}\left(1 - \frac{1}{\idistmaxw^2+1}\right) \right)\|\larg\|^2\biggr] > 0\quad \forall \larg \in \reals^{\dimn}
        \end{align*}
        and now observe that
        \begin{align*}
                &\biggl[ \frac{1}{2}\idistrw^2(\larg) \frac{\beta^2 \dissm \dissb}{\idistmaxw^2+1}  + \frac{1}{2}\idistrw^2(\larg)\left(\frac{(\beta \dissm)^2}{\idistmaxw^2 + 1}\tailbound^2\right) + 1 \\&\quad - \left(\frac{1}{2}\idistrw^2(\larg)\frac{(\beta \dissm)^2}{\idistmaxw^2+1}\left(1 - \frac{1}{\idistmaxw^2+1}\right) \right)\|\larg\|^2\biggr]\\
                &\leq \biggl[ \frac{1}{2}(\idistrw^2(\larg)+1) \frac{\beta^2 \dissm \dissb}{\idistmaxw^2+1}  + \frac{1}{2}(\idistrw^2(\larg)+1)\left(\frac{(\beta \dissm)^2}{\idistmaxw^2 + 1}\tailbound^2\right)  \\&\quad - \left(\frac{1}{2}(\idistrw^2(\larg)+1)\frac{(\beta \dissm)^2}{\idistmaxw^2+1}\left(1 - \frac{1}{\idistmaxw^2+1}\right) \right)\|\larg\|^2\biggr] \\
                &\leq (\idistrw^2(\larg)+1)\biggl[ \frac{1}{2} \frac{\beta^2 \dissm \dissb}{\idistmaxw^2+1}  + \frac{1}{2}\left(\frac{(\beta \dissm)^2}{\idistmaxw^2 + 1}\tailbound^2\right) \\&\quad - \left(\frac{1}{2}\frac{(\beta \dissm)^2}{\idistmaxw^2+1}\left(1 - \frac{1}{\idistmaxw^2+1}\right) \right)\|\larg\|^2\biggr] \\
                &\leq (\idistmaxw^2+1)\biggl[ \frac{1}{2} \frac{\beta^2 \dissm \dissb}{\idistmaxw^2+1}  + \frac{1}{2}\left(\frac{(\beta \dissm)^2}{\idistmaxw^2 + 1}\tailbound^2\right) \\&\quad - \left(\frac{1}{2}\frac{(\beta \dissm)^2}{\idistmaxw^2+1}\left(1 - \frac{1}{\idistmaxw^2+1}\right) \right)\|\larg\|^2\biggr] \\
        \end{align*}
        where we use that $1 \leq \frac{(\beta \dissm)^2}{\idistmaxw + 1}\tailbound^2$, which is derived from Assumption~\ref{ass:beta}. 
        Thus we have:
        \begin{align*}
                \frac{\gen V(\larg)}{V(\larg)} & \leq \biggl\{\biggl(\frac{1}{2} \beta \dissm \dimn  + \dissm \beta \idiffc \biggr)  \\&\quad + (\idistmaxw^2+1)\biggl[ \frac{1}{2} \frac{\beta^2\dissm \dissb}{\idistmaxw^2+1}  + \frac{1}{2}\left(\frac{(\beta\dissm)^2}{\idistmaxw^2 + 1}\tailbound^2\right) \biggr] \\
                &\quad - (\idistmaxw^2 + 1)\left(\frac{1}{2}\frac{(\beta\dissm)^2}{\idistmaxw^2+1}\left(1 - \frac{1}{\idistmaxw^2+1}\right) \right)\|\larg\|^2\biggr\} \\
                &\leq \biggl(\frac{1}{2} \beta\dissm \dimn  + \beta \dissm \idiffc \biggr) + \frac{1}{2}\biggl[ \beta^2\dissm \dissb  + \left(\beta\dissm\tailbound\right)^2 \biggr]\\
                &\quad - \frac{1}{2}\left((\beta\dissm)^2 + \left( 1-\frac{1}{\idistmax^2+1}\right)\right)\|\larg\|^2
        \end{align*}

        \begin{align*}
             := \kappa - \gamma\|\larg\|^2\\
        \end{align*}
        and observe that $\kappa, \gamma >0$. Thus, condition (1) of Proposition~\ref{prop:cat} holds. 

        \item From \eqref{eq:Lyapbound}, we have 
        \[\frac{\gen V(\larg)}{V(\larg)} \leq \kappa - \gamma\|\larg\|^2\]
        Observe that this implies 
        \[\frac{\gen V(\larg)}{V(\larg)} \leq -\kappa + 2\kappa\ind_{(\|\larg\|^2 \leq 2\kappa / \gamma)}\]
    Moreover, by Lemma~\ref{lem:quad_bd} and since $\rew(\larg) \geq 0$, we have 
    \[O_{\bccc}(\beta \rew) \leq \beta\left(\frac{\lipGJ}{2}\|\larg\|^2 + \|\larg\| + \constA \right) \leq \beta\left(\frac{(\lipGJ+\constB)\bccc^2}{2} + \constA + \constB\right) \]
    and so by Proposition~\ref{prop:Bakry}, with $\bc = 2\kappa, \bcc = \kappa, \bccc^2 = 2\kappa / \gamma$, we have that $\gibbs$ satisfies a Poincar\^e inequality with constant
    \begin{align*}
            &\pconst = \frac{1}{\lambda} \leq  \frac{1}{2\kappa} \biggl(1 + \frac{4C\kappa^2}{\gamma}\exp\biggl(\beta\biggl(\frac{(\lipGJ+\constB)\kappa}{\gamma} + \constA + \constB\biggr)\biggr) \biggr)
    \end{align*} 
    where $\kappa$ and $\gamma$ are defined above and provided in \eqref{eq:kg}.
    
        \item By assumption \ref{ass:Msmooth}, we have 
        \[\grad^2\beta\rew \succeq -\beta \lipGJ I_d \succeq 0\]
    \end{enumerate}
    Thus the conditions of Proposition~\ref{prop:cat} are met, with $K = 0$. So, letting $\zeta = 1$:
    \[ Z_1 = \frac{2}{\gamma} + 1, \quad Z_2 = \frac{2}{\gamma}\left( \kappa + \gamma\int_{\reals^{\dimn}}
    \|w\|^2 \gibbs(dw)\right)\]
    Now we would like to make the bound on $\LSconst$ \eqref{eq:lsconst} more explicit by providing a bound on $\int_{\reals^{\dimn}} \|w\|^2\gibbs(dw)$. From \eqref{eq:ct_bd} we have $\wass(\ctlawt,\gibbs) \to 0$ as $t \to \infty$, and thus by Theorem 7.12 of \cite{villani2021topics} and Lemma~\ref{lem:diffbound} it follows that (with $\idw\leq1)$
    \begin{align}
        \begin{split}
            &\int_{\reals^{\dimn}}\|w\|^2\gibbs(dw) = \lim_{t\to\infty}\int_{\reals^{\dimn}}\|w\|^2\ctlawt(dw) \\&\quad \leq \knw + \frac{(\beta \dissb + \dimn)\idistmaxw + 2\idiffc}{(\dissm\beta)\idistmaxw} \leq \kn + \frac{(\beta \dissb + \dimn)\bar{\pi} + 2\idiffc}{(\dissm\beta)\bar{\pi}}
        \end{split}
    \end{align}
    So now letting \[Z_2 = \frac{2}{\gamma}\left( \kappa + \gamma\left(\ \kn + \frac{(\beta \dissb + \dimn)\bar{\pi} + 2\idiffc}{(\dissm\beta)\bar{\pi}}\right)\right)\]
    we have that $\gibbs$ satisfies a log-Sobolev inequality with constant 
    \[\LSconst = Z_1 + (Z_2 + 2)\pconst\]
\end{proof}

\end{document}